\documentclass[english]{IEEEtran}
\usepackage[T1]{fontenc}
\usepackage[latin9]{inputenc}
\usepackage{color}
\usepackage{array}
\usepackage{float}
\usepackage{textcomp}
\usepackage{bm}
\usepackage{multirow}
\usepackage{amsmath}
\usepackage{amsthm}
\usepackage{amssymb}
\usepackage{graphicx}

\makeatletter

\providecommand{\tabularnewline}{\\}
\floatstyle{ruled}
\newfloat{algorithm}{tbp}{loa}
\providecommand{\algorithmname}{Algorithm}
\floatname{algorithm}{\protect\algorithmname}

\theoremstyle{remark}
\newtheorem{rem}{\protect\remarkname}
\theoremstyle{plain}
\newtheorem{prop}{\protect\propositionname}
\theoremstyle{plain}
\newtheorem{lem}{\protect\lemmaname}
\theoremstyle{plain}
\newtheorem{thm}{\protect\theoremname}

\usepackage{amsmath,amssymb}

\usepackage{color}  
\usepackage{graphicx,psfrag,cite,subfigure}

\usepackage{flushend}

\author{
	Zheng~Xing,~\IEEEmembership{Student Member,~IEEE},
	and Junting~Chen,~\IEEEmembership{Member,~IEEE}
	
\thanks{Manuscript submitted June 27, 2023, revised November 16, 2023, and accepted February 5, 2023. (Corresponding author: Junting Chen.)}
\thanks{The work was supported in part by the National Science Foundation of China (NSFC) under Grant No. 62171398, by the Basic Research Project No. HZQB-KCZYZ-2021067 of Hetao Shenzhen-HK S\&T Cooperation Zone, by NSFC Grant No. 62293482, by the Shenzhen Science and Technology Program under Grant No. JCYJ20210324134612033 and No. KQTD20200909114730003, by Guangdong Research Projects No. 2019QN01X895, No. 2017ZT07X152, and No. 2019CX01X104, by the Shenzhen Outstanding Talents Training Fund 202002, by the Guangdong Provincial Key Laboratory of Future Networks of Intelligence (Grant No. 2022B1212010001), by the National Key R\&D Program of China with grant No. 2018YFB1800800, and by the Key Area R\&D Program of Guangdong Province with grant No. 2018B030338001.}
\thanks{Z.~Xing and J.~Chen are with the School of Science and Engineering, 
	and the Future Network of Intelligence Institute (FNii), The Chinese University of Hong Kong, Shenzhen, Guangdong 518172, China.}

}

\usepackage[acronym]{glossaries}
\newcommand{\newac}{\newacronym}
\newcommand{\ac}{\gls}
\newcommand{\Ac}{\Gls}

\makeglossaries

\usepackage{enumitem}

\newac{ips}{IPS}{indoor positioning system}
\newac{wlan}{WLAN}{wireless local area network}
\newac{ap}{AP}{access point}
\newac{sc}{SC}{subspace clustering}
\newac{uwb}{UWB}{ultrawideband}
\newac{hmm}{HMM}{hidden Markov model}
\newac{ema}{EM}{expectation maximization}
\newac{lbs}{LBS}{location-based service}
\newac{gpc}{GPC}{Gaussian process classification}

\newac{gpca}{GPCA}{generalized PCA}
\newac{msl}{MSL}{multistage learning}
\newac{mppca}{MPPCA}{mixture of probabilistic PCA}
\newac{alc}{ALC}{agglomerative lossy compression}
\newac{ransac}{RANSAC}{random sample consensus}
\newac{fa}{FA}{factorization-based affinity}
\newac{gpcaa}{GPCAA}{gpca-based affinity}
\newac{slbf}{SLBF}{spectral local best-fit flats}
\newac{lsa}{LSA}{local subspace affinity}
\newac{llmc}{LLMC}{locally linear manifold clustering}
\newac{ssc}{SSC}{sparse subspace clustering}
\newac{lrr}{LRR}{low-rank representation}
\newac{scc}{SCC}{spectral curvature clustering}
\newac{mfb}{MFB}{matrix factorization-based}
\newac{mssc}{MSSC}{multi-view low-rank sparse subspace clustering}
\newac{dsc}{DSC}{deep subspace clustering}
\newac{sssc}{SSSC}{Stochastic sparse subspace clustering}
\newac{mdssc}{MDSSC}{}
\newac{nmi}{NMI}{normalized mutual information}
\newac{svm}{SVM}{support vector machine}
\newac{tdoa}{TDOA}{time difference of arrival}
\newac{gmm}{GMM}{Gaussian mixture model}
\newac{pca}{PCA}{Principal Component Analysis}
\newac{dnn}{DNN}{deep neural network}
\newac{mlp}{MLP}{multilayer perceptron}
\newac{admm}{ADMM}{alternating direction method of multipliers}
\newac{gan}{GAN}{generative adversarial networks}

\newac{rss}{RSS}{received signal strength}
\newac{wcl}{WCL}{weighted centroid localization}
\newac{knn}{KNN}{$k$-nearest neighbor}
\newac{toa}{TOA}{time-of-arrival}
\newac{aoa}{AoA}{angle of arrival}
\newac{los}{LOS}{line-of-sight}
\newac{csi}{CSI}{channel state information}
\newac{em}{EM}{expectation-maximization}
\newac{wrt}{w.r.t.}{with respect to}
\newac{rhs}{R.H.S.}{right hand side}

\usepackage{babel}

\makeatother

\usepackage{babel}
\providecommand{\lemmaname}{Lemma}
\providecommand{\propositionname}{Proposition}
\providecommand{\remarkname}{Remark}
\providecommand{\theoremname}{Theorem}

\begin{document}
\title{Constructing Indoor Region-based Radio Map without Location Labels}
\maketitle
\begin{abstract}
Radio map construction requires a large amount of radio measurement
data with location labels, which imposes a high deployment cost. This
paper develops a region-based radio map from \ac{rss} measurements
without location labels. The construction is based on a set of blindly
collected \ac{rss} measurement data from a device that visits each
region in an indoor area exactly once, where the footprints and timestamps
are not recorded. The main challenge is to cluster the \ac{rss} data
and match clusters with the physical regions. Classical clustering
algorithms fail to work as the \ac{rss} data naturally appears as
non-clustered due to multipaths and noise. In this paper, a signal
subspace model with a sequential prior is constructed for the \ac{rss}
data, and an integrated segmentation and clustering algorithm is developed,
which is shown to find the globally optimal solution in a special
case. Furthermore, the clustered data is matched with the physical
regions using a graph-based approach. Based on real measurements from
an office space, the proposed scheme reduces the region localization
error by roughly 50\% compared to a \ac{wcl} baseline, and it even
outperforms some supervised localization schemes, including \ac{knn},
\ac{svm}, and \ac{dnn}, which require labeled data for training.
\end{abstract}

\begin{IEEEkeywords}
Localization, blind calibration, radio map, subspace clustering, segmentation.
\end{IEEEkeywords}

\glsresetall

\section{Introduction}

Location-based services have gained significant attention in the industry
and research community due to the proliferation of mobile devices
\cite{Zaf:J19,Chen:J22}. While several advanced triangulation-based
approaches, such as \ac{toa}\cite{GuoNi:J22}, \ac{tdoa}\cite{Haj:J22},
or \ac{aoa}\cite{SunYi:J22}, can achieve sub-meter level localization
performance under \ac{los} conditions, they require specialized and
complex hardware to enable. In many indoor applications, rough accuracy
is acceptable, but hardware cost is a primary concern. For instance,
monitoring the locations of numerous equipment in a factory necessitates
a large number of low-cost, battery-powered devices, while, in this
application, meter-level accuracy suffices, {\em e.g.}, it suffices
to determine whether the equipment is in room A or B. Consequently,
\ac{rss}-based localization may be found as the most cost-effective
solution for indoor localization in these scenarios, as it does not
require complicated hardware or a sophisticated localization protocol.

Traditional \ac{rss}-based indoor localization algorithms can be
roughly categorized into model-based, model-free, and data-driven
approaches. Model-based approaches \cite{Pand:J22,PraSur:20} first
estimate a path loss model to describe how the \ac{rss} varies with
propagation distance, and then estimate the target location by measuring
the propagation distance from sensors with known locations based on
the measured \ac{rss}. However, these approaches require calibration
for the path loss model and are also highly sensitive to signal blockage.
Model-free approaches employ an empirical formula to estimate the
target location. For example, \ac{wcl} approaches \cite{PhoSon:J18,Pand:J22}
estimate the target location as the weighted average of sensor locations
using an empirical formula, where the \ac{rss} values can be used
as weights. However, the choice of the empirical formula significantly
affects the localization accuracy of \ac{wcl}. Data-driven approaches
mostly require an offline measurement campaign to collect {\em location-labeled}
\ac{rss} measurements at numerous spots in the target area to build
a {\em fingerprint} database \cite{ShrSag:J22,DuXia:J21,WangLuo:J21,SunChen:C22,WangLuo:J18,WanXia:J20,GaoZha:C22,WuGao:C23,SunChen:J22}.
The target is localized by comparing the \ac{rss} measurements with
the fingerprints in the database. However, such fingerprinting approaches
not only require extensive labor to collect a large amount of \ac{rss}
measurements tagged with location labels, but also require significant
calibration effort after the system is deployed, because the \ac{rss}
signature may vary due to changes in the environment, such as of the
change of furniture. An outdated fingerprint database may degrade
the localization performance. Therefore, reducing the construction
and calibration costs is a critical issue for \ac{rss}-based localization.

There are some works on reducing the construction and calibration
effort required for fingerprint localization. For example, the work
\cite{CheKev:C21} employed \ac{gan} to augment the fingerprint training
dataset. Moreover, the work \cite{GaoWu:J23} proposed a meta-learning
approach to train the fingerprint with a few labeled samples. In addition,
some works utilize interpolation, such as Kriging spatial interpolation\cite{LiyNis:C19},
to recover a significant amount of unlabeled data based on a small
amount of labeled data. Furthermore, the work \cite{FerRau:J21} proposed
a spatial \ac{csi} process (channel chart) based on \ac{csi} data
without location labels, and calibrated the geographic coordinates
of the channel chart based on a small amount of location-labeled \ac{csi}
data. These approaches still require a certain amount of location-labeled
\ac{rss} measurement data, resulting in non-negligible construction
and calibration cost. Finally, there are some works on constructing
fingerprint database using unsupervised approaches. For example, the
work \cite{JunMoo:J15} recovered location labels by estimating the
\ac{rss} measurement trajectory using a \ac{hmm}-based method based
on a number of revisiting trajectories. Other attempts \cite{WuYan:J12,GaoHar:J17,ChoJeo:J22}
required mobility information obtained from additional sensors, including
inertial sensors, accelerometers, gyroscopes, and magnetometers, etc.

This paper proposes a {\em region-based radio map} for coarse indoor
localization, where the radio map is constructed via {\em unsupervised}
learning from \ac{rss} measurements {\em without} location labels.
For coarse localization, the indoor area is divided into several regions
of interest, and the target is localized to one of these regions.
It is important to note that such a localization problem arises in
various application scenarios, such as tracking equipment in a factory
and monitoring visitors in restricted areas, where having a coarse
accuracy is sufficient, but hardware cost and calibration cost are
the primary concerns. The region-based radio map consists of \ac{rss}
features associated with each region. Therefore, the fundamental question
is how to construct the radio map using unlabeled \ac{rss} data to
reduce the calibration cost.

To tackle this problem, we construct and explore some characteristics
for the \ac{rss} data. First, we assume that the data is mostly collected
sequentially from each region. This corresponds to a scenario where
a mobile device visits each region once, while the sensor network
collects the \ac{rss} of the signal emitted from the mobile. It is
important to note that the trajectory of the mobile and the timing
the mobile entering or leaving the region are not necessarily known
by the network. Therefore, such an assumption induces almost no calibration
cost. Second, we assume a subspace model for the \ac{rss} data, where
the \ac{rss} vector lies in a low-dimensional affine subspace that
varies across different regions.

With these two assumptions, the construction of a region-based radio
map can be formulated as a clustering problem using sequential data.
However, classical subspace clustering algorithms \cite{vidal:J11,ColAnt:J21,FanJic:J21}
were not optimized for \ac{rss} sequential data. Some works in a
broader domain investigated sequential data clustering for video segmentation
applications. For example, the work \cite{DuWan:C22} partitioned
the entire sequence into parts and grouped segments together. The
works \cite{TieGao:C14,Xia:J22,zhou:J22} learned the similarity between
samples with sequential priors and adopted graph-based clustering
methods. Furthermore, the work \cite{ZhoFu:C20} leveraged knowledge
from relevant labeled source sequential data to enhance the similarity
graph of unlabeled target sequential data. However, these methods
cannot be extended to clustering \ac{rss} data. This is because the
\ac{rss} data naturally appears as non-clustered due to multipaths
and noise, and even adjacent \ac{rss} data collected in the same
room can be divided into two clusters, despite the usage of the sequential
prior. Moreover, it is challenging to associate clusters with physical
regions using the irregular clustering results generated by these
methods. Thus, we need to address the following two major challenges:
\begin{itemize}
\item \textbf{How to cluster the unlabeled \ac{rss} measurement data?}
In practice, it is observed that the measured \ac{rss} fluctuates
significantly even within a small area. Therefore, clustering them
into groups using classical clustering approaches poses a challenge.
\item \textbf{How to match the clustered, yet unlabeled data to the physical
regions?} Since the sensor network cannot observe any location labels,
extracting location information remains a challenge, even if the \ac{rss}
data is perfectly clustered.
\end{itemize}
$\quad$In this paper, we formulate a maximum-likelihood estimation
problem with a sequential prior to cluster the \ac{rss} data. Consequently,
the clustering problem is transformed into a sequence segmentation
problem. Our preliminary work \cite{Xing:C22} attempted to solve
the segmentation problem using a gradient-type algorithm, but the
solution is prone to getting trapped at a poor local optimum. Here,
we introduce a merge-and-split algorithm that is to converge to a
globally optimal solution for a special case. Global convergence in
a general case is also observed in our numerical experiments. To match
the clusters with the physical regions, we construct a graph model
for a set of possible routes that may form the \ac{rss} sequential
data; such a model leads to a Viterbi algorithm for the region matching,
which can achieve a matching error of less than 1\%. To compare with
some existing unsupervised learning approaches, such as \ac{hmm}-based
approaches \cite{JunMoo:J15}, the proposed method does not require
a parametric propagation model for the \ac{rss}, which is usually
needed in a conventional \ac{hmm}, and hence, the proposed work is
algorithmically more stable as compared to \ac{hmm} which needs to
alternatively learn the model parameters and recovering the physical
trajectories.

To summarize, we make the following contributions:
\begin{itemize}
\item We develop an unsupervised learning framework to construct a region-based
radio map without location labels. The approach is based on solving
a subspace clustering problem with sequential prior.
\item We transform the clustering problem to a segmentation problem which
is solved by a novel merge-and-split algorithm. We establish optimality
guarantees for the algorithm under a special case.
\item We conduct numerical experiments using real measurements from an office
space and a larger area. It is found that the proposed unsupervised
scheme even achieves a better localization performance than several
supervised learning schemes which use location labels during the training,
including \ac{knn}, \ac{svm}, and \ac{dnn}.
\end{itemize}
$\quad$The remaining part of the paper is organized as follows. Section
\ref{sec:System-Model} introduces a signal subspace model for the
region-based radio map, a sequential data collection model, and a
probability model with a sequential prior. Section \ref{sec:Subspace-Feature}
develops the subspace feature solution, and the solution for matching
clusters to physical regions. Section \ref{sec:Clustering-via-Segmentation}
focuses on the development of the clustering algorithm and the potential
optimality guarantees. Experimental results are reported in Section
\ref{sec:Experiments} and the paper is concluded in Section \ref{sec:Conclusion}.

\section{System Model}

\label{sec:System-Model}

\subsection{Signal Subspace Model for the Region-based Radio Map}

Suppose that there are $D$ sensors with locations $\mathbf{z}_{j}\in\mathbb{R}^{2}$,
$j=1,2,\dots,D$, deployed in an indoor area. The sensors, such as
WiFi sensors, are capable of measuring the \ac{rss} of the signal
emitted by a wireless device, forming an \ac{rss} measurement vector
$\mathbf{x}\in\mathbb{R}^{D}$, although they may not be able to decode
the message of the device. Consider partitioning the indoor area into
$K$ non-overlapping regions. It is assumed that signals emitted from
the same region to share common feature due to the proximity of the
transmission location and the similarity of the propagation environment.
In practice, a room or a semi-closed space separated by large furniture
or walls can be naturally considered as a region, where the intuition
is that walls and furniture may shape a common feature for signals
emitted from a neighborhood surrounded by these objects. Given the
region partition, this paper focuses on extracting the large-scale
{\em feature} of each of the regions and building a region-based
radio map from the \ac{rss} measurements $\{\mathbf{x}\}$ without
location labels.

Assume that the \ac{rss} measurements $\mathbf{x}_{i}$ in decibel
scale taken in region $k$ satisfy
\begin{equation}
\mathbf{x}_{i}=\mathbf{U}_{k}\bm{{\theta}}_{i}+\mathbf{{\rm \bm{\mu}}}_{k}+\bm{\epsilon}_{i},\quad\forall i\in\mathcal{C}_{k}\label{eq:x1}
\end{equation}
where $\mathbf{U}_{k}\in\mathbb{R}^{D\times d_{k}}$ is a semi-unitary
matrix with $\mathbf{U}_{k}^{\text{T}}\mathbf{U}_{k}=\mathbf{I}$,
$\bm{\theta}_{i}\sim\mathcal{N}(\bm{0},\bm{\Sigma}_{k})$ is an independent
variable that models the uncertainty due to the actual measurement
location when taking the measurement sample $\mathbf{x}_{i}$ in region
$k$, $\bm{\Sigma}_{k}$ is assumed as a full-rank diagonal matrix
with non-negative diagonal elements, $\bm{\mu}_{k}\in\mathbb{R}^{D}$
captures the offset of the signal subspace, $\bm{\epsilon}_{i}\sim\mathcal{N}(\bm{0},s_{k}^{2}\mathbf{I}_{D\times D})$
models the independent measurement noise, and $\mathcal{C}_{k}$ is
the index set of the measurements $\mathbf{x}_{i}$ taken within the
$k$th region. As such, the parameters $\{\mathbf{U}_{k},\bm{\mu}_{k}\}$
specify an affine subspace with dimension $d_{k}$ for the noisy measurement
$\mathbf{x}_{i}$, and they are the \emph{feature} to be extracted
from the measurement data $\{\mathbf{x}_{i}\}$. Thus, region-based
radio map is a database that maps the $k$th physical region to the
signal subspace feature $\{\mathbf{U}_{k},\bm{\mu}_{k}\}$.
\begin{rem}
(Interpretation of the Subspace Model): As the user has $2$ spatial
degrees of freedom to move around in the region, the \ac{rss} vector
$\mathbf{x}$ may be modeled as a point moving on a two-dimensional
hyper-surface $\mathcal{S}$ embedded in $\mathbb{R}^{D}$. Therefore,
for a sufficiently small area, an affine subspace with dimension $d_{k}$$=2$
or $3$ can locally be a good approximation of $\mathcal{S}$.
\end{rem}

\subsection{Sequential Data Collection and the Graph Model}

\label{subsec:Sequential-Data-Collection-graph-model}

When the measurement location label sets are \emph{not} available,
it is very difficult to obtain the subspace feature $\{\mathbf{U}_{k},\bm{\mu}_{k}\}$.
While conventional subspace clustering algorithms, such as the \ac{em}
approach, are designed for recovering both the location label sets
and the subspace feature $\{\mathbf{U}_{k},\bm{\mu}_{k}\}$, they
may not work for the large noise case, which is a typical scenario
here as the \ac{rss} data has a large fluctuation due to the multipath
effect. To tackle this challenge, we consider a type of measurement
that provide some implicit structural information without substantially
increasing the effort on data collection.

We assume that the sequence of measurements $\mathbf{x}_{1},\mathbf{x}_{2},\dots,\mathbf{x}_{N}$
are taken along an arbitrary route that visits all the $K$ regions
without repetition.\footnote{Note that, our model and algorithm can accommodate some repetition, i.e., visiting a region multiple times, as long as the repetition is limited. This is accounted for by the noise term  in the subspace model (1).}
Recall that $\mathcal{C}_{k}$ is the collection of the measurements
collected from the $k$th region along the route, and therefore, for
any $i\in\mathcal{C}_{k}$ and $j\in\mathcal{C}_{k+1}$, we must have
$i<j$. Note that the exact route, the locations of the measurements,
the sojourn time that the mobile device spends in each region, and
the association between the sequential measurement set $\mathcal{C}_{k}$
and the $k$th physical region are \emph{unknown} to the system.

Nevertheless, we can model the eligibility of a specific route. A
route $\bm{\pi}$ is modeled as a permutation sequence with the first
$K$ natural numbers, where the $k$th element $\pi(k)$ refers to
the location label of the $k$th region along the route. Define a
graph $\mathcal{G}=(\mathcal{V},\mathcal{E})$ where each node in
$\mathcal{V}=\{1,2,\dots,K\}$ represents one of the $K$ regions,
and each edge $(k,j)\in\mathcal{E}$ represents that it is possible
to directly travel from region $k$ to region $j$ without entering
the other region. Therefore, a route $\bm{\pi}$ is eligible only
if there is an edge between adjacent nodes along the route, \emph{i.e.},
$(\pi(j),\pi(j+1))\in\mathcal{E}$, for $\forall1\leq j\leq K-1$.

\subsection{Probability Model with a Sequential Prior}

\label{subsec:Probability-Model-sequential-prior}

From the signal subspace model (\ref{eq:x1}), the conditional distribution
of $\mathbf{x}_{i}$, given that it belongs to the $k$th region,
is given by
\begin{align}
p_{k}(\mathbf{x};\bm{\Theta}) & =\frac{1}{(2\pi)^{D/2}|\mathbf{C}_{k}|^{1/2}}\nonumber \\
 & \qquad\times\mathrm{exp}\left(-\frac{1}{2}(\mathbf{x}-\mathbf{{\rm \bm{\mu}}}_{k})^{\text{{T}}}\mathbf{C}_{k}^{-1}(\mathbf{x}-\mathbf{{\rm \bm{\mu}}}_{k})\right)\label{eq:pkxi}
\end{align}
where $\mathbf{C}_{k}=\mathbf{U}_{k}\bm{\Sigma}_{k}\mathbf{U}_{k}^{\text{{T}}}+s_{k}^{2}\mathbf{I}$
is the conditional covariance matrix for the $k$th cluster, and $\bm{\Theta}=\{\mathbf{U}_{k},\bm{\Sigma}_{k},\mathbf{{\rm \bm{\mu}}}_{k},s_{k}^{2}\}_{k=1}^{K}$
is a shorthand notation for the collection of parameters.

Let $t_{k}$ be the last index of $\mathbf{x}_{i}$ before the device
leaves the $k$th region and enters the $(k+1)$th region. Consequently,
we have $t_{0}=0<t_{1}<t_{2}<\cdots<t_{K-1}<t_{K}=N$, and for each
$k=1,2,\dots,K$, all elements $i\in\mathcal{C}_{k}$ satisfy $t_{k-1}<i\leq t_{k}$.
For a pair of parameters $a<b$, an indicator function is defined
as
\begin{equation}
z_{i}(a,b)=\begin{cases}
\begin{array}{l}
1,\\
0,
\end{array} & \begin{array}{l}
a<i\leq b\\
\text{otherwise.}
\end{array}\end{cases}\label{eq:rectangle-window}
\end{equation}
As a result, the probability density function of measurement $\mathbf{x}_{i}$
can be given by 
\begin{equation}
p(\mathbf{x}_{i};\bm{\Theta},\bm{t})=\prod_{k=1}^{K}p_{k}(\mathbf{x}_{i};\bm{\Theta})^{z_{i}(t_{k-1},t_{k})}\label{eq:pdf}
\end{equation}
where $\bm{t}=(t_{1},t_{2},\dots,t_{K-1})$ is a collection of the
time indices of the segment boundaries. Note that, for each $i$,
$\sum_{k}z_{i}(t_{k-1},t_{k})=1$, and $z_{i}(t_{k-1},t_{k})=1$ only
under $t_{k-1}<i\leq t_{k}$. Therefore, given $i\in\mathcal{C}_{k}$,
equation (\ref{eq:pdf}) reduces to $p(\mathbf{x}_{i})=p_{k}(\mathbf{x}_{i})$.

Consider a log-likelihood cost function $\log\prod_{i=1}^{N}p(\mathbf{x}_{i};\bm{\Theta},\bm{\tau})$
which can be equivalently written as 
\begin{equation}
\mathcal{J}(\bm{\Theta},\bm{\tau})=\frac{1}{N}\sum_{i=1}^{N}\sum_{k=1}^{K}z_{i}(\tau_{k-1},\tau_{k})\log p_{k}(\mathbf{x}_{i};\bm{\Theta},\bm{\tau})\label{eq:J}
\end{equation}
where $\bm{\tau}=(\tau_{1},\tau_{2},\dots,\tau_{K-1})$ denotes an
estimator for $\bm{t}$. Throughout the paper, we implicitly define
$\tau_{0}=0,\tau_{K}=N$ for mathematical convenience.

It follows that $z_{i}(\tau_{k-1},\tau_{k})$ represents a rectangle
window that selects only the terms $\log p_{k}(\mathbf{x}_{i};\bm{\Theta},\bm{\tau})$
for $\tau_{k-1}<i\leq\tau_{k}$ and suppresses all the other terms.
Thus, it acts as a \emph{sequential prior} that selects a subset of
$\{\mathbf{x}_{i}\}$ in a row for $\mathcal{C}_{k}$.

While we have made an assumption from model (\ref{eq:x1}) that the
measurements are statistically independent, in practice, the measurements
$\mathbf{x}_{i}$ taken in the transient phase from one region to
the other may substantially deviate from both subspaces $\{\mathbf{U}_{k-1},\bm{\mu}_{k-1}\}$
and $\{\mathbf{U}_{k},\bm{\mu}_{k}\}$, leading to large modeling
noise $\bm{\epsilon}_{i}$. To down-weight the data possibly taken
in the transient phase, we extend the rectangle window model $z_{i}(a,b)$
in (\ref{eq:rectangle-window}) for the sequential prior to a smooth
window 
\begin{equation}
z_{i}(a,b)=\sigma_{\beta}\left(i-a\right)-\sigma_{\beta}\left(i-b\right)\label{eq:sigmoid_2}
\end{equation}
where $\sigma_{\beta}(x)$ is a sigmoid function that maps $\mathbb{R}$
to $[0,1]$ with a property that $\sigma_{\beta}(x)\to0$ as $x\to-\infty$
and $\sigma_{\beta}(x)\to1$ as $x\to+\infty$. The parameter $\beta$
controls the slope of the transition from $0$ to $1$. A choice of
a specific sigmoid function is $\sigma_{\beta}(x)=(1+e^{-(x-1/2)/\beta})^{-1}$.
For such a choice of $\sigma_{\beta}(x)$, the rectangle window (\ref{eq:rectangle-window})
is a special case of (\ref{eq:sigmoid_2}) for $\beta\to0$, where
one can easily verify that, for all $k$, we have $z_{i}(t_{k-1},t_{k})\to1$
for all $i\in\mathcal{C}_{k}$ and $z_{i}(t_{k-1},t_{k})\to0$ for
$i\notin\mathcal{C}_{k}$.

The subspace clustering problem with a sequential prior is formulated
as follows
\begin{align}
\underset{\bm{\Theta},\bm{\tau}}{\text{maximize}} & \quad\mathcal{J}(\bm{\Theta},\bm{\tau})\label{eq:P0}\\
\text{subject to} & \quad0<\tau_{1}<\tau_{2}<...<\tau_{K-1}<N.\label{eq:P01-constraint}
\end{align}

Note that an \ac{em}-type algorithm cannot solve (\ref{eq:P0}) due
to the sequential structure imposed by $z_{i}(\tau_{k-1},\tau_{k})$.
Our prior work \cite{Xing:C22} investigated a gradient approach,
but the iteration easily gets trapped at a bad local optimum even
under the simplest form of the model at $d_{k}=0$, as shown later
in the experiment section. In the rest of the paper, we establish
a new solution framework to solve (\ref{eq:P0}) and derive the conditions
for achieving the optimality.

\section{Subspace Feature Extraction and\\ Region Matching}

\label{sec:Subspace-Feature}

In this section, we focus on the solution $\bm{\Theta}$ when the
partition variable $\bm{\tau}$ in (\ref{eq:P0}) is fixed, and develop
a maximum likelihood estimator $\hat{\bm{\Theta}}(\bm{\tau})$ as
a function of $\bm{\tau}$ and the data. Then, we develop a method
to map the subspace feature $\{\mathbf{U}_{k},\bm{\mu}_{k}\}$ to
the physical region.

\subsection{Subspace Feature via Maximum-Likelihood \Ac{pca}}

\label{subsec:Subspace-Feature-via-ML}

The solution $\bm{\Theta}$ for a given $\bm{\tau}$ can be derived
as follows. First, from the conditional probability model (\ref{eq:pkxi}),
the maximizer of $\bm{\mu}_{k}$ to $\mathcal{J}$ in (\ref{eq:P0})
can be obtained by setting the derivative of $\mathcal{J}$ with respect
to $\bm{\mu}_{k}$ to zero, leading to the unique solution
\begin{equation}
\hat{\mathbf{{\rm \bm{\mu}}}}_{k}=\frac{1}{\sum_{i=1}^{N}z_{i}(\tau_{k-1},\tau_{k})}\sum_{i=1}^{N}z_{i}(\tau_{k-1},\tau_{k})\mathbf{x}_{i}.\label{eq:hat-mu-k}
\end{equation}

Then, denote $\mathbf{W}_{k}=\mathbf{U}_{k}\bm{\Sigma}_{k}^{1/2}$
for notational convenience. The critical point of $\mathcal{J}$ with
respect to the variable $\mathbf{W}_{k}$ is obtained by setting the
derivative 
\begin{equation}
\frac{\partial\mathcal{J}}{\partial\mathbf{W}_{k}}=\mathbf{C}_{k}^{-1}\mathbf{S}_{k}\mathbf{C}_{k}^{-1}\mathbf{W}_{k}-\mathbf{C}_{k}^{-1}\mathbf{W}_{k}\label{eq:deriv_Wk}
\end{equation}
to zero, leading to the equation 
\begin{equation}
\mathbf{S}_{k}\mathbf{C}_{k}^{-1}\mathbf{W}_{k}=\mathbf{W}_{k}\label{eq:equation-for-W}
\end{equation}
where 
\begin{equation}
\mathbf{S}_{k}=\frac{1}{\sum_{i=1}^{N}z_{i}(\tau_{k-1},\tau_{k})}\sum_{i=1}^{N}z_{i}(\tau_{k-1},\tau_{k})(\mathbf{x}_{i}-\hat{\bm{\mu}}_{k})(\mathbf{x}_{i}-\hat{\bm{\mu}}_{k})^{\text{{T}}}\label{eq:Sk}
\end{equation}
is the sample covariance of the data $\{\mathbf{x}_{i}\}$ that are
weighted by the $k$th sequential prior $z_{i}(\tau_{k-1},\tau_{k})$
for the $k$th subspace.

Recall that $\mathbf{C}_{k}=\mathbf{W}_{k}\mathbf{W}_{k}^{\text{T}}+s_{k}^{2}\mathbf{I}$,
and $\mathbf{W}_{k}^{\text{T}}\mathbf{W}_{k}=\bm{\Sigma}_{k}^{1/2}\mathbf{U}_{k}^{\text{T}}\mathbf{U}_{k}\bm{\Sigma}_{k}^{1/2}=\bm{\Sigma}_{k}$,
since $\mathbf{U}_{k}$ is semi-orthogonal and $\bm{\Sigma}_{k}$
is diagonal. We have the identity $\mathbf{C}_{k}^{-1}\mathbf{W}_{k}=\mathbf{W}_{k}(\bm{\Sigma}_{k}+s_{k}^{2}\mathbf{I})^{-1}$,
which can be easily verified by the relation $\mathbf{W}_{k}=\mathbf{U}_{k}\bm{\Sigma}_{k}^{1/2}$.
Hence, 
\begin{align}
\mathbf{S}_{k}\mathbf{C}_{k}^{-1}\mathbf{W}_{k} & =\mathbf{S}_{k}\mathbf{W}_{k}(\bm{\Sigma}_{k}+s_{k}^{2}\mathbf{I})^{-1}.\label{eq:SCW}
\end{align}
Using (\ref{eq:SCW}) and $\mathbf{W}_{k}=\mathbf{U}_{k}\bm{\Sigma}_{k}^{1/2}$,
equation (\ref{eq:equation-for-W}) becomes $\mathbf{S}_{k}\mathbf{U}_{k}\bm{\Sigma}_{k}^{1/2}(\bm{\Sigma}_{k}+s_{k}^{2}\mathbf{I})^{-1}=\mathbf{U}_{k}\bm{\Sigma}_{k}^{1/2}$,
which can be simplified to
\begin{equation}
\mathbf{S}_{k}\mathbf{U}_{k}=\mathbf{U}_{k}(\bm{\Sigma}_{k}+s_{k}^{2}\mathbf{I}).\label{eq:SU}
\end{equation}

It follows that (\ref{eq:SU}) is an eigenvalue problem, which can
be solved by constructing $\mathbf{U}_{k}$ as a collection of $d_{k}$
eigenvectors of $\mathbf{S}_{k}$. In addition, let $\lambda_{k,j}^{2}$
be the corresponding eigenvalue of the $j$th selected eigenvector
of\textbf{ }$\mathbf{S}_{k}$ for the construction of $\mathbf{U}_{k}$.
Then, the $j$th diagonal element of the diagonal matrix $\bm{\Sigma}_{k}$
can be set as $\sigma_{k,j}^{2}=\lambda_{k,j}^{2}-s_{k}^{2}$. It
will become clear as explained below that the best choice for constructing
$\mathbf{U}_{k}$ is to select the eigenvectors of $\mathbf{S}_{k}$
corresponding to the $d_{k}$-largest eigenvalues, and the best estimate
of $s_{k}^{2}$ is given by
\begin{align}
s_{k}^{2}=\frac{\sum_{j=d_{k}+1}^{D}\lambda_{k,j}^{2}}{D-d_{k}}.\label{eq:sk}
\end{align}

To see this, substituting $\hat{\bm{\mu}}_{k}$ from (\ref{eq:hat-mu-k})
and the solution $\mathbf{U}_{k}$ and $\bm{\Sigma}_{k}$ obtained
from (\ref{eq:SU}) to the log-likelihood function $\mathcal{J}(\bm{\Theta},\bm{\tau})$
in (\ref{eq:J}), we obtain
\begin{align}
\mathcal{J} & =-\frac{1}{2N}\sum_{k=1}^{K}\sum_{i=1}^{N}z_{i}(\tau_{k-1},\tau_{k})\Big\{ D\log(2\pi)+\log\prod_{j=1}^{d_{k}}\lambda_{k,j}^{2}\nonumber \\
 & \qquad\qquad+\log(s_{k}^{2(D-d_{k})})+\frac{1}{s_{k}^{2}}\sum_{j=d_{k}+1}^{D}\lambda_{k,j}^{2}+d_{k}\Big\}.\label{eq:sk-1}
\end{align}

\begin{figure}
	\begin{centering}
		\includegraphics[width=1\columnwidth]{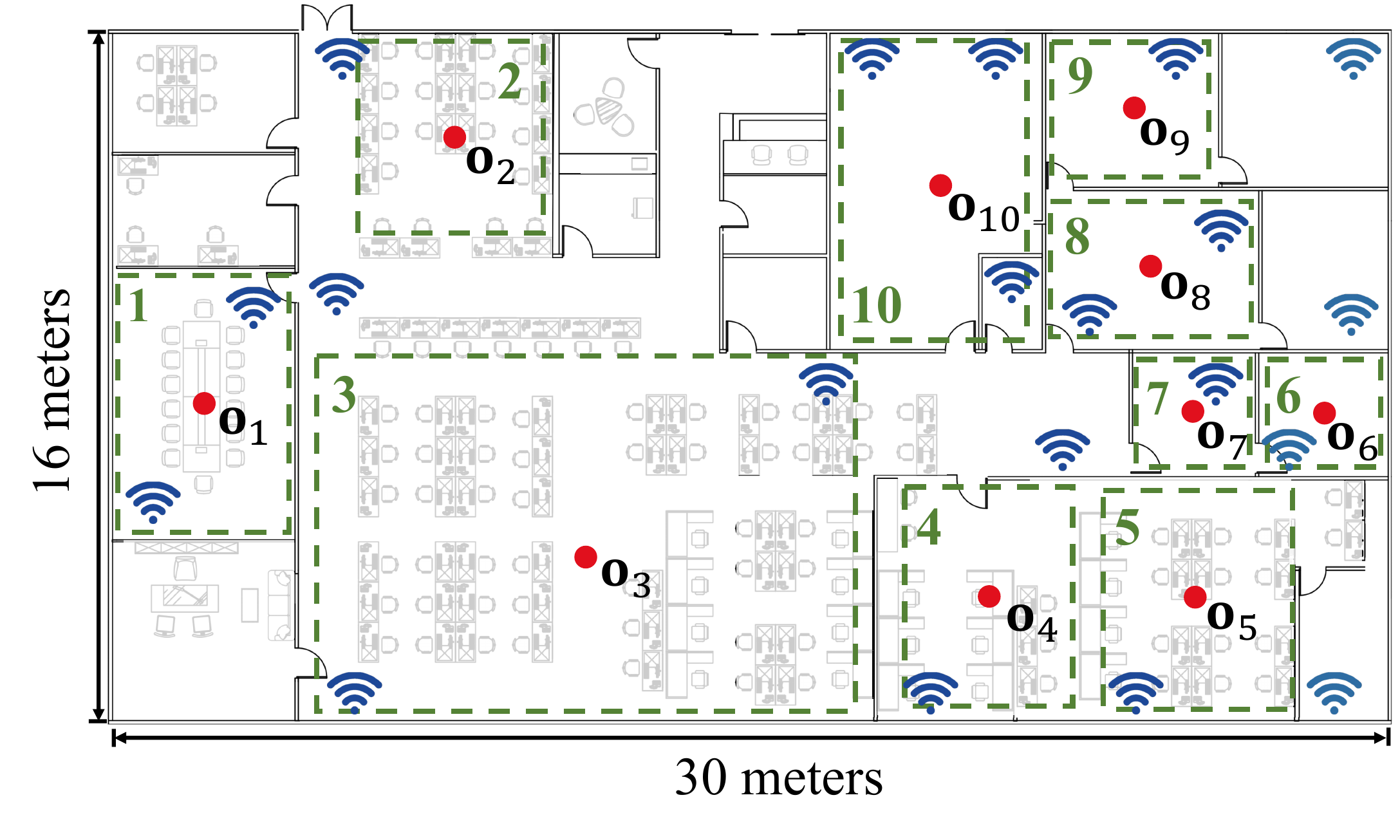}
		\par\end{centering}
	\caption{A $30\times16$ $\text{\text{m}}^{2}$\textbf{ }indoor area deployed
		with 21 sensors (blue icons) partitioned into 10 non-overlapping regions
		(green dashed rectangles), where the red circles refers to the reference
		region center.\label{fig:Ten-indoor-zones}}
\end{figure}

It has been shown in \cite{Tip:J99} that the maximizer to (\ref{eq:sk-1})
is obtained as the solution (\ref{eq:sk}) with $\lambda_{k,j}$,
$j=1,2,\dots,d_{k}$, being chosen as the $d_{k}$-largest eigenvalues
of $\mathbf{S}_{k}$. As such, we have obtained the solution $\bm{\Theta}$
from (\ref{eq:hat-mu-k}), (\ref{eq:SU})--(\ref{eq:sk}).

\subsection{Matching Clusters to Physical Regions}

\label{subsec:regionMatch}

Denote $\mathcal{D}_{k}\subseteq\mathbb{R}^{2}$ as the area of the
$k$th physical region. Our goal here is to map the subspace feature
$\{\mathbf{U}_{k},\bm{\mu}_{k}\}$ or the segment $(\tau_{k-1},\tau_{k})$
for the $k$th cluster of $\{\mathbf{x}_{i}\}$ to the physical region
$\mathcal{D}_{k}$. Mathematically, recall that the route $\bm{\pi}=(\pi(1),\pi(2),\dots,\pi(K))$
is modeled as a permutation sequence of the first $K$ natural numbers.
As a result, $\bm{\pi}$ is a function that matches the $k$th subspace
$\{\mathbf{U}_{k},\bm{\mu}_{k}\}$ to the physical region $\mathcal{D}_{\pi(k)}$,
for $1\leq k\leq K$.

Define $\mathbf{o}_{k}$ as the reference center location of the $k$th
region $\mathcal{D}_{k}$ as indicated by the red dots in Fig.~\ref{fig:Ten-indoor-zones}.
Thus, one essential idea is to link the clustered \ac{rss} measurements
$\mathbf{x}_{i}=(x_{i,1},x_{i,2},\dots,x_{i,D})$ with the location
topology of the sensors. Specifically, we employ the \ac{wcl} approach
to compute a reference location $\hat{\mathbf{o}}_{k}$ for the measurements
in the $k$th cluster $\mathcal{C}_{k}$:
\begin{equation}
\hat{\mathbf{o}}_{k}=\frac{1}{\left|\mathcal{C}_{k}\right|}\sum_{i\in\mathcal{C}_{k}}\frac{\sum_{j=1}^{D}w_{i,j}\mathbf{z}_{j}}{\sum_{j=1}^{D}w_{i,j}}\label{eq:ref-cent}
\end{equation}
where $w_{i,j}=(10^{x_{i,j}/10})^{\alpha}$ is the weight on the location
$\mathbf{z}_{j}$ of the $j$th sensor, $x_{i,j}$ is the \ac{rss}
of the $j$th sensor in the $i$th measurement $\mathbf{x}_{i}$,
and $\alpha$ is an empirical parameter typical chosen as \cite{PhoSon:J18,Pand:J22,TohEhs:C20,WanUrr:LJ11}.

Then, we exploit the {\em consistency property} that adjacent features
$\{\mathbf{U}_{k},\bm{\mu}_{k}\}$ and $\{\mathbf{U}_{k+1},\bm{\mu}_{k+1}\}$
should be mapped to physically adjacent regions $\mathcal{D}_{\pi(k)}$
and $\mathcal{D}_{\pi(k+1)}$, where $(\pi(k),\pi(k+1))\in\mathcal{E}$.
In Section \ref{subsec:Sequential-Data-Collection-graph-model}, we
have modeled the adjacency of any two physical regions using the graph
model $\mathcal{G}=(\mathcal{V},\mathcal{E})$ based on the layout
of the target area. With such a graph model, a graph-constrained matching
problem can be formulated as 
\begin{align}
\underset{\bm{\pi}}{\text{minimize}} & \quad\sum_{k=1}^{K}c(\hat{\mathbf{o}}_{\pi(k)},\mathbf{o}_{k})\label{eq:segmentation-based-matching}\\
\text{subject to} & \quad(\pi(j),\pi(j+1))\in\mathcal{E},\,\forall j=1,2,\dots,K-1\label{eq:constraint1}
\end{align}
where $c(\mathbf{p}_{1},\mathbf{p}_{2})$ is a cost function that
quantifies the difference between the two locations $\mathbf{p}_{1}$
and $\mathbf{p}_{2}$, for example, $c(\mathbf{p}_{1},\mathbf{p}_{2})=\|\mathbf{p}_{1}-\mathbf{p}_{2}\|_{2}$.
The constraint is to ensure that an eligible route only travels along
the edges of the graph $\mathcal{G}$ and the chosen region is non-repetitive.

Problem (\ref{eq:segmentation-based-matching}) can be solved by the
dynamic programming-based path searching strategy. For example, the
Viterbi algorithm can be applied to optimally solve (\ref{eq:segmentation-based-matching})
with a complexity of $\mathcal{O}(K^{3}N)$. To see this, any feasible
route $\bm{\pi}$ consists of $K$ moves. At each region, there are
at most $K-1$ candidate regions for the next move, and therefore,
there are at most $K(K-1)$ moves to evaluate in the Viterbi algorithm.
To evaluate the utility of each move, one computes (\ref{eq:segmentation-based-matching})
with a complexity of $\mathcal{O}(KN)$. This leads to the overall
complexity of $\mathcal{O}(K^{3}N)$ to solve (\ref{eq:segmentation-based-matching})
using the Viterbi algorithm.

\section{Clustering via Segmentation for\\ Sequential Data}

\label{sec:Clustering-via-Segmentation}

Based on the solution $\hat{\bm{\Theta}}(\bm{\tau})$ developed in
Section \ref{subsec:Subspace-Feature-via-ML}, it remains to maximize
$\mathcal{J}(\hat{\bm{\Theta}}(\bm{\tau}),\bm{\tau})$ over $\bm{\tau}$
subject to (\ref{eq:P01-constraint}), and the subspace clustering
problem (\ref{eq:P0}) becomes a segmentation problem. The main challenge
is the existence of multiple local maxima of $\mathcal{J}(\hat{\bm{\Theta}}(\bm{\tau}),\bm{\tau})$.

In this section, we first establish the optimality for a special case
of $d_{k}=0$ and $s_{k}^{2}=s^{2}$ for some $s>0$, which corresponds
to $K$-means clustering with a sequential prior. Based on this theoretical
insight, we then develop a robust algorithm to solve for the general
case of $d_{k}\geq0$.

For the special case of $d_{k}=0$ and $s_{k}^{2}=s^{2}$, the subspace
model (\ref{eq:x1}) degenerates to $\mathbf{x}_{i}=\bm{\mu}_{k}+\bm{\epsilon}_{i}$,
where $\bm{\mu}_{k}$ is the offset of the $k$th 0-dimensional subspace
and $\mathbf{U}_{k}=0$. The conditional probability (\ref{eq:pkxi})
becomes $p_{k}(\mathbf{x}_{i};\bm{\Theta})=\frac{1}{(2\pi)^{D/2}s}\exp(-\frac{1}{2s^{2}}\|\mathbf{x}_{i}-\bm{\mu}_{k}\|_{2}^{2})$.
Then, problem (\ref{eq:P0}) is equivalent to the following minimization
problem
\begin{align}
\underset{\bm{\Theta},\bm{\tau}}{\text{minimize}} & \quad\frac{1}{N}\sum_{i=1}^{N}\sum_{k=1}^{K}z_{i}(\tau_{k-1},\tau_{k})\|\mathbf{x}_{i}-\mathbf{{\rm \bm{\mu}}}_{k}\|_{2}^{2}\label{eq:P0-1}\\
\text{subject to} & \quad0<\tau_{1}<\tau_{2}<...<\tau_{K-1}<N\nonumber 
\end{align}
where the variable $\bm{\Theta}$ degenerates to a matrix $\bm{\Theta}=[\bm{\mu}_{1}\;\bm{\mu}_{2}\;\cdots\bm{\mu}_{K}]$
that captures the centers of the clusters. Substituting the solution
$\hat{\bm{\mu}}_{k}(\tau_{k-1},\tau_{k})$ in (\ref{eq:hat-mu-k})
to (\ref{eq:P0-1}), the cost function (\ref{eq:P0-1}) simplifies
to
\[
\tilde{f}(\bm{\tau})\triangleq\frac{1}{N}\sum_{i=1}^{N}\sum_{k=1}^{K}z_{i}(\tau_{k-1},\tau_{k})\big\|\mathbf{x}_{i}-\hat{\mathbf{{\rm \bm{\mu}}}}_{k}(\tau_{k-1},\tau_{k})\big\|_{2}^{2}
\]
which is a function \ac{wrt} $\bm{\tau}$.

\subsection{Asymptotic Property for Small $\beta$}

\label{subsec:Asymptotic-Property-for}

For each $\bm{\tau}$, it is observed from the solution $\hat{\bm{\mu}}_{k}(\tau_{k-1},\tau_{k})$
in (\ref{eq:hat-mu-k}) and the asymptotic property of the window
function $z_{i}(\cdot)$ in (\ref{eq:sigmoid_2}) that, as $\beta\to0$,
\begin{equation}
\hat{\mathbf{{\rm \bm{\mu}}}}_{k}(\tau_{k-1},\tau_{k})\to\tilde{\bm{\mu}}(\tau_{k-1},\tau_{k})\triangleq\frac{1}{\tau_{k}-\tau_{k-1}}\sum_{j=\tau_{k-1}+1}^{\tau_{k}}\mathbf{x}_{j}\label{eq:tilde-mu-k}
\end{equation}
uniformly for each $\tau_{k-1}$ and $\tau_{k}$ that satisfy $\tau_{k-1}+1\leq\tau_{k}$.

Based on $\tilde{\bm{\mu}}(\tau_{k-1},\tau_{k})$ in (\ref{eq:tilde-mu-k}),
define
\begin{align}
f_{k}(\tau_{k};\bm{\tau}_{-k}) & \triangleq\frac{1}{N}\sum_{i=\tau_{k-1}+1}^{\tau_{k+1}}\Big[\Big(1-\sigma_{\beta}(i-\tau_{k})\Big)\nonumber \\
 & \qquad\times\big\|\mathbf{x}_{i}-\tilde{\bm{\mu}}(\tau_{k-1},\tau_{k})\big\|_{2}^{2}\nonumber \\
 & \qquad+\sigma_{\beta}(i-\tau_{k})\big\|\mathbf{x}_{i}-\tilde{\bm{\mu}}(\tau_{k},\tau_{k+1})\big\|_{2}^{2}\Big]\label{eq:f_k}
\end{align}
for $k=1,2,\dots,K-1$, where $\bm{\tau}_{-k}\triangleq(\tau_{k-1},\tau_{k+1})$.
In addition, define
\begin{align}
f_{0}(\tau_{0}) & =\frac{1}{N}\sum_{i=1}^{\tau_{1}}\sigma_{\beta}(i-\tau_{0})\big\|\mathbf{x}_{i}-\tilde{\bm{\mu}}(\tau_{0},\tau_{1})\big\|_{2}^{2}\label{eq:f_0}\\
f_{K}(\tau_{K}) & =\frac{1}{N}\sum_{i=\tau_{K-1}+1}^{N}\big(1-\sigma_{\beta}(i-\tau_{K})\big)\nonumber \\
 & \qquad\times\big\|\mathbf{x}_{i}-\tilde{\bm{\mu}}(\tau_{K-1},\tau_{K})\big\|_{2}^{2}\label{eq:f_K}
\end{align}
for mathematical convenience. The cost function $\tilde{f}(\bm{\tau})$
can be asymptotically approximated as $\frac{1}{2}\sum_{k=0}^{K}f_{k}(\tau_{k};\bm{\tau}_{-k})$
as formally stated in the following result.
\begin{prop}
[Asymptotic Hardening]\label{prop:Asymptotic-Hardening} As $\beta\to0$,
we have 
\[
\tilde{f}(\bm{\tau})\to\frac{1}{2}\sum_{k=0}^{K}f_{k}(\tau_{k};\bm{\tau}_{-k})
\]
uniformly for every sequence $\bm{\tau}$ satisfying the constraint
(\ref{eq:P01-constraint}).
\end{prop}
\begin{proof}
See Appendix \ref{appendix:Prop1}.
\end{proof}
Proposition \ref{prop:Asymptotic-Hardening} decomposes $\tilde{f}(\bm{\tau})$
into sub-functions $f_{k}(\tau_{k};\bm{\tau}_{-k})$ that only depend
on a subset of data $\mathbf{x}_{i}$ from $i=\tau_{k-1}+1$ to $i=\tau_{k+1}$.

\subsection{Asymptotic Consistency for Large $N$}

While the function $f_{k}(\cdot)$ in (\ref{eq:f_k}) has been much
simplified from $\tilde{f}(\bm{\tau})$, it is a stochastic function
affected by the measurement noise in (\ref{eq:x1}). Consider a scenario
of large $N$ for asymptotically dense measurement in each region
$k$. Specifically, the total number of the sequential measurements
grows in such a way that the segment boundary indices $t_{1},t_{2},\dots,t_{K-1}$
grow with a constant ratio $t_{k}/N=\bar{\gamma}_{k}$ with respect
to $N$ as $N$ grows, and the measurements in each region are independent.
In practice, this may correspond to a random walk in each of the region
for a fix, but \emph{unknown}, portion of time $\bar{\gamma}_{k}$
as $N$ grows.

We consider a deterministic proxy for $f_{k}(\cdot)$ under large
$N$ defined as 
\begin{equation}
F_{k}(\tau_{k};\bm{\tau}_{-k})=\mathbb{E}\{f_{k}(\tau_{k};\bm{\tau}_{-k})\}\label{eq:Fk-def}
\end{equation}
where the expectation is over the randomness of the measurement noise
$\bm{\epsilon}_{i}$ in (\ref{eq:x1}). As a result, $F_{k}(\cdot)$
represents the cost in the noiseless case. In the remaining part of
the paper, we may omit the argument $\bm{\tau}_{-k}$ and write $f_{k}(\tau_{k})$
and $F_{k}(\tau_{k})$ for simplicity, as long as it is clear from
the context.

To investigate the asymptotic property for large $N$, we replace
$\tau_{k}$ with $\gamma_{k}N$ in (\ref{eq:f_k}) and (\ref{eq:Fk-def}).
We define $\bar{f}_{k}(\gamma_{k})=\underset{\beta\rightarrow0}{\mathrm{lim}}\enskip f_{k}(\gamma_{k}N)$,
and $\bar{F}_{k}(\gamma_{k})=\underset{\beta\rightarrow0}{\mathrm{lim}}\enskip F_{k}(\gamma_{k}N)$
for $\gamma_{k}\in\Gamma=\{i/N:i=1,2,...,N\}$. Denote $\gamma_{k}^{*}$
as the minimizer of $\bar{F}_{k}(\gamma_{k})$, and $\hat{\gamma}_{k}$
as the minimizer of $\bar{f}_{k}(\gamma_{k})$. Then, we have the
following result.
\begin{prop}
[Asymptotic Consistency]\label{prop:Asymptotic-Consistency} Suppose
that, for some $k$, there exists only one index $t_{j}\in\{t_{1},t_{2},\dots,t_{K-1}\}$
within the interval $(\tau_{k-1},\tau_{k+1})$. Then, it holds that
$\hat{\gamma}_{k}\rightarrow\gamma_{k}^{*}$ in probability as $N\to\infty$.
\end{prop}
\begin{proof}
See Appendix \ref{appendix:Prop2}.
\end{proof}
Proposition \ref{prop:Asymptotic-Consistency} implies that if $\hat{\tau}_{k}$
minimizes $f_{k}(\tau_{k})$, and $\tau_{k}^{*}$ minimizes $F_{k}(\tau_{k})$
as $\beta\rightarrow0$, then, we have $\hat{\tau}_{k}/N\rightarrow\tau_{k}^{*}/N$
as $N\rightarrow\infty$ and $\beta\rightarrow0$. Thus, the estimator
$\hat{\tau}_{k}$ obtained from $f_{k}(\cdot)$ and the solution $\tau_{k}^{*}$
obtained from $F_{k}(\cdot)$ are asymptotically consistent in a wide
sense. Therefore, we use $F_{k}(\tau_{k})$ as a deterministic proxy
for the stochastic function $f_{k}(\tau_{k};\bm{\tau}_{-k})$ and
we thus study the property of $F_{k}(\tau_{k})$.

\subsection{The Property of the Deterministic Proxy $F_{k}(\tau_{k})$}

We find the following properties for $F_{k}(\tau_{k})$.
\begin{prop}
[Unimodality]\label{prop:Unimodality} Suppose that, for some $k$,
there exists only one index $t_{j}\in\{t_{1},t_{2},\dots,t_{K-1}\}$,
within the interval $(\tau_{k-1},\tau_{k+1})$. Then, for any $\varepsilon>0$,
there exists a small enough $\beta$, and some finite constants $C_{1},C_{2},C_{1}',C_{2}'>0$
independent of $\varepsilon$, such that $F_{k}(\tau)-F_{k}(\tau-1)<\varepsilon C_{1}-C_{2}<0$,
for $\tau_{k-1}<\tau\leq t_{j}$, and $F_{k}(\tau)-F_{k}(\tau-1)>C_{2}'-\varepsilon C_{1}'>0$,
for $t_{j}<\tau<\tau_{k+1}$. In addition, $t_{j}$ minimizes $F_{k}(\tau)$
in $(\tau_{k-1},\tau_{k+1})$.
\end{prop}
\begin{proof}
See Appendix \ref{appendix:Prop3}-1.
\end{proof}
This result implies that, once the condition is satisfied, there exists
a unique local minima $t_{j}$ of $F_{k}(\tau)$ over $(\tau_{k-1},\tau_{k+1})$
.
\begin{prop}
[Flatness]\label{prop:flatness}Suppose that, for some $k$, there
is no index $t_{j}\in\{t_{1},t_{2},\dots,t_{K-1}\}$ in the interval
$(\tau_{k-1},\tau_{k+1})$. Then, for any $\varepsilon>0$, there
exists a small enough $\beta$ and a finite constant $C_{0}>0$ independent
of $\varepsilon$, such that $|F_{k}(\tau)-F_{k}(\tau-1)|<\varepsilon s^{2}C_{0}$
for all $\tau\in(\tau_{k-1},\tau_{k+1})$.
\end{prop}
\begin{proof}
See Appendix \ref{appendix:Prop4}-2.
\end{proof}
It follows that, when the interval $(\tau_{k-1},\tau_{k+1})$ is completely
contained in $(t_{j},t_{j+1})$ for some $j$, the function $F_{k}(\tau)$
appears as an almost flat function for $\tau\in(\tau_{k-1},\tau_{k+1})$
with only small fluctuation according to the noise variance $s^{2}$.
\begin{prop}
[Monotonicity near Boundary]\label{prop:monotonicity-near-boundary}
Suppose that, for some $k$, there are multiple partition indices
$t_{j},t_{j+1},\dots,t_{j+J}\in\{t_{1},t_{2},\dots,t_{K-1}\}$ within
the interval $(\tau_{k-1},\tau_{k+1})$. In addition, assume that
the vectors $\{\bm{\mu}_{k}\}$ are linearly independent. Then, for
any $\varepsilon>0$, there exists a small enough $\beta$, and some
finite constant $C_{3},C_{4},C_{3}',C_{4}'>0$, such that $F_{k}(\tau)-F_{k}(\tau-1)<\varepsilon C_{3}-C_{4}<0$
for $\tau_{k-1}<\tau\leq t_{j}$, and $F_{k}(\tau)-F_{k}(\tau-1)>C_{4}'-\varepsilon C_{3}'>0$
for $t_{j+J}<\tau<\tau_{k+1}$.
\end{prop}
\begin{proof}
See Appendix \ref{appendix:Prop5}-3.
\end{proof}
This result implies that the proxy cost function $F_{k}(\tau_{k})$
monotonically decreases in $(\tau_{k-1},t_{j}]$ and increases in
$[t_{j+J},\tau_{k+1})$.

Properties in Proposition~\ref{prop:Unimodality}--\ref{prop:monotonicity-near-boundary}
lead to a useful design intuition for the algorithm. First, if $(\tau_{k-1},\tau_{k+1})$
contains only one index $t_{j}\in\{t_{1},t_{2},\dots,t_{K-1}\}$,
then $t_{j}$ can be found by minimizing $F_{k}(\tau)$ in $(\tau_{k-1},\tau_{k+1})$.
Second, if $(\tau_{k-1},\tau_{k+1})$ contains none index $t_{j}$,
then there must be another interval $(\tau_{k'-1},\tau_{k'+1})$ containing
more than one indices $t_{j},t_{j+1},\dots,t_{j+J}$, and as a result,
by minimizing $F_{k'}(\tau)$ over $(\tau_{k'-1},\tau_{k'+1})$, the
solution satisfies $\tau_{k'}\in[t_{j},t_{j+J}]$. This intuition
leads to a successive merge-and-split algorithm derived as follows.

\subsection{A Merge-and-Split Algorithm under the Proxy Cost}

\begin{algorithm}
Initialize $\tau_{k}^{(0)}=\lfloor kN/K\rfloor$ for $k=0,1,2,\dots,K-1,K$.

Loop
\begin{itemize}
\item For each $k=1,2,\dots,K-1$,
\begin{enumerate}
\item Merge the adjacent clusters $\mathcal{C}_{k}^{(m)}$ and $\mathcal{C}_{k+1}^{(m)}$
to form a new set of clusters $\tilde{\mathcal{C}}_{1}^{(m,k)},\tilde{\mathcal{C}}_{2}^{(m,k)},\text{\ensuremath{\dots}},\tilde{\mathcal{C}}_{K-1}^{(m,k)}$.
\item For each $j=1,2,\dots,K-1$,
\begin{enumerate}
\item Split the $j$th set $\tilde{\mathcal{C}}_{j}^{(m,k)}$ into two subsets
to form the new clustering $\tilde{\mathcal{C}}_{1}^{(m,k,j)}$, $\tilde{\mathcal{C}}_{2}^{(m,k,j)}$,
$\dots,\tilde{\mathcal{C}}_{K}^{(m,k,j)}$ with the corresponding
segmentation indices $\bm{\tau}^{(m,k,j)}$;
\item Compute $F_{*}^{(m,k,j)}=\text{min}\{F_{j}(\tau;\bm{\tau}_{-j}^{(m,k,j)}):\tau_{j-1}^{(m,k,j)}<\tau<\tau_{j+1}^{(m,k,j)}\}$,
and denote the minimizer as $\tau_{j}^{*}$. The merge-and-split forms
a new segmentation $\hat{\bm{\tau}}^{(m,k,j)}=(\tau_{1}^{(m,k,j)},...,\tau_{j-1}^{(m,k,j)},\tau_{j}^{*},\tau_{j+1}^{(m,k,j)},...,\tau_{K-1}^{(m,k,j)})$;
\item Compute the cost reduction $\triangle F_{*}^{(m,k,j)}$ as in (\ref{eq:cost-reduction});
\end{enumerate}
\item Pick $j^{*}(k)\triangleq\arg\max_{j}\triangle F_{*}^{(m,k,j)}$.
\end{enumerate}
\item For all $\{\hat{\bm{\tau}}^{(m,k,j^{*}(k))},k=1,2,\dots,K-1\}$, pick
the one that minimizes $\sum_{k=1}^{K-1}F_{k}(\hat{\tau}_{k}^{(m,k,j^{*}(k))};\hat{\bm{\tau}}_{-k}^{(m,k,j^{*}(k))})$,
$k\in\{1,2,\dots,K-1\}$, and assign it as $\bm{\tau}^{(m+1)}$.
\end{itemize}
Repeat until $\tilde{F}(\bm{\tau}^{(m+1)})=\tilde{F}(\bm{\tau}^{(m)})$.

\caption{The merge-and-split algorithm\label{alg:merge-and-split}}
\end{algorithm}

Denote $\bm{\tau}^{(m)}$ as the segmentation variable from the $m$th
iteration, and $\mathcal{C}_{k}^{(m)}=\{\tau_{k-1}^{(m)}+1,\tau_{k-1}^{(m)}+2,\dots,\tau_{k}^{(m)}\}$
as the corresponding index set of the $k$th cluster based on the
segmentation variable $\bm{\tau}^{(m)}$. For the $(m+1)$th iteration,
the \textbf{merge} step picks a cluster $\mathcal{C}_{k}^{(m)}$,
$k\in\{1,2,...,K-1\}$, and merges it with the adjacent cluster $\mathcal{C}_{k+1}^{(m)}$,
forming in a new set of $K-1$ clusters $\tilde{\mathcal{C}}_{1}^{(m,k)},\tilde{\mathcal{C}}_{2}^{(m,k)},\dots,\tilde{\mathcal{C}}_{K-1}^{(m,k)}$;
algebraically, it is equivalent to removing the $k$th variable $\tau_{k}^{(m)}$,
resulting in a set of $K-2$ segmentation variables, denoted as an
$(K-2)$-tuple $\bm{\tau}^{(m,k)}=(\tau_{1}^{(m,k)},\tau_{2}^{(m,k)},\dots,\tau_{K-2}^{(m,k)})$.

In the \textbf{split} step, a cluster $\tilde{\mathcal{C}}_{j}^{(m,k)}$
is selected and split it into two, resulting in a new set of $K$
clusters $\tilde{\mathcal{C}}_{1}^{(m,k,j)}$, $\tilde{\mathcal{C}}_{2}^{(m,k,j)}$,
$\dots,\tilde{\mathcal{C}}_{K}^{(m,k,j)}$. The corresponding segmentation
indices are denoted as an $(K-1)$-tuple $\bm{\tau}^{(m,k,j)}=(\tau_{1}^{(m,k,j)},\tau_{2}^{(m,k,j)},...,\tau_{K-1}^{(m,k,j)})$.
Then, we minimize $F_{j}(\tau;\bm{\tau}_{-j}^{(m,k,j)})$ subject
to $\tau\in(\tau_{j-1}^{(m,k,j)},\tau_{j+1}^{(m,k,j)})$, and denote
the minimal value as $F_{*}^{(m,k,j)}$. In addition, denote the cost
reduction as 
\begin{equation}
\triangle F_{*}^{(m,k,j)}=F_{j}(\tau_{j-1}^{(m,k,j)};\bm{\tau}_{-j}^{(m,k,j)})-F_{*}^{(m,k,j)}\label{eq:cost-reduction}
\end{equation}
where $F_{j}(\tau_{j-1};\bm{\tau}_{-j}^{(m,k,j)})$ equals to cost
of not splitting $\tilde{\mathcal{C}}_{j}^{(m,k)}$.

As a result, $\triangle F_{*}^{(m,k,j)}$ corresponds to the cost
reduction of merging clusters $\mathcal{C}_{k}^{(m)}$ and $\mathcal{C}_{k+1}^{(m)}$
followed by an optimal split of the $j$th cluster after the merge.
Then, by evaluating the cost reduction for all $(K-1)^{2}$ possible
combinations of the merge-and-split, one can find the best segmentation
variable $\bm{\tau}^{(m+1)}$ that yields the least cost for the $(m+1)$th
iteration. The overall procedure is summarized in Algorithm~\ref{alg:merge-and-split}.

The complexity can be analyzed as follows. For each iteration, there
are $(K-1)^{2}$ merge-and-split operations. An exhaustive approach
to compute Step 2b) in Algorithm~\ref{alg:merge-and-split} requires
$\mathcal{O}(N/K)$ steps to enumerate all possible integer values
$\tau$ in the interval $(\tau_{j-1}^{(m,k,j)},\tau_{j+1}^{(m,k,j)})$,
while for each step, the function $F_{j}(\tau;\bm{\tau}_{-j}^{(m,k,j)})$
can be approximately computed by its stochastic approximation $f_{j}(\tau;\bm{\tau}_{-j}^{(m,k,j)})$
in (\ref{eq:f_k}), which requires a complexity of $\mathcal{O}(DN/K)$.
As a result, for each iteration, it requires a complexity of $\mathcal{O}(DN^{2})$.
For a benchmark, an exhaustive approach to solve (\ref{eq:P0-1})
requires $\mathcal{O}(N^{(K-1)})$ iterations and each iteration requires
a complexity of $\mathcal{O}(NK)$ to evaluate (\ref{eq:P0-1}). Thus,
the proposed merge-and-split is efficient.

\subsection{Convergence and Optimality}

\label{subsec:Convergence-and-Optimality}

Algorithm~\ref{alg:merge-and-split} must converge due to the following
two properties. First, the cost function is lower bounded by $0$
since it is the sum of squares. Second, define $\tilde{F}(\bm{\tau})\triangleq\sum_{k=1}^{K-1}F_{k}(\tau_{k};\bm{\tau}_{-k})$,
if $\tilde{F}(\bm{\tau}^{(m+1)})\neq\tilde{F}(\bm{\tau}^{(m)})$,
then the $m$th iteration must \emph{strictly} decrease the\textcolor{black}{{}
cost. Specifically, for each $k$, Step 3) in Algorithm~\ref{alg:merge-and-split}
guarantees $\triangle F_{*}^{(m,k,j^{*}(k))}\geq\triangle F_{*}^{(m,k,k)}$,
implying that $\tilde{F}(\hat{\bm{\tau}}^{(m,k,j^{*}(k))})\leq\tilde{F}(\bm{\tau}^{(m)})$
for all $k$; in addition, the output of the outer loop implies that
$\tilde{F}(\bm{\tau}^{(m+1)})\leq\tilde{F}(\hat{\bm{\tau}}^{(m,k,j^{*}(k))})$
for all $k$, and there must be at least one strict inequality for
$\tilde{F}(\bm{\tau}^{(m+1)})<\tilde{F}(\bm{\tau}^{(m)})$ if $\bm{\tau}^{(m)}\neq(t_{1},t_{2},...,t_{K-1})$,
which is proved in Appendix}~\textcolor{black}{\ref{appendix:Prop7}.}

To investigate the optimality of the converged solution from Algorithm~\ref{alg:merge-and-split},
consider the clusters $\tilde{\mathcal{C}}_{j}^{(m,k)}$ constructed
in Step 1) of Algorithm~\ref{alg:merge-and-split}. Recall that the
data $\{\mathbf{x}_{i}\}$ is clustered sequentially with segment
boundaries $t_{1},t_{2},\dots,t_{K-1}$. We have the following property
on the cost reduction $\triangle F_{*}^{(m,k,j)}$ in (\ref{eq:cost-reduction}).
\begin{lem}
[Cost Reduction]\label{lem:binary partition} Consider two distinct
clusters $\tilde{\mathcal{C}}_{j}^{(m,k)}$ and $\tilde{\mathcal{C}}_{j'}^{(m,k)}$
constructed from the $m$th iteration and the $k$th loop of Step
1) in Algorithm~\ref{alg:merge-and-split}. Suppose that there exists
at least one index $t_{k'}\in\{t_{1},t_{2},\dots,t_{K-1}\}$ in $\tilde{\mathcal{C}}_{j}^{(m,k)}$,
and no such $t_{k'}$ in $\tilde{\mathcal{C}}_{j'}^{(m,k)}$. Then,
for a sufficiently small $\beta$, $\triangle F_{*}^{(m,k,j)}>\triangle F_{*}^{(m,k,j')}$.
\end{lem}
\begin{proof}
See Appendix~\ref{appendix:Prop6}.
\end{proof}
Lemma \ref{lem:binary partition} can be intuitively understood from
Propositions \ref{prop:Unimodality} and \ref{prop:flatness}, which
suggest that $F_{j}(\tau;\bm{\tau}_{-j}^{(m,k,j)})$ is unimodal in
$(\tau_{j-1}^{(m,k,j)},\tau_{j+1}^{(m,k,j)})$, but $F_{j'}(\tau;\bm{\tau}_{-j'}^{(m,k,j')})$
is almost flat in $(\bm{\tau}_{j'-1}^{(m,k,j')},\bm{\tau}_{j'+1}^{(m,k,j')})$,
and hence, the former one has a larger potential to reduce the total
cost $\tilde{F}(\bm{\tau})=\sum_{k=1}^{K-1}F_{k}(\tau_{k};\bm{\tau}_{-k})$.
\begin{thm}
[Optimality]\label{prop:optimal} Algorithm~\ref{alg:merge-and-split}
terminates at $\bm{\tau}^{*}=(\tau_{1}^{*},\tau_{2}^{*},\dots,\tau_{K-1}^{*})$,
with $\tau_{k}^{*}=t_{k}$, $k=1,2,\dots,K-1$.
\end{thm}
\begin{proof}
See Appendix~\ref{appendix:Prop7}.
\end{proof}
As a result, Algorithm~\ref{alg:merge-and-split} can converge to
the globally optimal solution $t_{k}$ under the proxy cost $F_{k}(\cdot)$
for $d_{k}=0$ and $s_{k}^{2}=s^{2}$ despite the problem being non-convex.

\subsection{Merge-and-Split Clustering for Sequential Data}

\label{subsec:Seg-Opt-direct}

We now extend Algorithm~\ref{alg:merge-and-split} to the general
case for $d_{k}\geq0$. Recall from Proposition~\ref{prop:Asymptotic-Consistency}
that the segment boundary estimator $\hat{\tau}_{k}^{(N)}$ obtained
from minimizing $f_{k}(\tau_{k};\bm{\tau}_{-k})$ is asymptotically
consistent with the minimizer $\tau_{k}^{*}$ of the proxy cost $F_{k}(\tau_{k};\bm{\tau}_{-k})$
for asymptotically small $\beta$ and large $N$. Therefore, it is
encouraged to extend Algorithm~\ref{alg:merge-and-split} for minimizing
the actual cost $f_{k}(\tau;\bm{\tau}_{-k})$. It is clear that $f_{k}$,
which can be computed directly from the data $\{\mathbf{x}_{i}\}$,
is a stochastic approximation of the proxy function $F_{k}(\tau_{k};\bm{\tau}_{-k})$.

Specifically, based on the probability model (\ref{eq:pkxi}), for
$d_{k}\geq0$, we define
\begin{align}
\mathcal{F}_{k}(\bm{\Theta},\bm{\tau}) & \triangleq\frac{1}{N}\sum_{i=\tau_{k-1}+1}^{\tau_{k+1}}\Big[\big(1-\sigma_{\beta}(i-\tau_{k})\big)\Big(\mathrm{ln}|\mathbf{C}_{k}|\label{eq:new_F_k}\\
 & \qquad\qquad+(\mathbf{x}_{i}-\bm{\mu}_{k})^{\text{{T}}}\mathbf{C}_{k}^{-1}(\mathbf{x}_{i}-\bm{\mu}_{k})\Big)\nonumber \\
 & \qquad\qquad+\sigma_{\beta}(i-\tau_{k})\Big(\mathrm{ln}|\mathbf{C}_{k+1}|\nonumber \\
 & \qquad\qquad+(\mathbf{x}_{i}-\bm{\mu}_{k+1})^{\text{{T}}}\mathbf{C}_{k+1}^{-1}(\mathbf{x}_{i}-\bm{\mu}_{k+1})\Big)\Big]\nonumber 
\end{align}
for $k=1,2,\dots,K-1$, and
\begin{align*}
\mathcal{F}_{0}(\bm{\Theta},\bm{\tau}) & =\frac{1}{N}\sum_{i=1}^{\tau_{1}}\sigma_{\beta}(i-\tau_{0})\Big(\mathrm{ln}|\mathbf{C}_{1}|\\
 & \qquad\qquad+(\mathbf{x}_{i}-\bm{\mu}_{1})^{\text{{T}}}\mathbf{C}_{1}^{-1}(\mathbf{x}_{i}-\bm{\mu}_{1})\Big)\\
\mathcal{F}_{K}(\bm{\Theta},\bm{\tau}) & =\frac{1}{N}\sum_{i=\tau_{K-1}+1}^{N}\big(1-\sigma_{\beta}(i-\tau_{K})\big)\Big(\mathrm{ln}|\mathbf{C}_{K}|\\
 & \qquad\qquad+(\mathbf{x}_{i}-\bm{\mu}_{K})^{\text{{T}}}\mathbf{C}_{K}^{-1}(\mathbf{x}_{i}-\bm{\mu}_{K})\Big)
\end{align*}
in the same way as (\ref{eq:f_k})\textendash (\ref{eq:f_K}). Following
the same argument as in Proposition~\ref{prop:Asymptotic-Hardening},
it is observed that maximizing $\mathcal{J}(\bm{\Theta},\bm{\tau})$
in (\ref{eq:J}) is asymptotically equivalent to minimizing $\frac{1}{2}\sum_{k=0}^{K}\mathcal{F}_{k}(\bm{\Theta},\bm{\tau})$
for $\beta\to0$.

\begin{algorithm}
Initialize $\tau_{k}^{(0)}$ as the output of Algorithm~\ref{alg:merge-and-split}
by assuming $d_{k}=0$.

Loop for the $(m+1)$th iteration:
\begin{itemize}
\item Compute $\bm{\Theta}^{(m)}=\hat{\bm{\Theta}}(\bm{\tau}^{(m)})$ from
(\ref{eq:hat-mu-k}), (\ref{eq:SU})--(\ref{eq:sk});
\item Define the function $f_{k}(\tau_{k};\bm{\tau}_{-k})=\mathcal{F}_{k}(\bm{\Theta}^{(m)},\bm{\tau})$
by fixing $\bm{\Theta}^{(m)}$;
\item Implement one iteration in Algorithm~\ref{alg:merge-and-split} to
obtain the update $\bm{\tau}^{(m+1)}$ by replacing the function $F_{k}(\tau_{k};\bm{\tau}_{-k})$
in Algorithm~\ref{alg:merge-and-split} with $f_{k}(\tau_{k};\bm{\tau}_{-k})$
defined in the previous step.
\end{itemize}
Repeat until \footnotesize$\sum_{k=1}^{K-1}\mathcal{F}_{k}(\bm{\Theta}^{(m)},\bm{\tau}^{(m+1)})=\sum_{k=1}^{K-1}\mathcal{F}_{k}(\bm{\Theta}^{(m)},\bm{\tau}^{(m)})$.

\caption{Alternating optimization with merge-and-split \label{alg:alter-opt}}
\end{algorithm}

In Section \ref{subsec:Subspace-Feature-via-ML}, it has been shown
that, for a given segmentation variable $\bm{\tau}$, the solution
$\hat{\bm{\Theta}}(\bm{\tau})$ can be constructed from (\ref{eq:hat-mu-k}),
(\ref{eq:SU})--(\ref{eq:sk}). Therefore, the $f_{k}$ function
(\ref{eq:f_k}) studied in Section~\ref{subsec:Asymptotic-Property-for}
to Section~\ref{subsec:Convergence-and-Optimality} can be obtained
in a similar way as
\begin{equation}
f_{k}(\tau_{k};\bm{\tau}_{-k})=\mathcal{F}_{k}(\hat{\bm{\Theta}}(\bm{\tau}),\bm{\tau}).\label{eq:new_f_k}
\end{equation}
As a result, a direct extension of Algorithm~\ref{alg:merge-and-split}
to solve for the segmentation $\bm{\tau}$ is to replace the function
$F_{k}(\tau_{k};\bm{\tau}_{-k})$ in Algorithm~\ref{alg:merge-and-split}
with $f_{k}(\tau_{k};\bm{\tau}_{-k})$ defined in (\ref{eq:new_f_k}).
In the following, we call the extended version Algorithm~\ref{alg:merge-and-split}
for convenience.

Note that the complexity of evaluating (\ref{eq:new_f_k}) is $\mathcal{O}(D^{2}N+D^{3})$
because computing (\ref{eq:new_F_k}) requires $\mathcal{O}(D^{2}N/K)$
and computing $\hat{\bm{\Theta}}(\bm{\tau})$ requires $\mathcal{O}(D^{2}N+D^{3})$,
leading to a total complexity of $\mathcal{O}(D^{2}N^{2}K)$ per iteration
in Algorithm~\ref{alg:merge-and-split}.

To reduce the complexity, one may consider to \emph{alternatively}
update $\bm{\Theta}$ and $\bm{\tau}$ in $\mathcal{F}_{k}(\bm{\Theta},\bm{\tau})$.
Such an alternating optimization approach can be developed from Algorithm~\ref{alg:merge-and-split}
in a straight-forward way and is summarized in Algorithm~\ref{alg:alter-opt}.
As such, the per iteration complexity is reduced to $\mathcal{O}(D^{2}N+D^{3}+N^{2})$.

\section{Numerical Experiments}

\label{sec:Experiments}

Our experiments were conducted in an office space measuring 30 meters
by 16 meters (480 square meters), which is divided into 10 non-overlapping
regions as shown in Fig.~\ref{fig:Ten-indoor-zones}. There are 21
sensors installed as indicated by the blue icons, which are \ac{uwb}
sensors operating at 6 GHz frequency band. Measurement data was collected
in 3 different time periods spread across 3 different days to capture
the time-variation of the environment. During the measurement process
in each time period, a mobile device visited all the 10 regions once
without repetition, and the \ac{rss} value measured by the receivers
were recorded. The time interval of the collected data samples is
200 ms. The data collected on the first day was used for clustering
or training (for the baseline schemes), and the data collected on
the second and third days served as test datasets I and II, respectively.
The training dataset comprises 1,455 records, while the two testing
datasets contain 1,294 and 1,055 records, respectively.

We also expanded our experiments to a larger area (53\texttimes 55
m\texttwosuperior ) with 24 regions of interest, including corridors
and rooms with various layouts. A total of 50 sensors are installed
in the whole area. The samples were collected every 200 ms, and the
number of RSS data samples collected in each region ranges from 2,000
to 5,000. The data collected on the first day was used for clustering
or training, and the data collected on the second day was used as
test dataset III.

\subsection{Verification of the Subspace Model from Real Data}

\begin{figure}
\begin{centering}
\includegraphics[width=1\columnwidth]{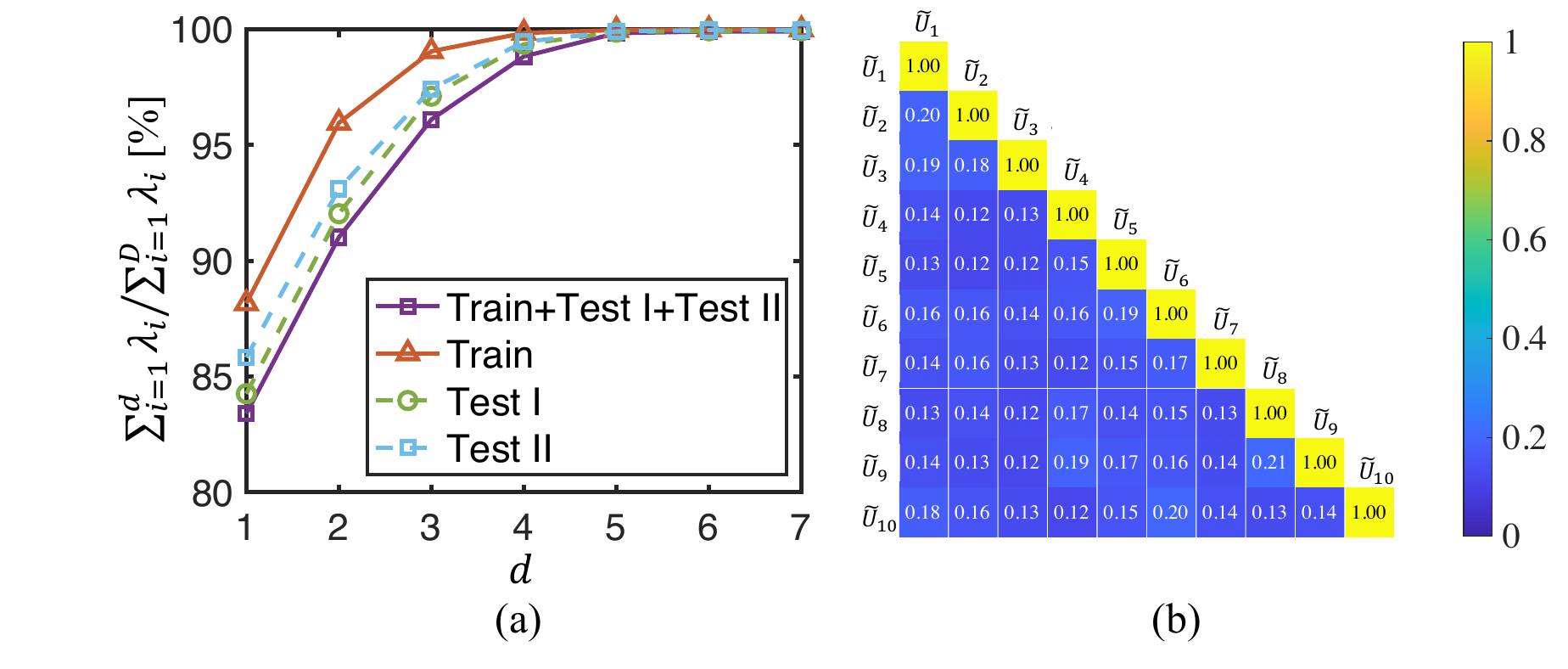}
\par\end{centering}
\caption{a) Percentage of the sum of the first $d$ dominant eigenvalues of
the covariance matrix of the data collected from Region 1. (b) Similarity
scores $[0,1]$ between subspaces, where 0 means the two subspaces
are orthogonal and 1 means identical. \label{fig:low-rank}}
\end{figure}

We first justify the subspace model (1) from real data. Specifically,
we need to verify that the covariance of the data $\mathbf{x}_{i}$
has a low rank structure, and moreover, the subspaces are non-identical
across regions.

In Fig. \ref{fig:low-rank} (a), we extract the subset of data $\{\mathbf{x}_{i},i\in\mathcal{C}_{1}\}$
collected from Region 1 illustrated in Fig. \ref{fig:Ten-indoor-zones},
and compute the corresponding covariance matrix $\hat{\mathbf{C}}_{1}$.
Denote the eigenvalues of $\hat{\mathbf{C}}_{1}$ as $\lambda_{1},\lambda_{2},...,\lambda_{D}$.
Fig. \ref{fig:low-rank} (a) shows the percentage of the sum of the
first $d$ dominant eigenvalues, i.e., $\sum_{i=1}^{d}\lambda_{i}/\sum_{i=1}^{D}\lambda_{i}$.
It is observed that, the top 3 principal components already contains
up to 97.0\% of energy. This indicates that the data collected from
Region 1 exhibits low-rank characteristics. A similar observation
can be made from the data collected from other regions. Hence, a subspace
model can describe the data with a good accuracy.

Fig. \ref{fig:low-rank} (b) plots the similarity between affine subspaces.
Specifically, the affine subspace model (\ref{eq:x1}) can also be
written as a linear subspace model form $\mathbf{x}_{i}=\tilde{\mathbf{U}}_{k}\left[\bm{{\theta}}_{i},\|\tilde{\mathbf{{\rm \bm{\mu}}}}{}_{k}\|_{2}\right]^{\mathrm{T}}+\bm{\epsilon}_{i}$,
where $\tilde{\mathbf{U}}_{k}=\left[\mathbf{U}_{k},\tilde{\mathbf{{\rm \bm{\mu}}}}{}_{k}/\|\tilde{\mathbf{{\rm \bm{\mu}}}}{}_{k}\|_{2}\right]$
is the subspace basis of the linear subspace, and $\tilde{\mathbf{{\rm \bm{\mu}}}}{}_{k}=(\mathbf{I}_{D}-\mathbf{U}_{k}\mathbf{U}_{k}^{\mathrm{T}})\mathbf{{\rm \bm{\mu}}}{}_{k}$
is the offset of the $k$th affine subspace satisfying orthogonality
to the columns of $\mathbf{U}_{k}$. The similarity between the linear
subspaces $\tilde{\mathbf{U}}_{i}$ and $\tilde{\mathbf{U}}_{j}$
is calculated by trace$\{\tilde{\mathbf{U}}_{i}\tilde{\mathbf{U}}_{i}^{\mathrm{T}}\tilde{\mathbf{U}}_{j}\tilde{\mathbf{U}}_{j}^{\mathrm{T}}\}/\mathrm{min}(d_{i}+1,d_{j}+1)$.
It is observed that the subspaces are nearly orthogonal to each other.

\subsection{Clustering Performance}

\begin{table}
\caption{Comparison on the clustering accuracy. \label{tab:Clustering-Performance-Evaluatio}}

\centering{}%
\begin{tabular}{l>{\centering}p{1cm}>{\centering}p{1cm}>{\centering}p{1cm}>{\centering}p{1cm}>{\centering}p{1cm}}
\hline 
Methods & Acc & NMI & F1 & ARI & P\tabularnewline
\hline 
ESC \cite{vidal:J11} & 48.9 & 62.7 & 49.7 & 42.2 & 70.9\tabularnewline
SCS \cite{Bai:J22} & 68.0 & 81.7 & 68.4 & 64.5 & 82.5\tabularnewline
LSE \cite{Xia:J22} & 80.6 & 87.6 & 80.8 & 79.2 & 88.4\tabularnewline
Algorithm~\ref{alg:merge-and-split} & 98.5 & 96.8 & 97.0 & 96.6 & 97.0\tabularnewline
Algorithm~\ref{alg:alter-opt} & 98.5 & 96.8 & 97.0 & 96.6 & 97.0\tabularnewline
\hline 
\end{tabular}
\end{table}

Four clustering metrics are used: clustering accuracy (Acc), normalized
mutual information (NMI), F-Score (F1), adjusted rand index (ARI),
and precision (P) \cite{Bai:J22}. We compare our method with a classical
subspace clustering method: \ac{em}-based subspace clustering (ESC)
\cite{vidal:J11} that iterates between clustering to subspace and
estimating the subspace model under a probabilistic \ac{pca} structure.
We also compare our method to two recently developed clustering algorithms
reported in the literature: self-constrained spectral clustering (SCS)
\cite{Bai:J22}, and local self-expression subspace clustering (LSE)
\cite{Xia:J22}. SCS is an extension of spectral clustering that incorporates
pairwise and label self-constrained terms into the objective function
of spectral clustering to guide the clustering process using some
prior information. LSE incorporates a variational autoencoder (VAE)
neural network with a temporal convolution module to construct a self-expression
affinity matrix with temporally consistent priors, followed by a conventional
graph-based clustering algorithm. The local self-expression layer
in the temporal convolution module only maintains representation relations
with temporally adjacent data to implement the local validity of self-expression.

Table~\ref{tab:Clustering-Performance-Evaluatio} summarizes the
comparison of clustering performance. The classical subspace clustering
algorithm ESC performs poorly due to significant fluctuations in \ac{rss}
measurements. The proposed algorithms, and LSE both outperform SCS,
highlighting the limitations of non-subspace clustering algorithms
when applied to high-dimensional sequential data clustering problems.
Although LSE outperforms SCS by employing high-dimensional feature
extraction with temporal consistency priors, the proposed algorithms
still performs better than LSE. This is because LSE focuses solely
on extracting smooth data features through the construction of a self-expression
matrix using a temporal convolution module, without considering temporal
consistency in identifying cluster structures. The proposed schemes
perform the best among all the schemes and achieve a 98.5\% clustering
accuracy. Note that in real datasets, there exists a transition phase
when a device moves from one region to another. Data in this transition
phase may not belong to any specific subspace, which could contribute
to the 1.5\% clustering error observed in the proposed algorithm.

\subsection{Convergence and Verification of the Theoretical Results}

In practice, the mobile device has non-negligible transition from
one region to another, and therefore, the exact timing at the region
boundary is not well-defined. We thus define the $\varepsilon$-tolerance
error 
\[
E_{\varepsilon}=\frac{1}{N}\sum_{k=1}^{K-1}\mathrm{max}\{|\tau_{k}-t_{k}|-\varepsilon N,0\}
\]
to evaluate the convergence of the algorithms. Here, global optimality
is claimed when $E_{\varepsilon}=0$.

The proposed Algorithms~\ref{alg:merge-and-split} and \ref{alg:alter-opt}
and their variants, marked as \textquotedbl R\textquotedbl , are
evaluated. The \textquotedbl R\textquotedbl{} version of the proposed
algorithms are randomly initialized following a uniform distribution
for $\tau_{k}^{(0)}$ in the first line of Algorithm~\ref{alg:merge-and-split}.
Algorithm~\ref{alg:alter-opt} is initialized by running Algorithm~\ref{alg:merge-and-split}
for 15 iterations, which have been counted in the total iterations
of Algorithm~\ref{alg:alter-opt} in Fig.~\ref{fig:converge}. In
Fig.~\ref{fig:converge} (a), the cost value reduction of Algorithm~\ref{alg:alter-opt}
tends to saturate at the 15th iteration (initial phase of Algorithm~\ref{alg:alter-opt}),
but the cost value continues to quickly decreases starting from the
$16$th iteration (main loop of Algorithm~\ref{alg:alter-opt}).
The convergence is benchmarked with a gradient-based subspace clustering
(GSC) method \cite{Xing:C22}, which employs stochastic gradient descent
to search for $\bm{\tau}$.

\begin{figure}[t]
	\begin{centering}
		\includegraphics[width=1\columnwidth]{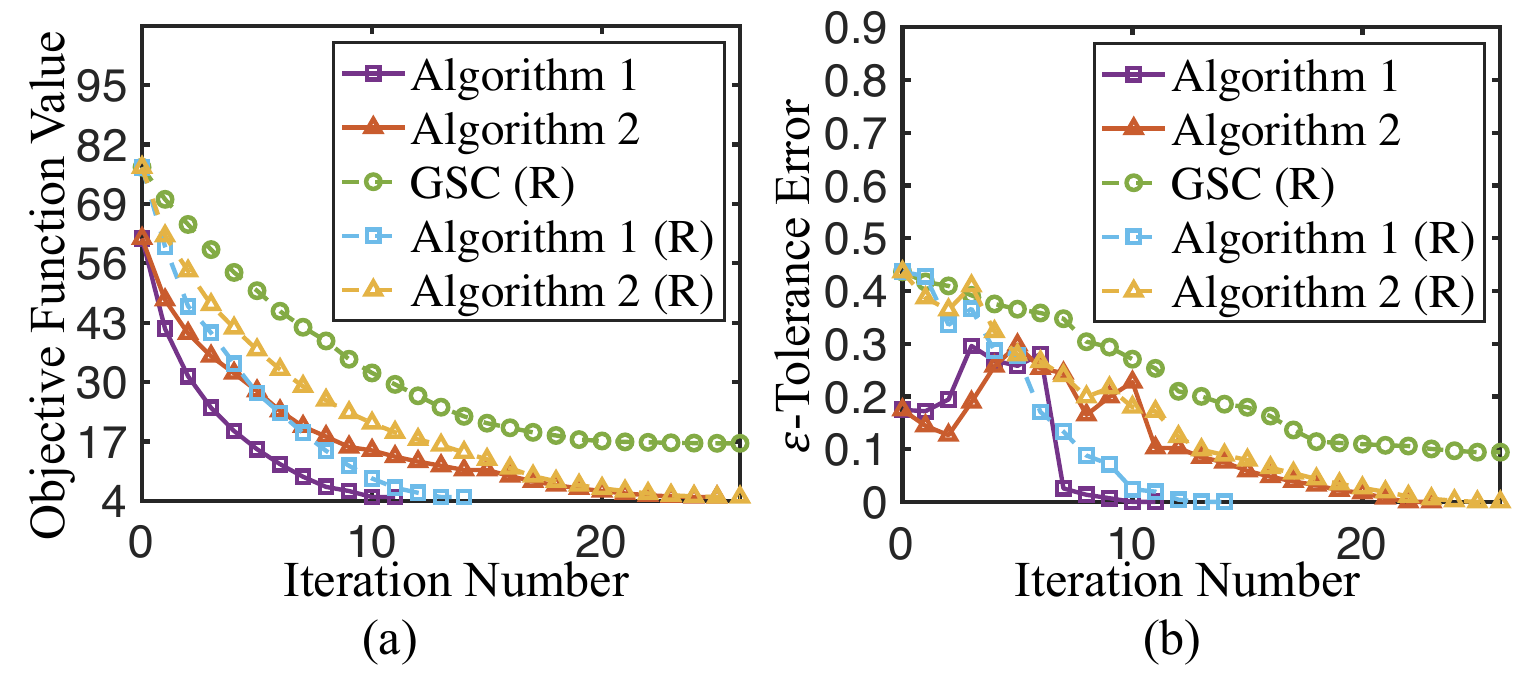}
		\par\end{centering}
	\caption{Convergence of the proposed algorithms, where the curves marked with
		\textquotedbl R\textquotedbl{} represents the mean trajectories for
		20 independent random initializations.\label{fig:converge}}
\end{figure}
Fig.~\ref{fig:converge} shows the objective function value $\sum_{k=1}^{K-1}\mathcal{F}_{k}(\hat{\bm{\Theta}}(\bm{\tau}),\bm{\tau})$
and the $\varepsilon$-tolerance error $E_{\varepsilon}$ against
the iteration number, where $\varepsilon$ is chosen as $\varepsilon=0.3$\%.
Both Algorithms~\ref{alg:merge-and-split} and \ref{alg:alter-opt},
as well as their variants with random initialization, converge to
the globally optimal solution. By contrast, the GSC baseline is not
guaranteed to converge to $E_{\varepsilon}=0$, as it is easily trapped
at poor local optimum. In addition, Algorithm~\ref{alg:merge-and-split}
requires fewer iterations than Algorithm~\ref{alg:alter-opt}, but
it has a higher computational complexity per iteration. Specifically,
for the dataset we used, the total computational time of Algorithm~\ref{alg:merge-and-split}
is 1033 seconds, whereas, that of Algorithm~\ref{alg:alter-opt}
is 245 seconds, 4X faster than Algorithm~\ref{alg:merge-and-split}.
Nevertheless, Algorithm~\ref{alg:merge-and-split} is guaranteed
to globally converge in a special case as stated in Theorem~\ref{prop:optimal}.

Next, we verify the theoretical properties in Propositions~\ref{prop:Unimodality}--\ref{prop:monotonicity-near-boundary}
through two numerical examples. In Fig.~\ref{fig:fk} (a), a simulated
dataset is constructed based on the subspace model (\ref{eq:x1})
with $D=40$, $d_{k}=0$, and $K$ clusters. The ratio $\|\bm{\mu}_{i}-\bm{\mu}_{j}\|_{2}^{2}/s^{2}$
of the squared-distance between the cluster centers over the noise
variance $s^{2}$ is set as 2.5. The data is segmented into two parts
by $\tau_{1}$ and the cost function $f_{1}(\tau_{1})$ is plotted.
First, as the number of samples $N$ increases, the cost function
$f_{1}(\tau_{1})$ eventually becomes the deterministic proxy $F_{1}(\tau_{1})=\mathbb{E}\{f_{1}(\tau_{1})\}$
as shown by the group of curves for $K=2$ clusters. Second, $f_{1}(\tau_{1})$
appears as unimodal for $K=2$ clusters under large $N$, which agrees
with the results in Proposition~\ref{prop:Unimodality}. Third, for
$K=1$ cluster, $f_{1}(\tau_{1})$ is an almost flat function under
large $N$ as discussed in Proposition~\ref{prop:flatness}. For
$K=3$ clusters under large $N$, $f_{1}(\tau_{1})$ appears as monotonic
near the left and right boundaries as discussed in Proposition~\ref{prop:monotonicity-near-boundary}.

In Fig.~\ref{fig:fk} (b), the experiment is extended to real data.
We extract a subset of data samples belonging to $K=1,2,3$ consecutive
clusters from the measurement dataset, and plot the cost function
$f_{1}(\tau_{1})$ under parameter $d_{k}=2$ and different $\beta$
values. It is observed that the cost function appears as unimodal
disturbed by noise for $K=2$ clusters, which is consistent with the
results in Proposition~\ref{prop:Unimodality}. This also provides
a justification that the proposed subspace model is essentially accurate
in real data. In addition, while a small $\beta$ may help amplify
the unimodality of the cost function as implied by Proposition~\ref{prop:Unimodality},
the cost function is prone to be disturbed by the modeling noise,
resulting in multiple local minimizers. On the contrary, a medium
to large $\beta$ may help eliminate the noise for a unique local
minimizer.

\subsection{Cluster to Physical Region Matching}
\begin{table}[t]
	\caption{Average matching error ($\times100$\%) performance of our subspace
		matching. \label{tab:Performance-of-matching}}
	
	\centering{}%
	\begin{tabular}{>{\raggedright}p{0.9cm}>{\centering}p{0.3cm}>{\centering}p{0.3cm}>{\centering}p{0.3cm}>{\centering}p{0.3cm}>{\centering}p{0.3cm}>{\centering}p{0.3cm}>{\centering}p{0.3cm}>{\centering}p{0.3cm}>{\centering}p{0.3cm}>{\centering}p{0.3cm}}
		\hline 
		$|\mathcal{E}|$ & 18 & 19 & 20 & 21 & 22 & 23 & 24 & 25 & 26 & 27\tabularnewline
		\hline 
		$\alpha=1$ & 0.3 & 0.5 & 1.0 & 1.3 & 1.4 & 1.5 & 2.0 & 1.9 & 2.8 & 1.9\tabularnewline
		$\alpha=2$ & 0.3 & 0.5 & 1.0 & 2.5 & 3.0 & 3.3 & 3.8 & 3.0 & 3.3 & 3.8\tabularnewline
		$\alpha=4$ & 0.5 & 1.0 & 0.8 & 2.5 & 4.0 & 3.8 & 4.0 & 4.5 & 4.0 & 3.5\tabularnewline
		\hline 
	\end{tabular}
\end{table}
\begin{figure}[t]
	\begin{centering}
		\includegraphics[width=1\columnwidth]{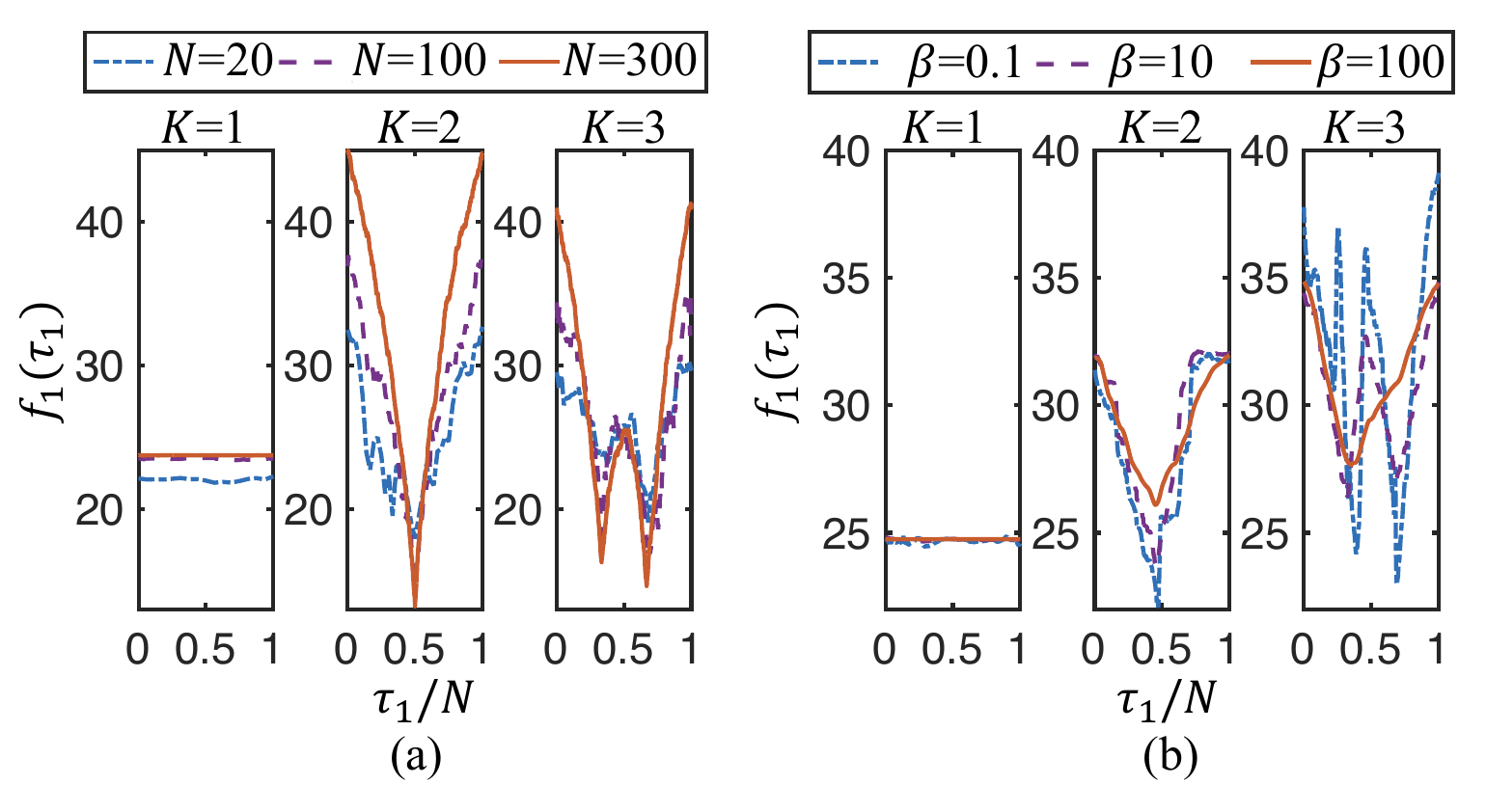}
		\par\end{centering}
	\caption{Verification of the unimodality, the flatness, and the monotonicity
		near boundary of the cost function. \label{fig:fk}}
\end{figure}

The proposed cluster-to-region matching is based on a graph $\mathcal{G}$,
which may be generated from the floor plan in practice. For evaluation
purpose here, we generate a set of graphs by randomly assigning edges
between regions. The probability of an edge joining two regions $j$
and $k$ is $q_{jk}=C_{\text{e}}\cdot\mathrm{exp}(-\|\mathbf{o}_{j}-\mathbf{o}_{k}\|_{2}^{2})$,
{\em i.e.}, the smaller the distance, the higher the probability,
where $C_{\text{e}}$ is a normalizing factor for the expected number
of edges in the graph, and $\mathbf{o}_{j}$ and $\mathbf{o}_{k}$
are the reference locations of the $j$th and $k$th region, respectively,
as shown in Fig.~\ref{fig:Ten-indoor-zones}. The matching error
is computed as $E_{\text{m}}=\frac{1}{K}\sum_{k}\mathbb{I}\{\pi^{*}(k)\neq\pi(k)\}$
where $\bm{\pi}^{*}$ is the desired matching.

Table~\ref{tab:Performance-of-matching} summarizes the matching
error for different graphs. It is observed that the fewer the edges,
the lower the matching error. This is because the constraint set in
(\ref{eq:constraint1}) is smaller for fewer edges. For the graphs
with 18 edges, the average number of eligible routes satisfying constraint
(\ref{eq:constraint1}) is 287 in our region topology; for the graphs
with 27 edges, that number is above 13,000. Nevertheless, the overall
matching error is mostly below 3\% under parameter $\alpha=1$ and
below 1\% for $|\mathcal{E}|\leq20$ under all parameter values $\alpha=1,2,4$
when computing the reference centroid in (\ref{eq:ref-cent}). In
addition, we compare the performance under different parameters $\alpha$.
It is shown that the performance is not sensitive with $\alpha$,
and the overall matching error is mostly below 4\%.

\subsection{Localization Performance}

We evaluate the localization performance of the region-based radio
map using test datasets I and II, which have not been used for the
radio map construction. A maximum-likelihood approach is used based
on the conditional probability function (\ref{eq:pkxi}), and the
estimated region $\hat{k}$ given the \ac{rss} measurement vector
$\mathbf{x}$ is given by 
\begin{equation}
\hat{k}=\underset{k\in\{1,2,\dots,K\}}{{\rm argmax}}\:p_{k}(\mathbf{x};\bm{\Theta}).\label{eq:region-assign}
\end{equation}

We compare the localization performance with two unsupervised schemes,
max-\ac{rss} (MR), which picks the location of the sensor that observes
the largest \ac{rss} as the target location, and \ac{wcl}\cite{PhoSon:J18},
which estimates the location as (\ref{eq:ref-cent}). For performance
benchmarking, we also evaluate three supervised localization approaches
\ac{knn}\cite{SadSeb:J20,Hu:J18}, \ac{svm}\cite{Wu:J07}, and \ac{dnn}\cite{Hao:J21,CheChu:J22},
which are trained using the training set with region labels that were
not available to the proposed scheme. The parameters of baseline methods
are determined and tuned using a ten-fold cross validation. For \ac{knn},
the optimal number of neighbors was found to be 8. A Gaussian kernel
was used for \ac{svm}. For \ac{dnn}, we adopt a three layer \ac{mlp}
neural network with 30 nodes in each layer to train the localization
classifier. The parameter $\beta$ in (\ref{eq:P0}) is set to be
1. The subspace dimension $d_{k}$ is chosen from 1 to 3 according
to the number of sensors located in the region as shown in Fig.~\ref{fig:Ten-indoor-zones}.
The performance is evaluated using the mean {\em region localization error}
defined as $\mathbb{E}\{\|\mathbf{o}_{\hat{k}}-\mathbf{o}_{k}\|\}$,
where $\mathbf{o}_{k}$ is the reference location of the $k$th region.

\begin{table}
\begin{centering}
\caption{Region localization error {[}meters{]} on 10-region Test Dataset I
and II.\label{tab:Space-classification-performance}}
\par\end{centering}
\centering{}%
\begin{tabular}{>{\raggedright}m{2.7cm}|>{\centering}p{0.4cm}>{\centering}p{0.4cm}>{\centering}p{0.4cm}|>{\centering}p{0.4cm}>{\centering}p{0.4cm}>{\centering}p{1cm}}
\hline 
\multirow{2}{2.7cm}{Method} & \multicolumn{3}{c|}{Supervised} & \multicolumn{3}{c}{Unsupervised}\tabularnewline
 & KNN & SVM & DNN & MR & \ac{wcl} & Proposed\tabularnewline
\hline 
Test Dataset I (Day 2) & 0.81 & 0.81 & 0.79 & 1.08 & 1.67 & \textbf{0.49}\tabularnewline
Test Dataset II (Day 3) & 0.92 & 0.91 & 0.88 & 1.40 & 1.90 & \textbf{0.68}\tabularnewline
\hline 
\end{tabular}
\end{table}

\subsubsection{Region Localization Performance on 10-Region Dataset}

Although we assume that the mobile device visits each divided region
only once, which imposes limitations on the applicability of our model
in certain scenarios, this assumption ensures that we achieve a good
localization performance. Table \ref{tab:Space-classification-performance}
summarizes the region localization error. Somewhat surprisingly, the
proposed scheme, trained without labels, performs even better then
the supervised methods; in fact, it performs the best among all the
schemes tested. The major performance gain is contributed by the signal
subspace model (\ref{eq:x1}), which allows more fluctuation of the
data, and tolerates the slight change of the environment. Moreover,
SVM does not fit the features of RSS data but instead searches for
the boundaries of RSS clusters. DNN attempts to fit the features of
RSS data using a general non-linear function, allowing it to capture
complex patterns and intricate relationships in the data. Consequently,
DNN and SVM are highly sensitive to noise, leading to potential overfitting
issues. SVM achieves optimal performance on test dataset I when the
slack variable is configured to $0.0007$, resulting in a localization
error of 0.60 meters. DNN achieves its best performance on test dataset
I when the regularization hyperparameter is configured to 1.67, resulting
in a localization error of 0.54 meters. Despite alleviating overfitting
by tuning overfitting parameters, finding the optimal parameter setting
remains challenging. Therefore, although they are supervised methods
with labels, the performance of SVM and DNN is slightly inferior to
the proposed one.

\begin{figure}[t]
\begin{centering}
\includegraphics[width=1\columnwidth]{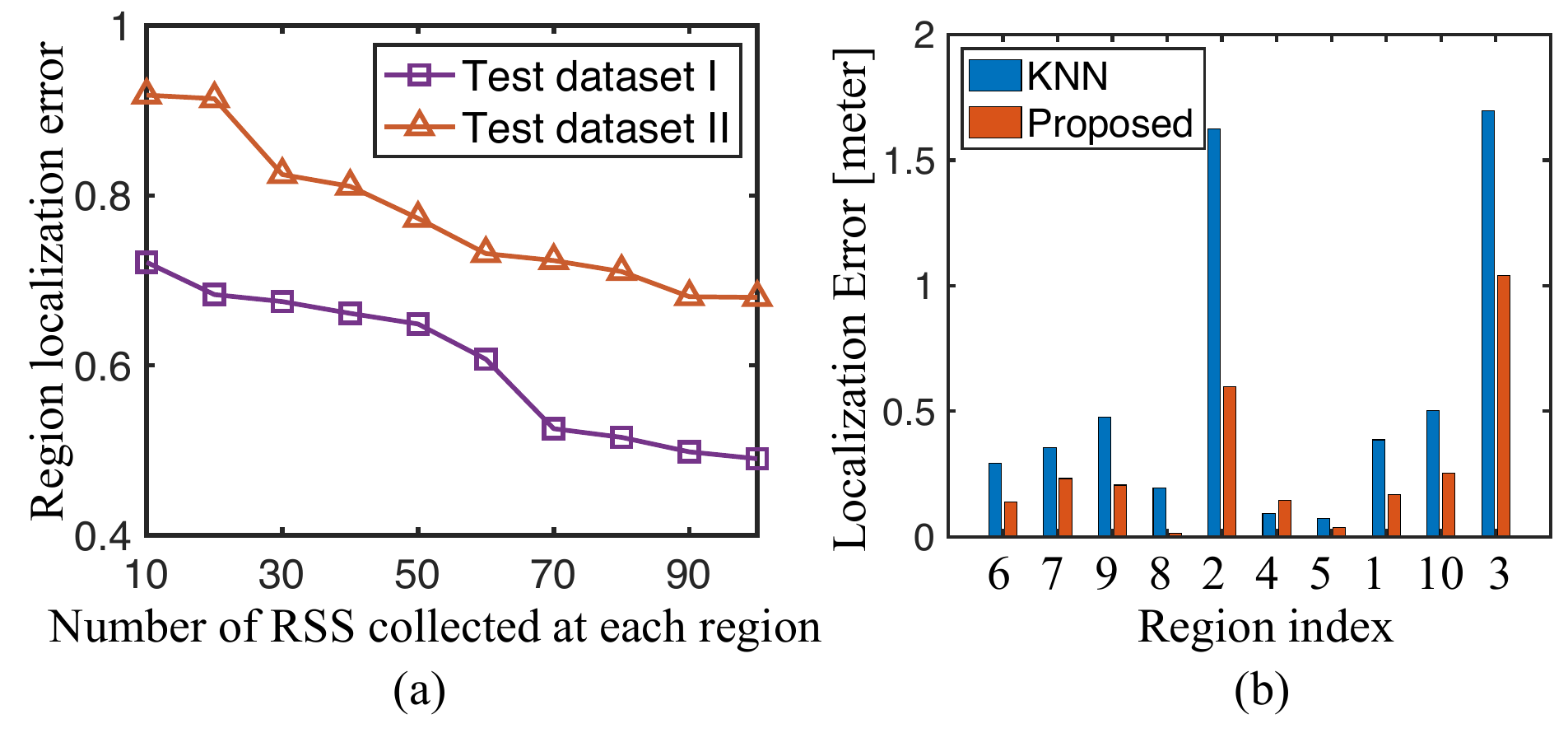}
\par\end{centering}
\caption{(a) Region localization error of our method versus the time interval
for collecting the training data. (b) Region localization error at
each region, where the regions are arranged in an increasing order
of their sizes. \label{fig:effect-region-size}}
\end{figure}

The decrease in positioning accuracy from Day 2 to Day 3 by more than
10\% is because we slightly change the environment to verify the robustness
of the algorithm. Specifically, the environment undergoes changes
as we gather data during regular working hours, simulating real-world
positioning scenarios, where we vary the numbers and locations of
personnel, and we vary the obstacles such as chairs and decorations
on tables. Moreover, the pose of the mobile device also changes during
measurements for test dataset I and II. For the training dataset and
the test dataset I, the mobile device is held in the hand, and it
swings with the arm. For the test dataset II, it is placed in a backpack.
Therefore, the fluctuation in the height of the mobile device and
variations in signal transmission strength result in a series of impacts.
Nevertheless, the proposed method is the least affected.

\subsubsection{Effect of Region Size and Measurement Quantity in Each Region}

Note that the limitation of the proposed work is to provide only a
region-based localization, and therefore, the localization accuracy
depends on the size of the region and the amount of data collected
from each region. We numerically evaluate the localization performance
versus the amount of training data as shown in Fig. \ref{fig:effect-region-size}
(a). We vary the number of RSS data points collected in each region
within the training set. As the quantity of RSS data increases, there
is a reduction in localization error as expected.

To evaluate the impact from the size of the area, we sorted 10 regions
based on their sizes in increasing order and calculated the localization
error for each region in test dataset I. As shown in Fig. \ref{fig:effect-region-size}
(b), there is no clear trend of an increase in localization error
with the increase in the area of the region. However, regions 2 and
3 exhibit relatively large localization errors for all methods, attributed
to the fact that these areas lack distinct separation by large furniture
or walls. Moreover, in our measurement, the mobile device traveled
around a region, and also stayed for a while at a spot. The measurement
was recorded for every 200 ms. Hence, at least, a subset of data is
highly correlated. From our results, we did not notice any significant
performance impact from the data correlation so far.

\begin{table}
\begin{centering}
\caption{Region localization error {[}meters{]} on 24-region Test Dataset III.\label{tab:Space-classification-performance-1}}
\par\end{centering}
\centering{}%
\begin{tabular}{>{\raggedright}m{2cm}|>{\centering}p{0.4cm}>{\centering}p{0.4cm}>{\centering}p{0.4cm}|>{\centering}p{0.4cm}>{\centering}p{0.4cm}>{\centering}p{1cm}}
\hline 
\multirow{2}{2cm}{Dataset III} & \multicolumn{3}{c|}{Supervised} & \multicolumn{3}{c}{Unsupervised}\tabularnewline
 & KNN & SVM & DNN & MR & \ac{wcl} & \multicolumn{1}{c}{Proposed}\tabularnewline
\hline 
Corridor & 0.69 & 0.60 & 0.59 & 0.99 & 1.15 & 0.36\tabularnewline
Room & 0.62 & 0.58 & 0.55 & 0.95 & 1.08 & 0.32\tabularnewline
All & 0.65 & 0.59 & 0.58 & 0.97 & 1.12 & \textbf{0.34}\tabularnewline
\hline 
\end{tabular}
\end{table}

\subsubsection{Region Localization Performance on 24-Region Dataset}

While the algorithm was tested in a relatively small area (but real
office environment), the proposed method is scalable to a larger indoor
area, such as shopping malls. From a computational complexity perspective,
a large indoor area typically has more APs; but since the complexity
of the proposed algorithm is $\mathcal{O}(DN^{2})$, i.e., linearly
in the number of APs $D$, the proposed method can be easily scaled
to a large indoor area. As shown in Table \ref{tab:Space-classification-performance-1},
our method in the extended scenario acquires a 0.34-meter region localization
error, and outperforms all baseline methods.

In Table \ref{tab:Space-classification-performance-1}, the localization
errors of our algorithm in corridor regions (including 8 out of 24
regions) were 0.36 meters, while the localization error in room regions
was measured at 0.32 meters. The localization error in corridor regions
is slightly higher than in room regions. The primary reason is the
absence of clear obstacles between corridors, such as doors or walls,
making it challenging for our signal subspace model to classify the
RSS data measured between two corridors.

\section{Conclusion}

\label{sec:Conclusion}

In this paper, a subspace clustering method with a sequential prior
is proposed to construct a region-based radio map from sequentially
collected \ac{rss} measurements. A maximum-likelihood estimation
problem with a sequential prior is formulated, and solved by a proposed
merge-and-split algorithm that is proven to converge to a globally
optimal solution for a special case. Furthermore, a graph model for
a set of possible routes is constructed which leads to a Viterbi algorithm
for the region matching and achieves less than 1\% matching error.
The numerical results demonstrated that the proposed unsupervised
scheme achieves an even better localization performance than several
supervised learning schemes, including \ac{knn}, \ac{svm}, and \ac{dnn}
which use location labels during the training.


\appendices{}


\section{Proof of Proposition \ref{prop:Asymptotic-Hardening} \label{appendix:Prop1}}

We have 
\begin{align*}
\tilde{f}(\bm{\tau}) & =\frac{1}{N}\sum_{i=1}^{N}\sum_{\substack{j=1,j\neq k,k+1}
}^{K}z_{\beta}(i,\tau_{j-1},\tau_{j})\big\|\mathbf{x}_{i}\\
 & \qquad-\hat{\mathbf{{\rm \bm{\mu}}}}_{k}(\tau_{k-1},\tau_{k})\big\|_{2}^{2}+\tilde{f}_{k}(\tau_{k})
\end{align*}
where
\begin{align*}
 & \tilde{f}_{k}(\tau_{k})=\frac{1}{N}\sum_{i=1}^{N}\Big[z_{\beta}(i,\tau_{k-1},\tau_{k})\big\|\mathbf{x}_{i}-\hat{\mathbf{{\rm \bm{\mu}}}}_{k}(\tau_{k-1},\tau_{k})\big\|_{2}^{2}\\
 & \qquad\qquad+z_{\beta}(i,\tau_{k},\tau_{k+1})\big\|\mathbf{x}_{i}-\hat{\mathbf{{\rm \bm{\mu}}}}_{k+1}(\tau_{k},\tau_{k+1})\big\|_{2}^{2}\Big].
\end{align*}

Recall the uniform convergence for $\hat{\mathbf{{\rm \bm{\mu}}}}_{k}(\tau_{k-1},\tau_{k})$
as $\beta\to0$ in (\ref{eq:tilde-mu-k}), i.e., $\hat{\mathbf{{\rm \bm{\mu}}}}_{k}(\tau_{k-1},\tau_{k})\to\tilde{\bm{\mu}}(\tau_{k-1},\tau_{k})$.
We have
\begin{align*}
\tilde{f}_{k}(\tau_{k}) & \to\bar{f}_{k}(\tau_{k})\\
 & \triangleq\frac{1}{N}\sum_{i=1}^{N}(z_{\beta}(i,\tau_{k-1},\tau_{k})\big\|\mathbf{x}_{i}-\tilde{\bm{\mu}}(\tau_{k-1},\tau_{k})\big\|_{2}^{2}\\
 & \qquad+z_{\beta}(i,\tau_{k},\tau_{k+1})\big\|\mathbf{x}_{i}-\tilde{\bm{\mu}}(\tau_{k},\tau_{k+1})\big\|_{2}^{2})
\end{align*}
where $\bar{f}_{k}(\tau_{k})$ is the asymptotic form of $\tilde{f}_{k}(\tau_{k})$
as $\beta\to0$.

Since $z_{\beta}(i,t_{k-1},t_{k})\to\mathbb{I}\{t_{k-1}<i\leq t_{k}\}$
as $\beta\rightarrow0$, where the indicator function $\mathbb{I}\{t_{k-1}<i\leq t_{k}\}=1$
if $t_{k-1}<i\leq t_{k}$, and 0 otherwise, we have
\begin{align*}
 & \sum_{i=1}^{N}z_{\beta}(i,\tau_{k-1},\tau_{k})\big\|\mathbf{x}_{i}-\tilde{\bm{\mu}}(\tau_{k-1},\tau_{k})\big\|_{2}^{2}\\
 & \rightarrow\sum_{i=\tau_{k-1}+1}^{\tau_{k+1}}(1-\sigma_{\beta}(i-\tau_{k}))\big\|\mathbf{x}_{i}-\tilde{\bm{\mu}}(\tau_{k-1},\tau_{k})\big\|_{2}^{2}
\end{align*}
 and
\begin{align*}
 & \sum_{i=1}^{N}z_{\beta}(i,\tau_{k},\tau_{k+1})\big\|\mathbf{x}_{i}-\tilde{\bm{\mu}}(\tau_{k},\tau_{k+1})\big\|_{2}^{2}\\
 & \rightarrow\sum_{i=\tau_{k-1}+1}^{\tau_{k+1}}\sigma_{\beta}(i-\tau_{k})\big\|\mathbf{x}_{i}-\tilde{\bm{\mu}}(\tau_{k},\tau_{k+1})\big\|_{2}^{2}
\end{align*}
as $\beta\rightarrow0$.

Thus, as $\beta\rightarrow0$, we have $\tilde{f}_{k}(\tau_{k})$$\to f_{k}(\tau_{k})$.
So, we have $\tilde{f}(\bm{\tau})\to\frac{1}{2}\sum_{k=0}^{K}f_{k}(\tau_{k})$
as $\beta\rightarrow0$.

\section{Proof of Proposition \ref{prop:Asymptotic-Consistency} \label{appendix:Prop2}}

Without loss of generality, we study the case $K=2$ and $k=1$. From
(\ref{eq:f_k}) and (\ref{eq:Fk-def}), $\bar{F}_{k}(\gamma_{k})$
and $\bar{f}_{k}(\gamma_{k})$ can be written as $\bar{F}_{1}(\gamma_{1})=\mathbb{E}\{\bar{f}_{1}(\gamma_{1})\}$
and
\begin{align}
\bar{f}_{1}(\gamma_{1}) & =\underset{\beta\rightarrow0}{\mathrm{lim}}\enskip\frac{1}{N}\sum_{i=1}^{N}\Big[\Big(1-\sigma_{\beta}(i-\gamma_{1}N)\Big)\big\|\mathbf{x}_{i}-\tilde{\bm{\mu}}(0,\gamma_{1}N)\big\|_{2}^{2}\nonumber \\
 & \qquad+\sigma_{\beta}(i-\gamma_{1}N)\big\|\mathbf{x}_{i}-\tilde{\bm{\mu}}(\gamma_{1}N,N)\big\|_{2}^{2}\Big].\label{eq:f1_tau1}
\end{align}

Firstly, we prove that $\bar{f}_{1}(\gamma_{1})\overset{\mathrm{P}}{\rightarrow}\bar{F}_{1}(\gamma_{1})$
uniformly for all $\gamma_{1}$, i.e., $\mathrm{sup}_{\gamma_{1}\in\Gamma}\enskip|\bar{f}_{1}(\gamma_{1})-\bar{F}_{1}(\gamma_{1})|\overset{\mathrm{P}}{\rightarrow}0$
as $N\to\infty$, where $\overset{\mathrm{P}}{\rightarrow}$ denotes
convergence in probability.

For any $\gamma_{1}\leq\bar{\gamma}_{1}=t_{1}/N$, from (\ref{eq:tilde-mu-k}), and (\ref{eq:x1}), the cost function
$f_{1}(\tau_{1})$ can be written as
\begin{align*}
	& f_{k}(\tau_{1})\\
	& =\frac{1}{N}\sum_{i=1}^{t_{1}}\sigma_{\beta}\left(i-\tau_{1}\right)\|-\frac{1}{N-\tau_{1}}\sum_{j=\tau_{1}+1}^{N}\bm{\epsilon}_{j}+\frac{N-t_{1}}{N-\tau_{1}}\\
	& \qquad\times(\bm{\mu}_{1}-\bm{\mu}_{2})+\bm{\epsilon}_{i}\|_{2}^{2}+\frac{1}{N}\sum_{i=t_{1}+1}^{N}\sigma_{\beta}\left(i-\tau_{1}\right)\\
	& \qquad\times\|-\frac{1}{N-\tau_{1}}\sum_{j=\tau_{1}+1}^{N}\bm{\epsilon}_{j}+\frac{t_{1}-\tau_{1}}{N-\tau_{1}}(\bm{\mu}_{2}-\bm{\mu}_{1})+\bm{\epsilon}_{i}\|_{2}^{2}\\
	& \qquad+\frac{1}{N}\sum_{i=t_{1}+1}^{N}\left(1-\sigma_{\beta}\left(i-\tau_{1}\right)\right)\|\bm{\mu}_{2}-\bm{\mu}_{1}-\frac{1}{\tau_{1}}\sum_{j=1}^{\tau_{1}}\bm{\epsilon}_{j}\\
	& \qquad+\bm{\epsilon}_{i}\|_{2}^{2}+\frac{1}{N}\sum_{i=1}^{t_{1}}\left(1-\sigma_{\beta}\left(i-\tau_{1}\right)\right)\|-\frac{1}{\tau_{1}}\sum_{j=1}^{\tau_{1}}\bm{\epsilon}_{j}+\bm{\epsilon}_{i}\|_{2}^{2}.
\end{align*}
and the cost function $F_{1}(\tau_{1})$ can be written as

\begin{align*}
	& F_{1}(\tau)\\
	& =\frac{1}{N}\sum_{i=t_{1}+1}^{N}\bigg\{\sigma_{\beta}\left(i-\tau_{1}\right)\left(\frac{t_{1}-\tau_{1}}{N-\tau_{1}}\right)^{2}\|\mathbf{{\rm \bm{\mu}}}_{1}-\mathbf{{\rm \bm{\mu}}}_{2}\|_{2}^{2}\\
	& \qquad+\Big(1-\sigma_{\beta}\left(i-\tau_{1}\right)\Big)||\mathbf{{\rm \bm{\mu}}}_{1}-\mathbf{{\rm \bm{\mu}}}_{2}||_{2}^{2}\bigg\}\\
	& \qquad+\frac{1}{N}\left(\frac{N-t_{1}}{N-\tau_{1}}\right)^{2}\sum_{i=1}^{t_{1}}\sigma_{\beta}\left(i-\tau_{1}\right)\|\mathbf{{\rm \bm{\mu}}}_{1}-\mathbf{{\rm \bm{\mu}}}_{2}\|_{2}^{2}\\
	& \qquad+s_{0}^{2}\bigg\{1+\frac{1}{\tau_{1}}-\frac{2}{N}+\frac{1}{N(N-\tau_{1})}\sum_{i=1}^{\tau_{1}}\sigma_{\beta}\left(i-\tau_{1}\right)\\
	& \qquad-\frac{1}{N\tau_{1}}\sum_{i=\tau_{1}+1}^{N}\sigma_{\beta}\left(i-\tau_{1}\right)+\frac{1}{N\tau_{1}}\sum_{i=1}^{\tau_{1}}\sigma_{\beta}\left(i-\tau_{1}\right)\bigg\}.
\end{align*}
Then,
\begin{align*}
	& \underset{\beta\rightarrow0}{\mathrm{lim}}f_{1}(\tau_{1})-\underset{\beta\rightarrow0}{\mathrm{lim}}F_{1}(\tau_{1})\\
	& =2\frac{N-t_{1}}{N}(\bm{\mu}_{1}-\bm{\mu}_{2})^{\mathrm{T}}\Big((\frac{1}{N-\tau_{1}}-\frac{1}{N-t_{1}})\sum_{i=t_{1}+1}^{N}\bm{\epsilon}_{i}\\
	& \qquad+\frac{1}{N-\tau_{1}}\sum_{i=\tau_{1}+1}^{t_{1}}\bm{\epsilon}_{i}\Big)+\frac{1}{N}\sum_{i=1}^{N}\bm{\epsilon}_{i}^{\mathrm{T}}\bm{\epsilon}_{i}-s_{0}^{2}.
\end{align*}
Thus, with
$s_{0}^{2}\overset{\triangle}{=}Ds^{2}$,
\begin{align}
 & \bar{f}_{1}(\gamma_{1})-\bar{F}_{1}(\gamma_{1})\label{eq:diff_f_F}\\
 & =\underset{\beta\rightarrow0}{\mathrm{lim}}\enskip f_{1}(\gamma_{1}N)-\underset{\beta\rightarrow0}{\mathrm{lim}}\enskip F_{1}(\gamma_{1}N)\nonumber \\
 & =2(\bm{\mu}_{1}-\bm{\mu}_{2})^{\mathrm{T}}\frac{1}{N}\Big[\frac{\gamma_{1}-\bar{\gamma}_{1}}{1-\gamma_{1}}\sum_{i=\bar{\gamma}_{1}N+1}^{N}\bm{\epsilon}_{i}\nonumber \\
 & \qquad+\frac{1-\bar{\gamma}_{1}}{1-\gamma_{1}}\sum_{i=\gamma_{1}N+1}^{\bar{\gamma}_{1}N}\bm{\epsilon}_{i}\Big]+\frac{1}{N}\sum_{i=1}^{N}\bm{\epsilon}_{i}^{\mathrm{T}}\bm{\epsilon}_{i}-s_{0}^{2}.\nonumber 
\end{align}

\begin{lem}
\label{lem:N-inf}It holds that $\mathrm{sup}_{\gamma_{1}\in\bar{\Gamma}_{1}}\enskip|\bar{f}_{1}(\gamma_{1})-\bar{F}_{1}(\gamma_{1})|\overset{\mathrm{P}}{\rightarrow}0$
with $N\rightarrow\infty$, where $\bar{\Gamma}_{1}\overset{\triangle}{=}\{1/N,2/N,...,\bar{\gamma}_{1}\}$.
\end{lem}
\begin{proof}
The absolute value of $\bar{f}_{1}(\gamma_{1})-\bar{F}_{1}(\gamma_{1})$
in (\ref{eq:diff_f_F}) can be upper bounded by
\begin{align}
 & |\bar{f}_{1}(\gamma_{1})-\bar{F}_{1}(\gamma_{1})|\label{eq:dfF}\\
 & \leq\Big|\frac{2(\gamma_{1}-\bar{\gamma}_{1})}{1-\gamma_{1}}\Big|\cdot|\bm{\mu}_{1}-\bm{\mu}_{2}|^{\mathrm{T}}\Big|\frac{1}{N}\sum_{i=\bar{\gamma}_{1}N+1}^{N}\bm{\epsilon}_{i}\Big|+\Big|\frac{2(1-\bar{\gamma}_{1})}{1-\gamma_{1}}\Big|\nonumber \\
 & \qquad\times|\bm{\mu}_{1}-\bm{\mu}_{2}|^{\mathrm{T}}\Big|\frac{1}{N}\sum_{i=\gamma_{1}N+1}^{\bar{\gamma}_{1}N}\bm{\epsilon}_{i}\Big|+\Big|\frac{1}{N}\sum_{i=1}^{N}\bm{\epsilon}_{i}^{\mathrm{T}}\bm{\epsilon}_{i}-s_{0}^{2}\Big|.\nonumber 
\end{align}
where $|\mathbf{x}|$ means $(|x_{1}|,|x_{2}|,...,|x_{n}|)^{\mathrm{T}}$,
i.e., taking absolute values for each element of the vector $\mathbf{x}$.

For the first term on the \ac{rhs} of (\ref{eq:dfF}), we have 
\begin{align*}
 & \underset{\gamma_{1}\in\bar{\Gamma}_{1}}{\mathrm{sup}}\Big|\frac{2(\bar{\gamma}_{1}-\gamma_{1})}{1-\gamma_{1}}\Big|\cdot|\bm{\mu}_{1}-\bm{\mu}_{2}|^{\mathrm{T}}\Big|\frac{1}{N}\sum_{i=\bar{\gamma}_{1}N+1}^{N}\bm{\epsilon}_{i}\Big|\\
 & \leq\frac{2(\bar{\gamma}_{1}N-1)}{N-1}\Big(\underset{j}{\mathrm{max}}|\mu_{1,j}-\mu_{2,j}|\Big)\Big\|\frac{1}{N}\sum_{i=\bar{\gamma}_{1}N+1}^{N}\bm{\epsilon}_{i}\Big\|_{1}\rightarrow0
\end{align*}
as $N\rightarrow\infty$. This is because $\frac{2(\bar{\gamma}_{1}N-1)}{N-1}(\mathrm{max}_{j}|\mu_{1,j}-\mu_{2,j}|)$
is bounded and 
\[
\frac{1}{N}\sum_{i=\bar{\gamma}_{1}N+1}^{N}\bm{\epsilon}_{i}=\frac{(1-\bar{\gamma}_{1})N}{N}\frac{1}{(1-\bar{\gamma}_{1})N}\sum_{i=\bar{\gamma}_{1}N+1}^{N}\bm{\epsilon}_{i}\rightarrow\mathbf{0}
\]
due to the strong law of large numbers, where we recall that $\bm{\epsilon}_{i}$
are i.i.d. with zero mean and bounded variance.

For the second term on the \ac{rhs} of (\ref{eq:dfF}), we have 
\begin{align*}
 & \underset{\gamma_{1}\in\bar{\Gamma}_{1}}{\mathrm{sup}}\Big|\frac{2(1-\bar{\gamma}_{1})}{1-\gamma_{1}}\Big|\cdot|\bm{\mu}_{1}-\bm{\mu}_{2}|^{\mathrm{T}}\Big|\frac{1}{N}\sum_{i=\gamma_{1}N+1}^{\bar{\gamma}_{1}N}\bm{\epsilon}_{i}\Big|\\
 & \leq2\Big(\underset{j}{\mathrm{max}}|\mu_{1,j}-\mu_{2,j}|\Big)\underset{\gamma_{1}\in\bar{\Gamma}_{1}}{\mathrm{sup}}\Big\|\frac{1}{N}\sum_{i=\gamma_{1}N+1}^{\bar{\gamma}_{1}N}\bm{\epsilon}_{i}\Big\|_{1}\rightarrow0.
\end{align*}
To see this, we compute
\begin{align*}
\underset{\gamma_{1}\in\bar{\Gamma}_{1}}{\mathrm{sup}}\Big\|\frac{1}{N}\sum_{i=\gamma_{1}N+1}^{\bar{\gamma}_{1}N}\bm{\epsilon}_{i}\Big\|_{1} & =\underset{\gamma_{1}\in\bar{\Gamma}_{1}}{\mathrm{sup}}\sum_{j=1}^{D}\Big|\frac{1}{N}\sum_{i=\gamma_{1}N+1}^{\bar{\gamma}_{1}N}\epsilon_{i,j}\Big|\\
 & \leq\sum_{j=1}^{D}\underset{\gamma_{1}\in\bar{\Gamma}_{1}}{\mathrm{sup}}\Big|\frac{1}{N}\sum_{i=\gamma_{1}N+1}^{\bar{\gamma}_{1}N}\epsilon_{i,j}\Big|.
\end{align*}
Note that, $\forall\:\lambda>0$, 
\begin{align*}
 & \mathbb{P}\Bigg\{\mathrm{sup}_{\gamma_{1}\in\bar{\Gamma}_{1}}\Big|\frac{1}{N}\sum_{i=\gamma_{1}N+1}^{\bar{\gamma}_{1}N}\epsilon_{i,j}\Big|>\lambda\Bigg\}\\
 & =\mathbb{P}\Bigg\{\underset{2\leq k\leq\bar{\gamma}_{1}N}{\mathrm{max}}\Big|\sum_{i=k}^{\bar{\gamma}_{1}N}\frac{\epsilon_{i,j}}{N}\Big|>\lambda\Bigg\}\\
 & \leq\frac{1}{\lambda^{2}}\sum_{i=2}^{\bar{\gamma}_{1}N}\mathbb{V}\Big\{\frac{\epsilon_{i,j}}{N}\Big\}=\frac{1}{\lambda^{2}}\frac{1}{N^{2}}(\bar{\gamma}_{1}N-1)s^{2}\rightarrow0
\end{align*}
as $N\rightarrow\infty$, where the inequality above is due to the
Kolmogorov's inequality. This confirms that the second term of the
\ac{rhs} of (\ref{eq:dfF}) indeed converges to 0 in probability
as $N\rightarrow\infty$.

For the last term on the \ac{rhs} of (\ref{eq:dfF}) that is independent
of $\gamma_{1}$, it is easy to verify $\frac{1}{N}\sum_{i=1}^{N}\bm{\epsilon}_{i}^{\mathrm{T}}\bm{\epsilon}_{i}-s_{0}^{2}\rightarrow0$
as $N\rightarrow\infty$.

Since $\mathrm{sup}_{\gamma_{1}\in\bar{\Gamma}_{1}}\enskip|\bar{f}_{1}(\gamma_{1})-\bar{F}_{1}(\gamma_{1})|$
is less than or equal to the supreme of the \ac{rhs} of (\ref{eq:dfF}),
and the \ac{rhs} of (\ref{eq:dfF}) converges to 0 in probability
as $N\rightarrow\infty$, we have the conclusion that $\mathrm{sup}_{\gamma_{1}\in\bar{\Gamma}_{1}}|\enskip\bar{f}_{1}(\gamma_{1})-\bar{F}_{1}(\gamma_{1})|\overset{\mathrm{P}}{\rightarrow}0$
as $N\rightarrow\infty$.
\end{proof}
For any $\bar{\gamma}_{1}\leq\gamma_{1}\leq1$, we can derive a formula
similar to (\ref{eq:diff_f_F}) and a result similar to Lemma \ref{lem:N-inf},
based on which, the convergence of $\mathrm{sup}_{\gamma_{1}\in\{\bar{\gamma}_{1},\bar{\gamma}_{1}+1/N,...,1\}}|\bar{f}_{1}(\gamma_{1})-\bar{F}_{1}(\gamma_{1})|\overset{\mathrm{P}}{\rightarrow}0$
as $N\rightarrow\infty$ can be established in a similar way.

Recall that $\gamma_{1}^{*}$ is the minimizer of $\bar{F}_{1}(\gamma_{1})$,
and $\hat{\gamma}_{1}$ is the minimizer of $\bar{f}_{1}(\gamma_{1})$.
We have $\bar{f}_{1}(\hat{\gamma}_{1})\leq\bar{f}_{1}(\gamma_{1}^{*})$
and $\bar{F}_{1}(\gamma_{1}^{*})\leq\bar{F}_{1}(\hat{\gamma}_{1})$.
Recall that $\bar{f}_{1}(\gamma_{1})\overset{\mathrm{P}}{\rightarrow}\bar{F}_{1}(\gamma_{1})$
uniformly for all $\gamma_{1}$. We have
\begin{align*}
\bar{f}_{1}(\hat{\gamma}_{1}) & =\bar{F}_{1}(\hat{\gamma}_{1})+o_{p}(1)\\
 & \geq\bar{F}_{1}(\gamma_{1}^{*})+o_{p}(1)=\bar{f}_{1}(\gamma_{1}^{*})+o_{p}(1)
\end{align*}
where $o_{p}(1)$ is a term that converges to $0$ in probability
as $N\rightarrow\infty$. Therefore, we must have $\bar{f}_{1}(\hat{\gamma}_{1})=\bar{f}_{1}(\gamma_{1}^{*})+o_{p}(1)$.

In addition, as is studied in Proposition \ref{prop:Unimodality},
$F_{1}(\tau)$ has the property that it has a unique local minimizer
with a small enough $\beta$, and thus, so is $\bar{F}_{1}(\gamma)=F_{1}(\gamma N)$
for all finite $N$. Finally, we invoke the following result from
\cite{VanDer:b00}.
\begin{lem}
[Theorem 5.7, \cite{VanDer:b00}]\label{lem:Asymptotic}With $\beta\rightarrow0$,
if $\mathrm{sup}_{\gamma_{1}\in\Gamma}|\bar{f}_{1}(\gamma_{1})-\bar{F}_{1}(\gamma_{1})|\overset{\mathrm{P}}{\rightarrow}0$,
and $\mathrm{inf}_{\gamma_{1}:d(\gamma_{1},\gamma_{1}^{*})\geq\varepsilon_{1}}\bar{F}_{1}(\gamma_{1})>\bar{F}_{1}(\gamma_{1}^{*})$
for every $\varepsilon_{1}>0$, where $d(\gamma_{1},\gamma_{1}^{*})=|\gamma_{1}-\gamma_{1}^{*}|$.
Then any sequence of estimators $\hat{\gamma}_{1}$ with $\bar{f}_{1}(\hat{\gamma}_{1})\leq\bar{f}_{1}(\gamma_{1}^{*})+o_{p}(1)$
converges in probability to $\gamma_{1}^{*}$.
\end{lem}
Therefore, we have $\hat{\gamma}_{1}\to\gamma_{1}^{*}$ in probability
as $N\to\infty$.

\section{Proof of Proposition \ref{prop:Unimodality}, \ref{prop:flatness},
and \ref{prop:monotonicity-near-boundary}}

\subsection*{1. Proof of Proposition \ref{prop:Unimodality}\label{appendix:Prop3}}

Following the signal model (\ref{eq:x1}), under $d_{j}=d_{j+1}=0$,
$s_{j}=s_{j+1}=s$, we have
\begin{equation}
\mathbf{x}_{i}=\begin{cases}
\mathbf{{\rm \bm{\mu}}}_{j}+\bm{\epsilon}_{i}, & \tau_{k-1}<i\leq t_{j}\\
\mathbf{{\rm \bm{\mu}}}_{j+1}+\bm{\epsilon}_{i}, & t_{j}<i\leq\tau_{k+1}.
\end{cases}\label{eq:xii-1}
\end{equation}

Case 1: Consider $\tau_{k-1}<i\leq t_{j}$. Using (\ref{eq:xii-1}),
one can compute the expectation of $f_{k}(\tau)$ and then take the
difference, i.e., $\triangle F_{k}(\tau)=F_{k}(\tau)-F_{k}(\tau-1)$
can be computed as
\begin{align}
\triangle F_{k}(\tau) & =\frac{1}{N}\|\mathbf{{\rm \bm{\mu}}}_{j}-\mathbf{{\rm \bm{\mu}}}_{j+1}\|_{2}^{2}u(\tau,t_{j},\beta)+\frac{1}{N}s_{0}^{2}\gamma(\tau,\beta)\label{eq:dF-1}
\end{align}
where $u(\tau,t_{j},\beta)$ and $\gamma(\tau,\beta)$ are functions
\ac{wrt} the segmentation variable $\tau$, the segment boundary
$t_{j}$, and the parameter $\beta$. For a detailed derivation of equation (\ref{eq:dF-1}), please refer to Appendix \ref{app:proof-prop-uni}.

To simplify the expression (\ref{eq:dF-1}), consider the definition
of $\sigma_{\beta}\left(x\right)=\sigma\left((x-1/2)/\beta\right)$
based on the sigmoid function $\sigma(x)=(1+\mathrm{exp}(-x))^{-1}$.
It follows that
\begin{equation}
\begin{cases}
1-\varepsilon<\sigma_{\beta}\left(i-\tau\right)<1, & i\geq\tau+1\\
0<\sigma_{\beta}\left(i-\tau\right)<\varepsilon, & i\leq\tau
\end{cases}\label{eq:bound_signoid_beta}
\end{equation}
for $\beta<(2\mathrm{ln}(1/\varepsilon-1))^{-1}$. Using the bounds
in (\ref{eq:bound_signoid_beta}), there exist positive constants
$B_{1},B_{2}<\infty$, such that $u(\tau,t_{j},\beta)\leq\varepsilon B_{1}-B_{2}$.

Likewise, using (\ref{eq:bound_signoid_beta}), the term $\gamma(\tau,\beta)$
can be upper bounded as $\gamma(\tau,\beta)\leq\varepsilon(4+\frac{(\tau_{k+1}-\tau_{k-1})^{2}}{\tau_{k+1}-\tau_{k-1}-1})$.
As a result, there exist constants $C_{1},C_{2}<\infty$, such that
the difference in (\ref{eq:dF-1}) $\bigtriangleup F_{k}(\tau)<\varepsilon C_{1}-C_{2}<0$
for a small enough $\beta$.

Case 2: Now, consider $t_{j}<i\leq\tau_{k+1}$. The derivation for
a lower bound of $\triangle F_{k}(\tau)$ in (\ref{eq:dF-1}) is similar
to the derivation of the upper bound of $\triangle F_{k}(\tau)$ in
Case 1 for $\tau_{k-1}<i\leq t_{j}$. The lower bound of $u(\tau,t_{j},\beta)$
is given by $B_{4}-\varepsilon B_{3}$ for some positive and finite
$B_{3},B_{4}$. In addition, the lower bound of $\gamma(\tau,\beta)$
is $\gamma(\tau,\beta)\geq-\frac{1}{2}(\tau_{k+1}-\tau_{k-1})\varepsilon$.
As a result, there exist constants $C_{1}',C_{2}'<\infty$, such that
$\bigtriangleup F_{k}(\tau)>C_{2}'-\varepsilon C_{1}'>0$ for a small
enough $\beta$. 

For a comprehensive derivation of the bound of $\bigtriangleup F_{k}(\tau)$ for case 1 and case 2, please consult Appendix \ref{app:proof-prop-uni}. Combining the above two cases, it follows that if 
\begin{align*}
	\beta & <\frac{1}{2}\mathrm{log}^{-1}\Big(\mathrm{max}\{\frac{B_{1}}{B_{2}}-1,\frac{B_{3}}{B_{4}}-1\}\\
	& \qquad+\mathrm{max}\{\frac{D}{B_{2}}\big(4+\frac{(\tau_{k+1}-\tau_{k-1})^{2}}{\tau_{k+1}-\tau_{k-1}-1}\big)\\
	& \qquad,\frac{D(\tau_{k+1}-\tau_{k-1})}{2B_{4}}\}\frac{s^{2}}{\|\mathbf{{\rm \bm{\mu}}}_{j}-\mathbf{{\rm \bm{\mu}}}_{j+1}\|_{2}^{2}}\Big)
\end{align*}
we must have $\triangle F_{k}(\tau)>0$ for $\tau\leq t_{j}$ and
$\triangle F_{k}(\tau)<0$ for $\tau>t_{j}$.

\subsection*{2. Proof of Proposition \ref{prop:flatness} \label{appendix:Prop4}}

Under $d_{k}=0$, $s_{k}=s$, and no partition index in $(\tau_{k-1},\tau_{k+1}]$,
it corresponds to the case where there exists a $t_{j}$ in $\{t_{1},t_{2},\dots,t_{K-1}\}$
such that $t_{j}\geq t_{k+1}$. Thus, the derivation of $\triangle F_{k}(\tau)$
follows (\ref{eq:diff_f_F}) by replacing $t_{j}$ with $\tau_{k+1}$.
Then, following the same derivation as in Appendix \ref{appendix:Prop3}-1,
one can easily arrive at $\triangle F_{k}(\tau)<\varepsilon s^{2}\frac{D(\tau_{k+1}-\tau_{k-1})^{2}}{(\tau-\tau_{k-1})(\tau_{k+1}-\tau)}$,
and $\triangle F_{k}(\tau)>-\varepsilon s^{2}\frac{D(\tau_{k+1}-\tau_{k-1})^{2}}{(\tau_{k+1}-\tau+1)(\tau-\tau_{k-1}-1)}$,
which are bounded since $\tau\in(\tau_{k-1},\tau_{k+1}]$. As a result,
$|\triangle F_{k}(\tau)|<\varepsilon s^{2}C_{0}$, for some finite
constant $C_{0}>0$.

\subsection*{3. Proof of Proposition \ref{prop:monotonicity-near-boundary} \label{appendix:Prop5}}

Following the signal model (\ref{eq:x1}) under $d_{k}=0$, $s_{k}=s$
for any $k\in\{1,2,...,K\}$, where $K\geq2$, we have
\begin{equation}
\mathbf{x}_{i}=\mathbf{{\rm \bm{\mu}}}_{k}+\bm{\epsilon}_{i},t_{k-1}<i\leq t_{k}\label{eq:xii}
\end{equation}

Case 1: Consider $t_{k-1}<\tau\leq t_{j}$. Define $\bm{\eta}(o,l)=\frac{1}{t_{l}-t_{o-1}}\sum_{a=o}^{l}(t_{a}-t_{a-1})\mathbf{{\rm \bm{\mu}}}_{a}$
as the mean of the sample located in the $k$th subspace with $s_{k}=s$,
and $d_{k}=0$, $k\in[o,l]$. Using (\ref{eq:xii}), one can compute
the expectation of $f_{k}(\tau)$ and then take the difference, i.e.,
$\triangle F_{k}(\tau)=F_{k}(\tau)-F_{k}(\tau-1)$. Using the bounds
in (\ref{eq:bound_signoid_beta}), there exist constants $C_{3},C_{4}$,
such that $\triangle F_{k}(\tau)$ is upper bounded by $\bigtriangleup F(\tau)<\varepsilon C_{3}-C_{4}$,
which is smaller than zero if $\beta$ is small enough, and $\|\mathbf{{\rm \bm{\mu}}}_{j}-\bm{\eta}(j+1,j+J+1)\|_{2}^{2}\neq0$,
i.e., $\mathbf{{\rm \bm{\mu}}}_{j},\mathbf{{\rm \bm{\mu}}}_{2},...,\mathbf{{\rm \bm{\mu}}}_{j+J+1}$
are linearly independent.

Case 2: Now, consider $t_{j+J}<\tau\leq\tau_{k+1}$. The derivation
for a lower bound of $\bigtriangleup F_{k}(\tau)$ is similar to the
derivation of the upper bound of $\bigtriangleup F_{k}(\tau)$ in
Case 1 for $t_{j+J}<\tau\leq\tau_{k+1}$. Using the bounds in (\ref{eq:bound_signoid_beta}),
there exist constants $C_{3}',C_{4}'$, such that $\triangle F_{k}(\tau)$
is lower bounded by $\bigtriangleup F(\tau)>-\varepsilon C_{3}'+C_{4}'$,
which is larger than zero if $\beta$ is small enough and $\|\bm{\eta}(j,j+J)-\mathbf{{\rm \bm{\mu}}}_{j+J+1}\|_{2}^{2}\neq0$.

For a detailed derivation of the bound of $\bigtriangleup F_{k}(\tau)$ for case 1 and case 2, please refer to Appendix \ref{app:deri-prop-mono}.
Combining the above two cases. It follows that if 
\begin{align*}
	\beta & <\frac{1}{2}\Big[\mathrm{log}\Big(\mathrm{max}\{\frac{B_{1}}{B_{2}}-1,\frac{B_{3}}{B_{4}}-1\}+\frac{D}{\Phi(\mathbf{{\rm \bm{\mu}}}_{j},\mathbf{{\rm \bm{\mu}}}_{2},...,\mathbf{{\rm \bm{\mu}}}_{j+J+1})}\\
	& \qquad\times\mathrm{max}\{\frac{(\tau_{k+1}-\tau_{k-1})^{2}}{(\tau-\tau_{k-1})(\tau_{k+1}-\tau)B_{2}}s^{2}\\
	& \qquad,\frac{(\tau_{k+1}-\tau_{k-1})^{2}}{(\tau_{k+1}-\tau+1)(\tau-\tau_{k-1}-1)B_{4}}\}\Big)\Big]^{-1}
\end{align*}
$\phi(\mathbf{{\rm \bm{\mu}}}_{j},\bm{\eta}(j+1,j+J+1))\neq0$, and
$\phi(\bm{\eta}(j,j+J),\mathbf{{\rm \bm{\mu}}}_{j+J+1})\neq0$, where
\begin{align*}
	& \Phi(\mathbf{{\rm \bm{\mu}}}_{j},\mathbf{{\rm \bm{\mu}}}_{j+1},...,\mathbf{{\rm \bm{\mu}}}_{j+J+1})\\
	& =\mathrm{min}\{\|\mathbf{{\rm \bm{\mu}}}_{j}-\frac{1}{\tau_{k+1}-t_{1}}\sum_{a=j+1}^{j+J+1}(t_{a}-t_{a-1})\mathbf{{\rm \bm{\mu}}}_{a}\|_{2}^{2}\\
	& \qquad,\|\frac{1}{t_{j+J}-\tau_{k-1}}\sum_{a=j}^{j+J}(t_{a}-t_{a-1})\mathbf{{\rm \bm{\mu}}}_{a}-\mathbf{{\rm \bm{\mu}}}_{j+J+1}\|_{2}^{2}\}
\end{align*}
We have $\bigtriangleup F_{k}(\tau)>0$ for $\tau_{k-1}+1<\tau\leq t_{j}$
and $\bigtriangleup F_{k}(\tau)<0$ for $t_{j+J}<\tau\leq\tau_{k+1}$.
The inequality $\phi(\mathbf{{\rm \bm{\mu}}}_{j},\bm{\eta}(j+1,j+J+1))\neq0$,
and $\phi(\bm{\eta}(j,j+J),\mathbf{{\rm \bm{\mu}}}_{j+J+1})\neq0$
hold if $\mathbf{{\rm \bm{\mu}}}_{j},\mathbf{{\rm \bm{\mu}}}_{2},...,\mathbf{{\rm \bm{\mu}}}_{j+J+1}$
are linear independent.


\section{Proof of Lemma \ref{lem:binary partition} \label{appendix:Prop6}}

Since there exists at least one index $t_{k'}\in\{t_{1},t_{2},...,t_{K-1}\}$
in the cluster $\tilde{\mathcal{C}}_{j}^{(m,k)}$, it follows that
there are at least one such index $t_{k'}$ within the interval $(\tau_{j-1}^{(m,k,j)},\tau_{j+1}^{(m,k,j)}]$.
It thus follows from Proposition \ref{prop:Asymptotic-Consistency}
and \ref{prop:monotonicity-near-boundary} that for any $\varepsilon>0$,
there exists a small enough $\beta$ and some finite constants $C$
and $C'>0$, such that 
\begin{align*}
\triangle F_{*}^{(m,k,j)} & =F_{j}(\tau_{j-1};\bm{\tau}_{-j}^{(m,k,j)})-F_{j}(\tau_{j}^{*};\bm{\tau}_{-j}^{(m,k,j)})\\
 & >-\varepsilon C+C'
\end{align*}
which is positive from a small enough $\varepsilon$ (hence, small
enough $\beta$). For the cluster $\tilde{\mathcal{C}}_{j'}^{(m,k)}$,
since there is no partition such index $t_{k'}$ in the interval $(\tau_{j'-1}^{(m,k,j')},\tau_{j'+1}^{(m,k,j')}]$,
Proposition \ref{prop:flatness} suggests that $|F_{j'}(\tau;\bm{\tau}_{-j'}^{(m,k,j')})-F_{j'}(\tau_{j}^{*};\bm{\tau}_{-j'}^{(m,k,j')})|<\varepsilon B$
which leads to $|\triangle F_{*}^{(m,k,j')}|<\varepsilon B$.

As a result, for a small enough $\beta$ (hence, small enough $\varepsilon$),
we must have $\triangle F_{*}^{(m,k,j)}>\triangle F_{*}^{(m,k,j')}$.

\section{Proof of Theorem \ref{prop:optimal} \label{appendix:Prop7}}

If $\exists k\in\{1,2,...K-1\}$ satisfying $\tau_{k}^{(m)}\neq t_{k}$,
then for any partition assignment $\{\tau_{1},\tau_{2},...,\tau_{k-1}$,$\tau_{k+1},...,\tau_{K-1}\}$,
there always exists $l\in\{1,2,...,K-1\}$ such that the interval
$(\tau_{l-1},\tau_{l})$ contains at least one of $\{t_{k}\}_{k=1}^{K-1}$.
We first prove that the following two cases will not occur if $\tilde{F}(\bm{\tau}^{(m+1)})=\tilde{F}(\bm{\tau}^{(m)})$.

1) If there exists $l\in\{1,2,...,K-1\}$ such that the interval $(\tau_{l-1},\tau_{l+1})$
contains none of $\{t_{k}\}_{k=1}^{K-1}$, then, there must exist
a interval $(\tau_{a-1},\tau_{a})$, $a\neq l,l+1$ containing at
least one of $\{t_{k}\}_{k=1}^{K-1}$. According to Lemma \ref{lem:binary partition},
we will obtain \textit{a lower total cost $\tilde{F}(\cdot)$} if
we replace the partition $\tau_{l}$ with one of the partition in
$(\tau_{a-1},\tau_{a})$. Thus, the algorithm will not stop if there
exists $l\in\{1,2,...,K-1\}$ such that the interval $(\tau_{l-1},\tau_{l+1})$
contains none of $\{t_{k}\}_{k=1}^{K-1}$.

2) If there exists $l\in\{1,2,...,K-1\}$ such that the interval $(\tau_{l-1},\tau_{l+1})$
contains at least two of $\{t_{k}\}_{k=1}^{K-1}$. Then, there must
exist an interval among $(\tau_{i-1},\tau_{i}]$, $i\in\{1,2,...l-1,l+2,...,K\}$
containing none of $\{t_{k}\}_{k=1}^{K-1}$. This is conflicted with
the conclusion of the above case 1), i.e., the interval $(\tau_{l-1},\tau_{l}]$
and $[\tau_{l},\tau_{l+1})$, for any $l\in\{1,2,...K-1\}$, aways
exist at least one of $\{t_{k}\}_{k=1}^{K-1}$ if the algorithm is
convergent. So, there exists none interval containing at least two
of $\{t_{k}\}_{k=1}^{K-1}$ if $\tilde{F}(\bm{\tau}^{(m+1)})=\tilde{F}(\bm{\tau}^{(m)})$.

Based on the above analysis, we have the conclusion that the interval
$(\tau_{l-1},\tau_{l+1})$ for any $l\in\{1,2,...K-1\}$ contains
exactly one of $\{t_{k}\}_{k=1}^{K-1}$ if $\tilde{F}(\bm{\tau}^{(m+1)})=\tilde{F}(\bm{\tau}^{(m)})$.

However, when Algorithm \ref{alg:merge-and-split} converges, Proposition
\ref{prop:Unimodality} implies that $\tau_{l}=t_{l}$ if $t_{l}$
is the only $\{t_{k}\}_{k=1}^{K-1}$ in $(\tau_{l-1},\tau_{l+1})$.
With this, we conclude that if $\tau_{k}^{(m+1)}=\tau_{k}^{(m)}=t_{k}$
for all $k$ when Algorithm \ref{alg:merge-and-split} has converged.

\section{Detail Derivation of Proof of Proposition \ref{prop:Unimodality}}
\label{app:proof-prop-uni}
\subsection{Detail Derivation of Equation (\ref{eq:dF-1})}

Consider $\tau_{k-1}<i\leq t_{j}$. The cost function $f_{k}(\tau)$
can be written as 
\begin{align*}
	& f_{k}(\tau)=\frac{1}{N}\sum_{i=\tau_{k-1}+1}^{t_{j}}\left(1-\sigma_{\beta}\left(i-\tau\right)\right)g_{11}(\tau)\\
	& \qquad+\frac{1}{N}\sum_{i=t_{j}+1}^{\tau_{k+1}}\left(1-\sigma_{\beta}\left(i-\tau\right)\right)g_{12}(\tau)+\frac{1}{N}\sum_{i=\tau_{k-1}+1}^{t_{j}}\\
	& \qquad\times\sigma_{\beta}\left(i-\tau\right)g_{21}(\tau)+\frac{1}{N}\sum_{i=t_{j}+1}^{\tau_{k+1}}\sigma_{\beta}\left(i-\tau\right)g_{22}(\tau)
\end{align*}
where $g_{11}(\tau)=\|\bm{\epsilon}_{i}-\frac{1}{\tau-\tau_{k-1}}\sum_{l=\tau_{k-1}+1}^{\tau}\bm{\epsilon}_{l}\|_{2}^{2}$,
$g_{12}(\tau)=\|\mathbf{{\rm \bm{\mu}}}_{j+1}-\mathbf{{\rm \bm{\mu}}}_{j}+\bm{\epsilon}_{i}-\frac{1}{\tau-\tau_{k-1}}\sum_{l=\tau_{k-1}+1}^{\tau}\bm{\epsilon}_{l}\|_{2}^{2}$,
$g_{21}(\tau)=\|\frac{\tau_{k+1}-t_{j}}{\tau_{k+1}-\tau}(\mathbf{{\rm \bm{\mu}}}_{j}-\mathbf{{\rm \bm{\mu}}}_{j+1})+\bm{\epsilon}_{i}-\frac{1}{\tau_{k+1}-\tau}\sum_{l=\tau+1}^{\tau_{k+1}}\bm{\epsilon}_{l}\|_{2}^{2}$,
and $g_{22}(\tau)=\|\frac{t_{j}-\tau}{\tau_{k+1}-\tau}(\mathbf{{\rm \bm{\mu}}}_{j}-\mathbf{{\rm \bm{\mu}}}_{j+1})+\bm{\epsilon}_{i}-\frac{1}{\tau_{k+1}-\tau}\sum_{l=\tau+1}^{\tau_{k+1}}\bm{\epsilon}_{l}\|_{2}^{2}$.
Since $\bm{\epsilon}_{l}$ are assumed to be i.i.d. following $\mathcal{N}(0,s^{2}\mathbf{I})$,
the expectation of $g_{11}(\tau)$ can be computed as 
\begin{align*}
	\mathbb{E}\left\{ g_{11}(\tau)\right\}  & =\mathbb{E}\{\bm{\epsilon}_{i}^{\mathrm{T}}\bm{\epsilon}_{i}\}-2\frac{1}{\tau-\tau_{k-1}}\mathbb{E}\Bigg\{\bm{\epsilon}_{i}^{\mathrm{T}}\sum_{l=\tau_{k-1}+1}^{\tau}\bm{\epsilon}_{l}\Bigg\}\\
	& \qquad\qquad+\mathbb{E}\Bigg\{\frac{1}{(\tau-\tau_{k-1})^{2}}\sum_{o,l=1}^{\tau}\bm{\epsilon}_{o}^{\mathrm{T}}\bm{\epsilon}_{l}\Bigg\}\\
	& =s_{0}^{2}-2s_{0}^{2}\frac{\mathbb{I}\{i\leq\tau\}}{\tau-\tau_{k-1}}+\frac{s_{0}^{2}}{\tau-\tau_{k-1}}\\
	& =s_{0}^{2}\Big(1+\frac{1}{\tau-\tau_{k-1}}-\frac{2\mathbb{I}\{i\leq\tau\}}{\tau-\tau_{k-1}}\Big).
\end{align*}

Likewise, we obtain $\mathbb{E}\{g_{21}(\tau)\}=\left(\frac{\tau_{k+1}-t_{j}}{\tau_{k+1}-\tau}\right)^{2}\|\mathbf{{\rm \bm{\mu}}}_{j}-\mathbf{{\rm \bm{\mu}}}_{j+1}\|_{2}^{2}+s_{0}^{2}\left(1+\frac{1}{\tau_{k+1}-\tau}-\frac{2\mathbb{I}\{i>\tau\}}{\tau_{k+1}-\tau}\right)$,
$\mathbb{E}\left\{ g_{12}(\tau)\right\} =\|\mathbf{{\rm \bm{\mu}}}_{j}-\mathbf{{\rm \bm{\mu}}}_{j+1}\|_{2}^{2}+s_{0}^{2}\left(1+\frac{1}{\tau-\tau_{k-1}}-\frac{2\mathbb{I}\{i\leq\tau\}}{\tau-\tau_{k-1}}\right)$,
and $\mathbb{E}\{g_{22}(\tau)\}=\left(\frac{t_{j}-\tau}{\tau_{k+1}-\tau}\right)^{2}\|\mathbf{{\rm \bm{\mu}}}_{j}-\mathbf{{\rm \bm{\mu}}}_{j+1}\|_{2}^{2}+s_{0}^{2}\Big(1+\frac{1}{\tau_{k+1}-\tau}-\frac{2\mathbb{I}\{i>\tau\}}{\tau_{k+1}-\tau}\Big)$.
Thus, the expectation of the cost function $f_{k}(\tau)$ can be computed
as
\begin{align}
	& F_{k}(\tau)=\frac{1}{N}\|\mathbf{{\rm \bm{\mu}}}_{j}-\mathbf{{\rm \bm{\mu}}}_{j+1}\|_{2}^{2}\Bigg\{(\tau_{k+1}-t_{j})+\Big(\frac{\tau_{k+1}-t_{j}}{\tau_{k+1}-\tau}\Big)^{2}\nonumber \\
	& \qquad\times\sum_{i=\tau_{k-1}+1}^{t_{j}}\sigma_{\beta}\left(i-\tau\right)+\Big(\Big(\frac{t_{j}-\tau}{\tau_{k+1}-\tau}\Big)^{2}-1\Big)\nonumber \\
	& \qquad\times\sum_{i=t_{j}+1}^{\tau_{k+1}}\sigma_{\beta}\left(i-\tau\right)\Bigg\}+\frac{1}{N}s_{0}^{2}\Bigg\{\tau_{k+1}\left(1+\frac{1}{\tau-\tau_{k-1}}\right)\nonumber \\
	& \qquad-2+\sum_{i=\tau_{k-1}+1}^{\tau_{k+1}}\sigma_{\beta}\left(i-\tau\right)\Big(\frac{1}{\tau_{k+1}-\tau}-\frac{1}{\tau-\tau_{k-1}}\Big)\nonumber \\
	& \qquad+2\sum_{i=\tau_{k-1}+1}^{\tau}\frac{\sigma_{\beta}\left(i-\tau\right)}{\tau-\tau_{k-1}}-2\sum_{i=\tau+1}^{\tau_{k+1}}\frac{\sigma_{\beta}\left(i-\tau\right)}{\tau_{k+1}-\tau}\Bigg\}.\label{eq:F K=00003D2-1}
\end{align}
In addition, the difference $\triangle F_{k}(\tau)=F_{k}(\tau)-F_{k}(\tau-1)$
can be computed as (\ref{eq:dF-1}), i.e., 
\begin{align*}
	\triangle F_{k}(\tau) & =\frac{1}{N}\|\mathbf{{\rm \bm{\mu}}}_{j}-\mathbf{{\rm \bm{\mu}}}_{j+1}\|_{2}^{2}u(\tau,t_{j},\beta)+\frac{1}{N}s_{0}^{2}\gamma(\tau,\beta)
\end{align*}
where

\begin{align*}
	& u(\tau,t_{j},\beta)=\left(\frac{\tau_{k+1}-t_{j}}{\tau_{k+1}-\tau}\right)^{2}\sum_{i=\tau_{k-1}+1}^{t_{j}}\sigma_{\beta}\left(i-\tau\right)\\
	& +\left(1-\frac{\tau_{k+1}-t_{j}}{\tau_{k+1}-\tau}\right)^{2}\sum_{i=t_{j}+1}^{\tau_{k+1}}\sigma_{\beta}\left(i-\tau\right)-\left(\frac{\tau_{k+1}-t_{j}}{\tau_{k+1}-\tau+1}\right)^{2}\\
	& \times\sum_{i=\tau_{k-1}+2}^{t_{j}+1}\sigma_{\beta}\left(i-\tau\right)-\left(1-\frac{\tau_{k+1}-t_{j}}{\tau_{k+1}-\tau+1}\right)^{2}\\
	& \times\sum_{i=t_{j}+2}^{\tau_{k+1}+1}\sigma_{\beta}\left(i-\tau\right)+\sum_{i=t_{j}+2}^{\tau_{k+1}+1}\sigma_{\beta}\left(i-\tau\right)-\sum_{i=t_{j}+1}^{\tau_{k+1}}\sigma_{\beta}\left(i-\tau\right)
\end{align*}
and
\begin{align*}
	\gamma(\tau,\beta) & =(\tau_{k+1}-\tau_{k-1})\left(\frac{1}{\tau-\tau_{k-1}}-\frac{1}{\tau-1-\tau_{k-1}}\right)\\
	& +\frac{\sum_{i=\tau_{k-1}+1}^{\tau}\sigma_{\beta}\left(i-\tau\right)}{\tau_{k+1}-\tau}-\frac{\sum_{i=\tau_{k-1}+2}^{\tau}\sigma_{\beta}\left(i-\tau\right)}{\tau_{k+1}-\tau+1}\\
	& \qquad+\frac{\sum_{i=\tau+1}^{\tau_{k+1}+1}\sigma_{\beta}\left(i-\tau\right)}{\tau-1-\tau_{k-1}}-\frac{\sum_{i=\tau+1}^{\tau_{k+1}}\sigma_{\beta}\left(i-\tau\right)}{\tau-\tau_{k-1}}\\
	& \qquad+\frac{\sum_{i=\tau_{k-1}+1}^{\tau}\sigma_{\beta}\left(i-\tau\right)}{\tau-\tau_{k-1}}-\frac{\sum_{i=\tau+1}^{\tau_{k+1}}\sigma_{\beta}\left(i-\tau\right)}{\tau_{k+1}-\tau}\\
	& \qquad+\frac{\sum_{i=\tau+1}^{\tau_{k+1}+1}\sigma_{\beta}\left(i-\tau\right)}{\tau_{k+1}-\tau+1}-\frac{\sum_{i=\tau_{k-1}+2}^{\tau}\sigma_{\beta}\left(i-\tau\right)}{\tau-1-\tau_{k-1}}.
\end{align*}

\subsection{Detail Derivation of the Bound of $\bigtriangleup F_{k}(\tau)$ for
	Case 1}

The term $u(\tau,t_{j},\beta)$ in (\ref{eq:dF-1}) can be upper bounded
as $u(\tau,t_{j},\beta)\leq\varepsilon B_{1}-B_{2}$, where 
\begin{align*}
	B_{1} & =(\tau_{k+1}-t_{j})^{2}\Big(\frac{\tau-\tau_{k-1}}{(\tau_{k+1}-\tau)^{2}}+\frac{t_{j}-\tau+1}{(\tau_{k+1}-\tau+1)(\tau_{k+1}-t_{j})}\\
	& \qquad\qquad+\frac{1}{\tau_{k+1}-t_{j}}\Big)
\end{align*}
and $B_{2}=(\tau_{k+1}-t_{j})^{2}\left(\frac{1}{\tau_{k+1}-\tau}-\frac{1}{\tau_{k+1}-\tau+1}\right)$.
Since $\tau_{k-1}<\tau<t_{j}<\tau_{k+1}$, it can be easily verified
that $B_{1}>0$ and $B_{2}>0$.

Recall the upper bound of $\gamma(\tau,\beta)$, i.e., $\gamma(\tau,\beta)\leq\varepsilon\left(4+\frac{(\tau_{k+1}-\tau_{k-1})^{2}}{\tau_{k+1}-\tau_{k-1}-1}\right)$.
As a result, the difference in (\ref{eq:dF-1}) can be upper bounded
as 
\begin{align*}
	\bigtriangleup F_{k}(\tau) & <\frac{1}{N}\|\mathbf{{\rm \bm{\mu}}}_{j}-\mathbf{{\rm \bm{\mu}}}_{j+1}\|_{2}^{2}\left(\varepsilon B_{1}-B_{2}\right)+\frac{1}{N}s_{0}^{2}\varepsilon\Big(4\\
	& \qquad\qquad+\frac{(\tau_{k+1}-\tau_{k-1})^{2}}{\tau_{k+1}-\tau_{k-1}-1}\Big)\\
	& =\frac{1}{N}\varepsilon\Bigg(\|\mathbf{{\rm \bm{\mu}}}_{j}-\mathbf{{\rm \bm{\mu}}}_{j+1}\|_{2}^{2}B_{1}+s_{0}^{2}\Big(4\\
	& +\frac{(\tau_{k+1}-\tau_{k-1})^{2}}{\tau_{k+1}-\tau_{k-1}-1}\Big)\Bigg)-\frac{1}{N}\|\mathbf{{\rm \bm{\mu}}}_{j}-\mathbf{{\rm \bm{\mu}}}_{j+1}\|_{2}^{2}B_{2}\\
	& =\varepsilon C_{1}-C_{2}<0
\end{align*}
where $C_{1}=\frac{1}{N}(\|\mathbf{{\rm \bm{\mu}}}_{j}-\mathbf{{\rm \bm{\mu}}}_{j+1}\|_{2}^{2}B_{1}+s_{0}^{2}(4+(\tau_{k+1}-\tau_{k-1})^{2}/(\tau_{k+1}-\tau_{k-1}-1)))$,
and $C_{2}=\frac{1}{N}\|\mathbf{{\rm \bm{\mu}}}_{j}-\mathbf{{\rm \bm{\mu}}}_{j+1}\|_{2}^{2}B_{2}$,
if $\varepsilon<\left(\frac{B_{1}}{B_{2}}+\frac{1}{B_{2}}\big(4+\frac{(\tau_{k+1}-\tau_{k-1})^{2}}{\tau_{k+1}-1-\tau_{k-1}}\big)\cdot\frac{s_{0}^{2}}{\|\mathbf{{\rm \bm{\mu}}}_{j}-\mathbf{{\rm \bm{\mu}}}_{j+1}\|_{2}^{2}}\right)^{-1}$,
or, equivalently, if $\beta<\frac{1}{2}\mathrm{log}^{-1}\Big(\frac{B_{1}}{B_{2}}+\frac{1}{B_{2}}\big(4+\frac{(\tau_{k+1}-\tau_{k-1})^{2}}{\tau_{k+1}-1-\tau_{k-1}}\big)\cdot\frac{s_{0}^{2}}{\|\mathbf{{\rm \bm{\mu}}}_{j}-\mathbf{{\rm \bm{\mu}}}_{j+1}\|_{2}^{2}}-1\Big)$.

\subsection{Detail Derivation of the Bound of $\bigtriangleup F_{k}(\tau)$ for
	Case 2}

Consider $t_{j}<i\leq\tau_{k+1}$. The lower bound of $u(\tau,t_{j},\beta)$
is given by $a(\tau,t_{j},\beta)\geq B_{4}-\varepsilon B_{3}$, where
$B_{3}=(t_{j}-\tau_{k-1})^{2}\left(\frac{\tau_{k+1}-\tau}{(\tau-\tau_{k-1})^{2}}+\frac{\tau-t_{j}-1}{(\tau-\tau_{k-1}-1)(t_{j}-\tau_{k-1})}+\frac{1}{t_{j}-\tau_{k-1}}\right)$,
and $B_{4}=(t_{j}-\tau_{k-1})^{2}\left(\frac{1}{\tau-\tau_{k-1}-1}-\frac{1}{\tau-\tau_{k-1}}\right)$.
Since $\tau_{k-1}<t_{j}<\tau<\tau_{k+1}$, one can verify that $B_{3}>0$
and $B_{4}>0$. In addition, the lower bound of $\gamma(\tau,\beta)$
is $\gamma(\tau,\beta)\geq-\frac{(\tau_{k+1}-\tau_{k-1})\varepsilon}{2}$.
As a result, the difference in (\ref{eq:dF-1}) can be lower bounded
as 
\begin{align*}
	\bigtriangleup F_{k}(\tau) & >\frac{1}{N}\|\mathbf{{\rm \bm{\mu}}}_{j}-\mathbf{{\rm \bm{\mu}}}_{j+1}\|_{2}^{2}\left(B_{4}-\varepsilon B_{3}\right)-\frac{(\tau_{k+1}-\tau_{k-1})\varepsilon}{2N}\\
	& =\frac{1}{N}\|\mathbf{{\rm \bm{\mu}}}_{j}-\mathbf{{\rm \bm{\mu}}}_{j+1}\|_{2}^{2}B_{4}-\frac{1}{N}\varepsilon\Big(\|\mathbf{{\rm \bm{\mu}}}_{j}-\mathbf{{\rm \bm{\mu}}}_{j+1}\|_{2}^{2}B_{3}\\
	& \qquad\qquad\qquad\qquad\qquad+\frac{\tau_{k+1}-\tau_{k-1}}{2}\Big)\\
	& =C_{2}'-\varepsilon C_{1}'>0
\end{align*}
where $C_{1}'=\frac{1}{N}(\|\mathbf{{\rm \bm{\mu}}}_{j}-\mathbf{{\rm \bm{\mu}}}_{j+1}\|_{2}^{2}B_{3}+\frac{\tau_{k+1}-\tau_{k-1}}{2})$,
and $C_{2}'=\frac{1}{N}\|\mathbf{{\rm \bm{\mu}}}_{j}-\mathbf{{\rm \bm{\mu}}}_{j+1}\|_{2}^{2}B_{4}$,
if $\varepsilon<\left(\frac{B_{3}}{B_{4}}+\frac{\tau_{k+1}-\tau_{k-1}}{2B_{4}}\cdot\frac{s_{0}^{2}}{\|\mathbf{{\rm \bm{\mu}}}_{j}-\mathbf{{\rm \bm{\mu}}}_{j+1}\|_{2}^{2}}\right)^{-1}$,
or, equivalently, if $\beta<\frac{1}{2}\mathrm{ln}^{-1}\Big(\frac{B_{3}}{B_{4}}+\frac{\tau_{k+1}-\tau_{k-1}}{2B_{4}}\cdot\frac{s_{0}^{2}}{\|\mathbf{{\rm \bm{\mu}}}_{j}-\mathbf{{\rm \bm{\mu}}}_{j+1}\|_{2}^{2}}-1\Big)$.


.

\section{Detail Derivation of Proof of Proposition \ref{prop:monotonicity-near-boundary}}
\label{app:deri-prop-mono}
\subsection{Detail Derivation of the Bound of $\bigtriangleup F_{k}(\tau)$ for
	Case 1}

Consider $t_{k-1}<\tau\leq t_{j}$. The cost function $f_{k}(\tau)$
can be written as 
\begin{align*}
	& f_{k}(\tau)=\frac{1}{N}\sum_{i=\tau_{k-1}+1}^{t_{j}}\left(1-\sigma_{\beta}\left(i-\tau\right)\right)g_{11}(\tau)\\
	& +\sum_{a=j+1}^{j+J+1}\sum_{i=t_{a-1}+1}^{t_{a}}\frac{\left(1-\sigma_{\beta}\left(i-\tau\right)\right)g_{12}(\tau)}{N}+\sum_{i=\tau_{k-1}+1}^{t_{j}}g_{21}(\tau)\\
	& \times\frac{\sigma_{\beta}\left(i-\tau\right)}{N}+\frac{1}{N}\sum_{a=j+1}^{j+J+1}\sum_{i=t_{a-1}+1}^{t_{a}}\sigma_{\beta}\left(i-\tau\right)g_{22}(\tau)
\end{align*}
where $g_{11}(\tau)=\|\bm{\epsilon}_{i}-\frac{1}{\tau-\tau_{k-1}}\sum_{l=\tau_{k-1}+1}^{\tau}\bm{\epsilon}_{l}\|_{2}^{2}$,
$g_{12}(\tau)=\|\mathbf{{\rm \bm{\mu}}}_{a}-\mathbf{{\rm \bm{\mu}}}_{j}+\bm{\epsilon}_{i}-\frac{1}{\tau-\tau_{k-1}}\sum_{l=\tau_{k-1}+1}^{\tau}\bm{\epsilon}_{l}\|_{2}^{2}$,
$g_{21}(\tau)=\|\frac{\tau_{k+1}-t_{j}}{\tau_{k+1}-\tau}\left(\mathbf{{\rm \bm{\mu}}}_{j}-\bm{\eta}(j+1,j+J+1)\right)+\bm{\epsilon}_{i}-\frac{1}{\tau_{k+1}-\tau}\sum_{l=\tau+1}^{\tau_{k+1}}\bm{\epsilon}_{l}\|_{2}^{2}$,
and $g_{22}(\tau)=\|\mathbf{{\rm \bm{\mu}}}_{a}-\frac{(t_{j}-\tau)\mathbf{{\rm \bm{\mu}}}_{j}+(\tau_{k+1}-t_{j})\bm{\eta}(j+1,j+J+1)}{\tau_{k+1}-\tau}+\bm{\epsilon}_{i}-\frac{1}{\tau_{k+1}-\tau}\sum_{l=\tau+1}^{\tau_{k+1}}\bm{\epsilon}_{l}\|_{2}^{2}$.

Since $\bm{\epsilon}_{l}$ are assumed to be i.i.d. following $\mathcal{N}(0,s^{2}\mathbf{I})$,
the expectation of $g_{11}(\tau)$ can be computed as 
\begin{align*}
	\mathbb{E}\left\{ g_{11}(\tau)\right\}  & =\mathbb{E}\{\bm{\epsilon}_{i}^{\mathrm{T}}\bm{\epsilon}_{i}\}-2\frac{1}{\tau}\mathbb{E}\Big\{\bm{\epsilon}_{i}^{\mathrm{T}}\sum_{j=\tau_{k-1}+1}^{\tau}\bm{\epsilon}_{j}\Big\}\\
	& \qquad+\mathbb{E}\Big\{\frac{1}{(\tau-\tau_{k-1})^{2}}\sum_{o,l=\tau_{k-1}+1}^{\tau}\bm{\epsilon}_{o}^{\mathrm{T}}\bm{\epsilon}_{l}\Big\}\\
	& =s_{0}^{2}-2s_{0}^{2}\frac{\mathbb{I}(i\leq\tau)}{\tau-\tau_{k-1}}+\frac{s_{0}^{2}}{\tau-\tau_{k-1}}\\
	& =s_{0}^{2}\Big(1+\frac{1}{\tau-\tau_{k-1}}-\frac{2\mathbb{I}(i\leq\tau)}{\tau-\tau_{k-1}}\Big).
\end{align*}

Likewise, define $\phi(\mathbf{{\rm \bm{\mu}}}_{i},\mathbf{{\rm \bm{\mu}}}_{j})=\|\mathbf{{\rm \bm{\mu}}}_{i}-\mathbf{{\rm \bm{\mu}}}_{j}\|_{2}^{2}$.
We obtain $\mathbb{E}\{g_{21}(\tau)\}=\left(\frac{\tau_{k+1}-t_{j}}{\tau_{k+1}-\tau}\right)^{2}\phi(\mathbf{{\rm \bm{\mu}}}_{l},\bm{\eta}(j+1,j+J+1))+s_{0}^{2}\left(1+\frac{1}{\tau_{k+1}-\tau}-\frac{2\mathbb{I}(i>\tau)}{\tau_{k+1}-\tau}\right)$,
$\mathbb{E}\left\{ g_{12}(\tau)\right\} =\phi(\mathbf{{\rm \bm{\mu}}}_{j},\mathbf{{\rm \bm{\mu}}}_{a})+s_{0}^{2}\left(1+\frac{1}{\tau-\tau_{k-1}}-\frac{2\mathbb{I}(i\leq\tau)}{\tau-\tau_{k-1}}\right)$,
and $\mathbb{E}\{g_{22}(\tau)\}=\|\mathbf{{\rm \bm{\mu}}}_{a}-\frac{(t_{j}-\tau)\mathbf{{\rm \bm{\mu}}}_{j}+(\tau_{k+1}-t_{j})\bm{\eta}(j+1,j+J+1)}{\tau_{k+1}-\tau}\|_{2}^{2}+s_{0}^{2}\left(1+\frac{1}{\tau_{k+1}-\tau}-\frac{2\mathbb{I}(i>\tau)}{\tau_{k+1}-\tau}\right)$.

Thus, the expectation of the cost function $F_{k}(\tau)=\mathbb{E}\{f_{k}(\tau;\tau_{k-1},\tau_{k+1})\}$
can be computed as
\begin{align}
	& F_{k}(\tau)=\frac{1}{N}\sum_{a=j+1}^{j+J+1}\sum_{i=t_{a-1}+1}^{t_{a}}\bigg\{\sigma_{\beta}\left(i-\tau\right)\Big\|\mathbf{{\rm \bm{\mu}}}_{a}-\frac{(t_{j}-\tau)\mathbf{{\rm \bm{\mu}}}_{j}}{\tau_{k+1}-\tau}\label{eq:F K=00003D2}\\
	& -\frac{(\tau_{k+1}-t_{j})\bm{\eta}(j+1,j+J+1)}{\tau_{k+1}-\tau}\Big\|_{2}^{2}+\Big(1-\sigma_{\beta}\left(i-\tau\right)\Big)\nonumber \\
	& \times||\mathbf{{\rm \bm{\mu}}}_{j}-\mathbf{{\rm \bm{\mu}}}_{a}||_{2}^{2}\bigg\}+\frac{1}{N}\left(\frac{\tau_{k+1}-t_{j}}{\tau_{k+1}-\tau}\right)^{2}\sum_{i=\tau_{k-1}+1}^{t_{j}}\sigma_{\beta}\left(i-\tau\right)\nonumber \\
	& \times\|\mathbf{{\rm \bm{\mu}}}_{j}-\bm{\eta}(j+1,j+J+1)\|_{2}^{2}+\frac{s_{0}^{2}}{N}\bigg\{\left(\tau_{k+1}+\frac{\tau_{k+1}}{\tau-\tau_{k-1}}\right)\nonumber \\
	& -2+\sum_{i=\tau_{k-1}+1}^{\tau_{k+1}}\sigma_{\beta}\left(i-\tau\right)\left(\frac{1}{\tau_{k+1}-\tau}-\frac{1}{\tau-\tau_{k-1}}\right)\nonumber \\
	& +2\sum_{i=\tau_{k-1}+1}^{\tau}\frac{\sigma_{\beta}\left(i-\tau\right)}{\tau-\tau_{k-1}}-2\sum_{i=\tau+1}^{\tau_{k+1}}\frac{\sigma_{\beta}\left(i-\tau\right)}{\tau_{k+1}-\tau}\bigg\}.\nonumber 
\end{align}
In addition, the difference $\triangle F_{k}(\tau)=F_{k}(\tau)-F_{k}(\tau-1)$
can be computed as
\begin{align}
	\triangle F_{k}(\tau) & =\frac{1}{N}\phi(\mathbf{{\rm \bm{\mu}}}_{j},\bm{\eta}(j+1,j+J+1))u(\tau,\{t_{k}\}_{k=j}^{j+J},\beta)\label{eq:dF}\\
	& \qquad\qquad+\frac{1}{N}s_{0}^{2}\gamma(\tau,\beta)\nonumber 
\end{align}
where
\begin{align*}
	& u(\tau,\{t_{k}\}_{k=j}^{j+J},\beta)\\
	& =\bigg\{\Big(\frac{\tau_{k+1}-t_{j}}{\tau_{k+1}-\tau}\Big)^{2}\sum_{a=j+1}^{j+J+1}\sigma_{\beta}\left(t_{a-1}+1\right)-\Big(\frac{\tau_{k+1}-t_{j}}{\tau_{k+1}-\tau+1}\Big)^{2}\\
	& \times\sum_{a=j+1}^{j+J+1}\sigma_{\beta}\left(t_{a}+1\right)+\bigg(\Big(\frac{\tau_{k+1}-t_{j}}{\tau_{k+1}-\tau}\Big)^{2}-\Big(\frac{\tau_{k+1}-t_{j}}{\tau_{k+1}-\tau+1}\Big)^{2}\bigg)\\
	& \times\sum_{i=t_{j}+2}^{\tau_{k+1}}\sigma_{\beta}\left(i-\tau\right)\bigg\}+\bigg\{\Big(\frac{\tau_{k+1}-t_{j}}{\tau_{k+1}-\tau}\Big)^{2}\sum_{i=\tau_{k-1}+1}^{t_{j}}\sigma_{\beta}\left(i-\tau\right)\\
	& -\left(\frac{\tau_{k+1}-t_{j}}{\tau_{k+1}-\tau+1}\right)^{2}\sum_{i=\tau_{k-1}+2}^{t_{j}+1}\sigma_{\beta}\left(i-\tau\right)\bigg\}\\
	& -\frac{2(\tau_{j+J+1}-t_{j})}{\phi(\mathbf{{\rm \bm{\mu}}}_{j},\bm{\eta}(j+1,j+J+1))}\left(\mathbf{{\rm \bm{\mu}}}_{j}-\bm{\eta}(j+1,j+J+1)\right)^{\mathrm{T}}\\
	& \times\sum_{a=j+1}^{j+J+1}\left(\mathbf{{\rm \bm{\mu}}}_{j}-\mathbf{{\rm \bm{\mu}}}_{a}\right)\bigg\{\frac{\sigma_{\beta}\left(t_{a-1}+1\right)}{\tau_{k+1}-\tau}-\frac{\sigma_{\beta}\left(t_{a}+1\right)}{\tau_{k+1}-\tau+1}\\
	& +\Big(\frac{1}{\tau_{k+1}-\tau}-\frac{1}{\tau_{k+1}-\tau+1}\Big)\sum_{i=t_{a-1}+2}^{t_{a}}\sigma_{\beta}\left(i-\tau\right)\bigg\}
\end{align*}
and
\begin{align*}
	& \gamma(\tau,\beta)=(\tau_{k+1}-\tau_{k-1})\left(\frac{1}{\tau-\tau_{k-1}}-\frac{1}{\tau-1-\tau_{k-1}}\right)\\
	& +\frac{\sum_{i=\tau_{k-1}+1}^{\tau}\sigma_{\beta}\left(i-\tau\right)}{\tau_{k+1}-\tau}-\frac{\sum_{i=\tau_{k-1}+2}^{\tau}\sigma_{\beta}\left(i-\tau\right)}{\tau_{k+1}-\tau+1}\\
	& +\frac{\sum_{i=\tau+1-\tau_{k-1}}^{\tau_{k+1}+1}\sigma_{\beta}\left(i-\tau\right)}{\tau-1-\tau_{k-1}}-\frac{\sum_{i=\tau+1}^{\tau_{k+1}}\sigma_{\beta}\left(i-\tau\right)}{\tau-\tau_{k-1}}\\
	& +\frac{\sum_{i=\tau_{k-1}+1}^{\tau}\sigma_{\beta}\left(i-\tau\right)}{\tau-\tau_{k-1}}-\frac{\sum_{i=\tau+1}^{\tau_{k+1}}\sigma_{\beta}\left(i-\tau\right)}{\tau_{k+1}-\tau}\\
	& +\frac{\sum_{i=t+1}^{\tau_{k+1}+1}\sigma_{\beta}\left(i-\tau\right)}{\tau_{k+1}-\tau+1}-\frac{\sum_{i=\tau_{k-1}+2}^{\tau}\sigma_{\beta}\left(i-\tau\right)}{\tau-\tau_{k-1}-1}.
\end{align*}

The term $u(\tau,\{t_{k}\}_{k=j}^{j+J},\beta)$ can be upper bounded
as $u(\tau,\{t_{k}\}_{k=j}^{j+J},\beta)\leq\varepsilon B_{1}-B_{2}$,
where $B_{1}=(\tau_{k+1}-t_{j})^{2}\left(\frac{\tau}{\left(\tau_{k+1}-\tau\right)^{2}}+\frac{2}{\tau_{k+1}-\tau}+\frac{1}{\tau_{k+1}-\tau+1}+\frac{J}{\left(\tau_{k+1}-\tau+1\right)^{2}}\right)$,
and $B_{2}=(\tau_{k+1}-t_{j})^{2}\left(\frac{\tau_{k+1}-\tau-J}{\left(\tau_{k+1}-\tau\right)^{2}}-\frac{\tau_{k+1}-\tau+1-J}{\left(\tau_{k+1}-\tau+1\right)^{2}}\right)$.
Since $\tau_{k-1}<\tau<t_{j}<\tau_{k+1}$, it can be easily verified
that $B_{1}>0$ and $B_{2}>0$ for $\tau<\tau_{k+1}-J$. The equality
$\tau<\tau_{k+1}-J$ always holds for Case 1.

Likewise, the term $\gamma(\tau,\beta)$ can be bounded as $-\varepsilon\frac{(\tau_{k+1}-\tau_{k-1})^{2}}{(\tau_{k+1}-\tau+1)(\tau-1-\tau_{k-1})}<\gamma(\tau,\beta)<\varepsilon\frac{(\tau_{k+1}-\tau_{k-1})^{2}}{(\tau-\tau_{k-1})(\tau_{k+1}-\tau)}$.

As a result, the difference in (\ref{eq:dF}) can be upper bounded
as 
\begin{align}
	\bigtriangleup F_{k}(\tau) & <\phi(\mathbf{{\rm \bm{\mu}}}_{j},\bm{\eta}(j+1,j+J+1))\left(\varepsilon B_{1}-B_{2}\right)\nonumber \\
	& \qquad\qquad+s_{0}^{2}\varepsilon\frac{(\tau_{k+1}-\tau_{k-1})^{2}}{\tau(\tau_{k+1}-\tau)}\nonumber \\
	& =\varepsilon\Big(\phi(\mathbf{{\rm \bm{\mu}}}_{j},\bm{\eta}(j+1,j+J+1))B_{1}\nonumber \\
	& \qquad+\frac{s_{0}^{2}}{\tau-\tau_{k-1}}\frac{(\tau_{k+1}-\tau_{k-1})^{2}}{\tau_{k+1}-\tau}\Big)\nonumber \\
	& \qquad-\phi(\mathbf{{\rm \bm{\mu}}}_{j},\bm{\eta}(j+1,j+J+1))B_{2}\nonumber \\
	& =\varepsilon C_{3}-C_{4}<0.\label{eq:dF_1}
\end{align}
where $C_{3}=\phi(\mathbf{{\rm \bm{\mu}}}_{j},\bm{\eta}(j+1,j+J+1))B_{1}+\frac{s_{0}^{2}}{\tau-\tau_{k-1}}\frac{(\tau_{k+1}-\tau_{k-1})^{2}}{\tau_{k+1}-\tau}$,
and $C_{4}=\phi(\mathbf{{\rm \bm{\mu}}}_{j},\bm{\eta}(j+1,j+J+1))B_{2}$,
if $\varepsilon<\left(\frac{B_{1}}{B_{2}}+\frac{(\tau_{k+1}-\tau_{k-1})^{2}}{(\tau-\tau_{k-1})(\tau_{k+1}-\tau)B_{2}}\cdot\frac{s_{0}^{2}}{\phi(\mathbf{{\rm \bm{\mu}}}_{j},\bm{\eta}(j+1,j+J+1))}\right)^{-1}$,
or, equivalently, if $\beta<\frac{1}{2}\mathrm{ln}^{-1}\Big(\frac{B_{1}}{B_{2}}+\frac{(\tau_{k+1}-\tau_{k-1})^{2}}{(\tau-\tau_{k-1})(\tau_{k+1}-\tau)B_{2}}\cdot\frac{s_{0}^{2}}{\phi(\mathbf{{\rm \bm{\mu}}}_{j},\bm{\eta}(j+1,j+J+1))}-1\Big)$,
and $\phi(\mathbf{{\rm \bm{\mu}}}_{j},\bm{\eta}(j+1,j+J+1))\neq0$.
In addition, 
\[
\bigtriangleup F(\tau)>-\phi(\mathbf{{\rm \bm{\mu}}}_{j},\bm{\eta}(j+1,j+J+1))B_{2}
\]
and 
\begin{align*}
	& \bigtriangleup F(\tau)<\frac{1}{1+\mathrm{exp}(\frac{1}{2\beta})}\Big(\phi(\mathbf{{\rm \bm{\mu}}}_{j},\bm{\eta}(j+1,j+J+1))B_{1}\\
	& +s_{0}^{2}\frac{(\tau_{k+1}-\tau_{k-1})^{2}}{(\tau-\tau_{k-1})(\tau_{k+1}-\tau)}\Big)-\phi(\mathbf{{\rm \bm{\mu}}}_{j},\bm{\eta}(j+1,j+J+1))B_{2}
\end{align*}
for $\forall\beta>0$.

\subsection{Detail Derivation of the Bound of $\bigtriangleup F_{k}(\tau)$ for
	Case 2}

Consider $t_{j+J}<\tau\leq\tau_{k+1}$. The difference in (\ref{eq:dF})
can be lower bounded as
\begin{align}
	\bigtriangleup F_{k}(\tau) & >\phi(\bm{\eta}(j,j+J),\mathbf{{\rm \bm{\mu}}}_{K})\left(B_{4}-\varepsilon B_{3}\right)\nonumber \\
	& \qquad-\varepsilon\frac{(\tau_{k+1}-\tau_{k-1})^{2}}{(\tau_{k+1}-\tau+1)(\tau-\tau_{k-1}-1)}\nonumber \\
	& =-\varepsilon\Bigg(\phi(\bm{\eta}(j,j+J),\mathbf{{\rm \bm{\mu}}}_{j+J+1})B_{3}+\frac{(\tau_{k+1}-\tau_{k-1})^{2}}{(\tau_{k+1}-\tau+1)}\nonumber \\
	& \qquad\times\frac{1}{(\tau-\tau_{k-1}-1)}\Bigg)+\phi(\bm{\eta}(j,j+J),\mathbf{{\rm \bm{\mu}}}_{K})B_{4}\nonumber \\
	& =-\varepsilon C_{3}'+C_{4}'>0\label{eq:dF_2}
\end{align}
where $C_{3}'=\phi(\bm{\eta}(j,j+J),\mathbf{{\rm \bm{\mu}}}_{j+J+1})B_{3}+\frac{(\tau_{k+1}-\tau_{k-1})^{2}}{(\tau_{k+1}-\tau+1)}\frac{1}{(\tau-\tau_{k-1}-1)}$,
and $C_{4}'=\phi(\bm{\eta}(j,j+J),\mathbf{{\rm \bm{\mu}}}_{K})B_{4}$,
if $\varepsilon<\left(\frac{B_{3}}{B_{4}}+\frac{(\tau_{k+1}-\tau_{j})^{2}}{(\tau_{k+1}-t+1)(\tau-\tau_{k-1}-1)B_{4}}\cdot\frac{s_{0}^{2}}{\phi(\bm{\eta}(j,j+J),\mathbf{{\rm \bm{\mu}}}_{j+J+1})}\right)^{-1}$,
or, equivalently, if 
\begin{align*}
	\beta & <\frac{1}{2}\Big[\mathrm{ln}\Big(\frac{B_{3}}{B_{4}}+\frac{(\tau_{k+1}-\tau_{k-1})^{2}}{(\tau_{k+1}-\tau+1)(\tau-\tau_{k-1}-1)B_{4}}\\
	& \qquad\qquad\times\frac{s_{0}^{2}}{\phi(\bm{\eta}(j,j+J),\mathbf{{\rm \bm{\mu}}}_{j+J+1})}-1\Big)\Big]^{-1}
\end{align*}
and $\phi(\bm{\eta}(j,j+J),\mathbf{{\rm \bm{\mu}}}_{j+J+1})\neq0$.
In addition, $\bigtriangleup F(t)<\phi(\bm{\eta}(j,j+J),\mathbf{{\rm \bm{\mu}}}_{j+J+1})B_{4}$
and
\begin{align*}
	\bigtriangleup F(\tau) & >-\frac{1}{1+\mathrm{exp}(\frac{1}{2\beta})}\Big(\phi(\bm{\eta}(j,j+J),\mathbf{{\rm \bm{\mu}}}_{j+J+1})B_{3}\\
	& \qquad+\frac{(\tau_{k+1}-\tau_{j})^{2}}{(\tau_{k+1}-\tau+1)(\tau-\tau_{k-1}-1)}\Big)\\
	& \qquad\qquad+\phi(\bm{\eta}(j,j+J),\mathbf{{\rm \bm{\mu}}}_{j+J+1})B_{4}
\end{align*}
for $\forall\beta>0$.

\bibliographystyle{IEEEtran}
\bibliography{my_ref}

\begin{IEEEbiography}[{\includegraphics[width=1in,height=1.25in,clip,keepaspectratio]{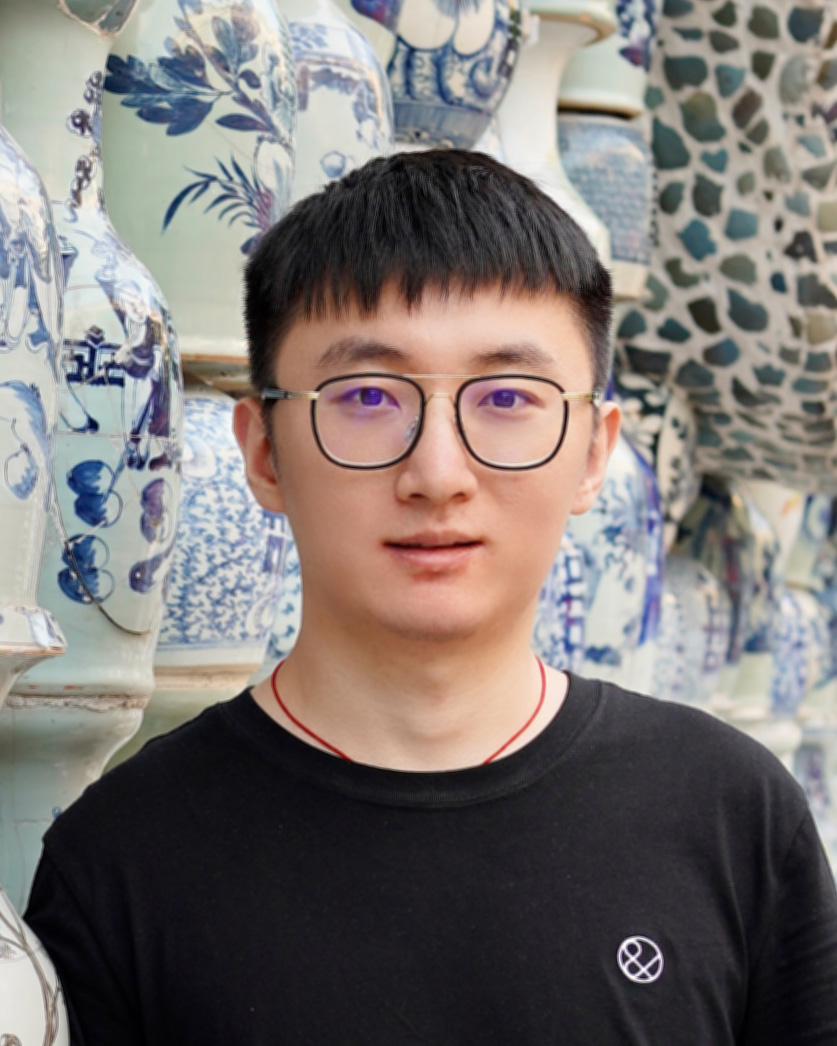}}]{Zheng Xing}(Student Member, IEEE)  received the B.Eng. degree from Ocean University of China, Qingdao, China, in 2017, and the M.Eng. degree from Beihang University, Beijing, China, in 2020. He is currently working toward the Ph.D.\ degree with the School of Science and Engineering and the Future Network of Intelligence Institute at the Chinese University of Hong Kong, Shenzhen, Guangdong, China. His research interests include channel estimation, MIMO beamforming, machine learning, and optimization for radio map reconstruction and localization.

\end{IEEEbiography} 
\begin{IEEEbiography}[{\includegraphics[width=1in,height=1.25in,clip,keepaspectratio]{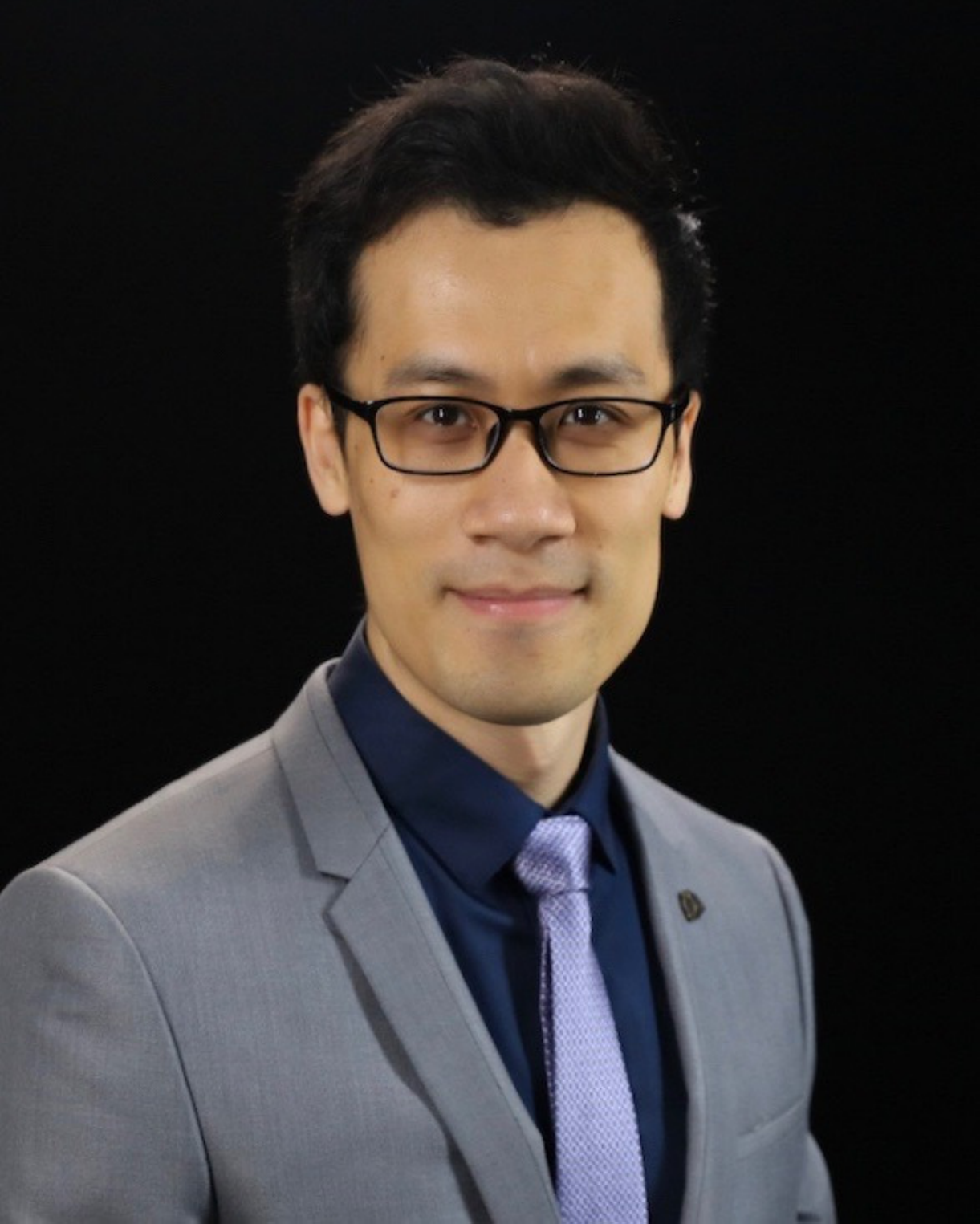}}]{Junting Chen} (S'11--M'16) received the Ph.D.\ degree in Electronic and Computer Engineering from the Hong Kong University of Science and Technology (HKUST), Hong Kong SAR China, in 2015, and the B.Sc.\ degree in Electronic Engineering from Nanjing University, Nanjing, China, in 2009. From 2014--2015, he was a visiting student with the Wireless Information and Network Sciences Laboratory at MIT, Cambridge, MA, USA.  
	
	He is an Assistant Professor with the School of Science and Engineering and the Future Network of Intelligence Institute (FNii) at The Chinese University of Hong Kong, Shenzhen (CUHK--Shenzhen), Guangdong, China. Prior to joining CUHK--Shenzhen, he was a Postdoctoral Research Associate with the Ming Hsieh Department of Electrical Engineering, University of Southern California (USC), Los Angeles, CA, USA, from 2016--2018, and with the Communication Systems Department of EURECOM, Sophia--Antipolis, France, from 2015--2016. His research interests include channel estimation, MIMO beamforming, machine learning, and optimization for wireless communications and localization. His current research focuses on radio map sensing, construction, and application for wireless communications. Dr. Chen was a recipient of the HKTIIT Post-Graduate Excellence Scholarships in 2012. He was nominated as the Exemplary Reviewer of {\scshape IEEE Wireless Communications Letters} in 2018. His paper received the Charles Kao Best Paper Award from WOCC 2022.
	
\end{IEEEbiography}

\end{document}